\documentclass[10pt,journal,compsoc]{IEEEtran}

\pdfoutput=1
\ifCLASSOPTIONcompsoc
  \usepackage[nocompress]{cite}
\else
  \usepackage{cite}
\fi
\usepackage{amsmath}
\usepackage[linesnumbered,lined,vlined,ruled,commentsnumbered]{algorithm2e}
\usepackage{array}
\usepackage{cite}
\usepackage{amsmath,amssymb,amsthm}
\usepackage{graphicx}
\usepackage{textcomp}
\usepackage{xcolor}

\newtheorem{theorem}{Theorem}
\newtheorem{lemma}{Lemma}
\newtheorem{definition}{Definition}
\newtheorem{game}{Game}
\newtheorem{corollary}{Corollary}
\newtheorem{proposition}{Proposition}
\usepackage{graphicx}
\usepackage{subfigure}
\usepackage{color}

\begin{document}
\title{Enabling Long-Term Cooperation in Cross-Silo Federated Learning: A Repeated Game Perspective}

\author{Ning~Zhang,~\IEEEmembership{Student~Member,~IEEE,}
        Qian~Ma,~\IEEEmembership{Member,~IEEE,}
        and~Xu~Chen,~\IEEEmembership{Senior~Member,~IEEE}
\IEEEcompsocitemizethanks{
\IEEEcompsocthanksitem This work was supported by the National Natural Science Foundation of China (No. 62002399, No. U20A20159, No. U1711265, No. 61972432), the Program for Guangdong Introducing Innovative and Entrepreneurial Teams (No.2017ZT07X355), and the Pearl River Talent Recruitment Program (No.2017GC010465). (Corresponding author: Qian Ma.)
\IEEEcompsocthanksitem N. Zhang and Q. Ma are with the School of Intelligent Systems Engineering, Sun Yat-sen University, Guangzhou 510006, China.\\
E-mail: zhangn87@mail2.sysu.edu.cn, maqian25@mail.sysu.edu.cn.
\IEEEcompsocthanksitem X. Chen is with the School of Computer Science and Engineering, Sun Yat-sen University, Guangzhou, 510006, China.\\
E-mail: chenxu35@mail.sysu.edu.cn.

}
}

\IEEEtitleabstractindextext{%
\begin{abstract}
Cross-silo federated learning (FL) is a distributed learning approach where clients of the same interest train a global model cooperatively while keeping their local data private. The success of a cross-silo FL process requires active participation of many clients. Different from cross-device FL, clients in cross-silo FL are usually organizations or companies which may execute multiple cross-silo FL processes repeatedly due to their time-varying local data sets, and aim to optimize their long-term benefits by selfishly choosing their participation levels. While there has been some work on incentivizing clients to join FL, the analysis of clients' long-term selfish participation behaviors in cross-silo FL remains largely unexplored. In this paper, we analyze the selfish participation behaviors of heterogeneous clients in cross-silo FL. Specifically, we model clients' long-term selfish participation behaviors as an infinitely repeated game, with the stage game being a selfish participation game in one cross-silo FL process (SPFL). For the stage game SPFL, we derive the unique Nash equilibrium (NE), and propose a distributed algorithm for each client to calculate its equilibrium participation strategy. We show that at the NE, clients fall into at most three categories: (i) \emph{free riders} who do not perform local model training, (ii) a unique \emph{partial contributor} (if exists) who performs model training with part of its local data, and (iii) \emph{contributors} who perform model training with all their local data. The existence of free riders has a detrimental effect on achieving a good global model and sustaining other clients' long-term participation. For the long-term interactions among clients, we derive a cooperative strategy for clients which minimizes the number of free riders while increasing the amount of local data for model training. We show that enforced by a punishment strategy, such a cooperative strategy is a subgame perfect Nash equilibrium (SPNE) of the infinitely repeated game, under which some clients who are free riders at the NE of the stage game choose to be (partial) contributors. We further propose an algorithm to calculate the optimal SPNE which minimizes the number of free riders while maximizing the amount of local data for model training. Simulation results show that our derived optimal SPNE can effectively reduce the number of free riders by up to $99.3\%$ and increase the amount of local data for model training by up to $82.3\%$. 
\end{abstract}

}

\maketitle

\IEEEdisplaynontitleabstractindextext

\IEEEpeerreviewmaketitle

\IEEEraisesectionheading{\section{Introduction}\label{sec:introduction}}
\subsection{Background and Motivations}\label{BM}

\IEEEPARstart{T}{he} rapid development of 5G and Internet of Things (IoT) technologies accelerates the generation of massive amount of user data \cite{5G}, such as users' finance data in banks and patients' clinical data in hospitals. User data is of paramount importance for artificial intelligence (AI). With sufficient user data, banks can develop AI models for customized financial advice and risk control services, and hospitals can train AI-assisted diagnosis and treatment models. Traditional machine learning approaches usually collect large amounts of raw data and train AI models on a central server, which may lead to privacy leakage. In order to protect data privacy, Google proposed federated learning (FL) \cite{google}, which is a distributed learning approach without sharing raw data. 

In FL, a central server coordinates the model training of many clients \cite{FL}. In the local training steps, each client trains the model by using its local data, and then sends the updated model to the central server. In the global aggregation steps, the central server aggregates the updates and sends the aggregated global model back to clients for the next iteration. The iterations stop until a predefined stopping criterion is satisfied.

\begin{figure}[!t]
\centering
\subfigure[cross-device FL]{
\begin{minipage}[c]{1\linewidth}\label{crod}
\centering
\includegraphics[width=3.5in]{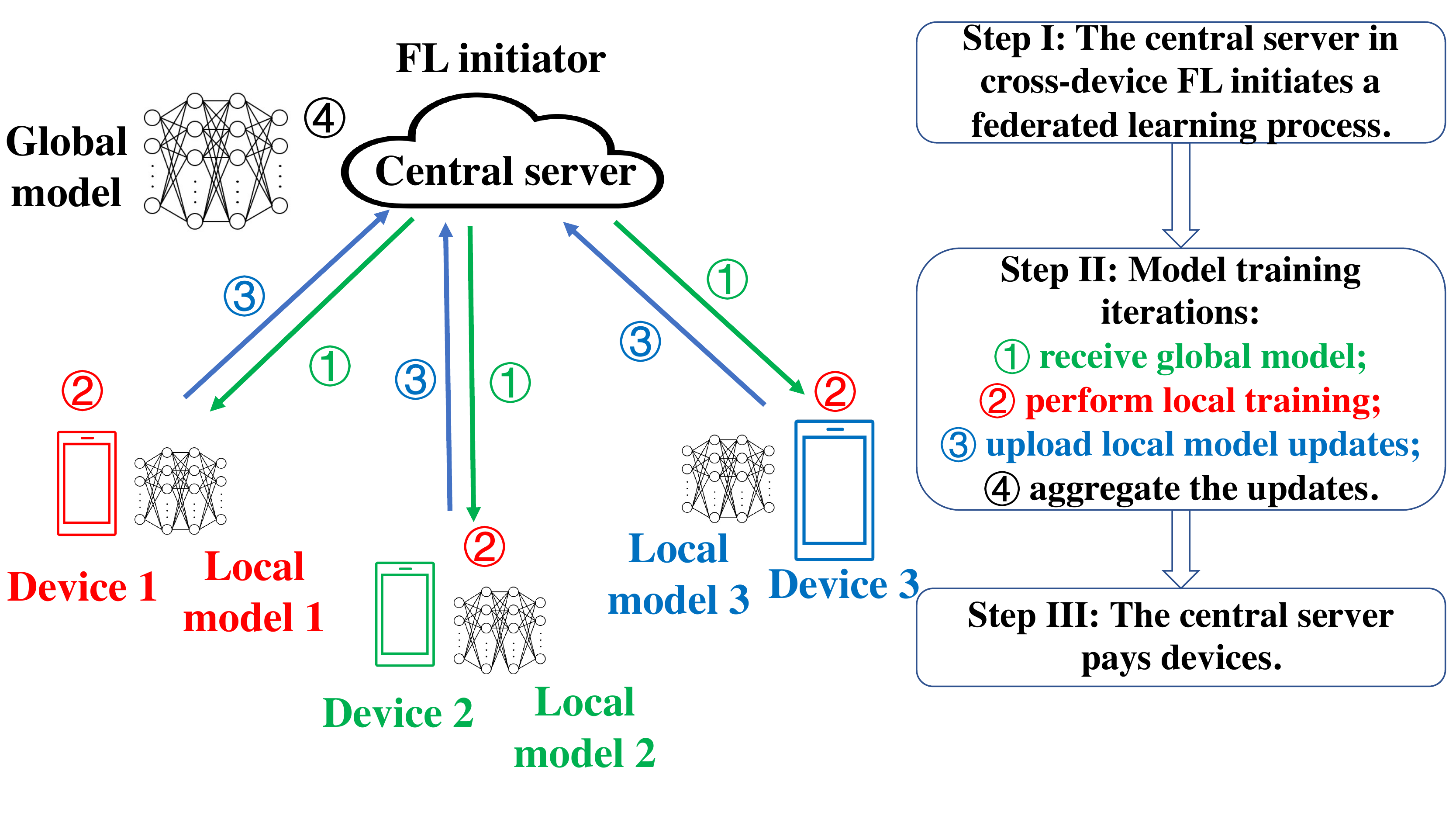}
\end{minipage}
}

\centering
\subfigure[cross-silo FL]{
\begin{minipage}[c]{1\linewidth}\label{cros}
\centering
\includegraphics[width=3.5in]{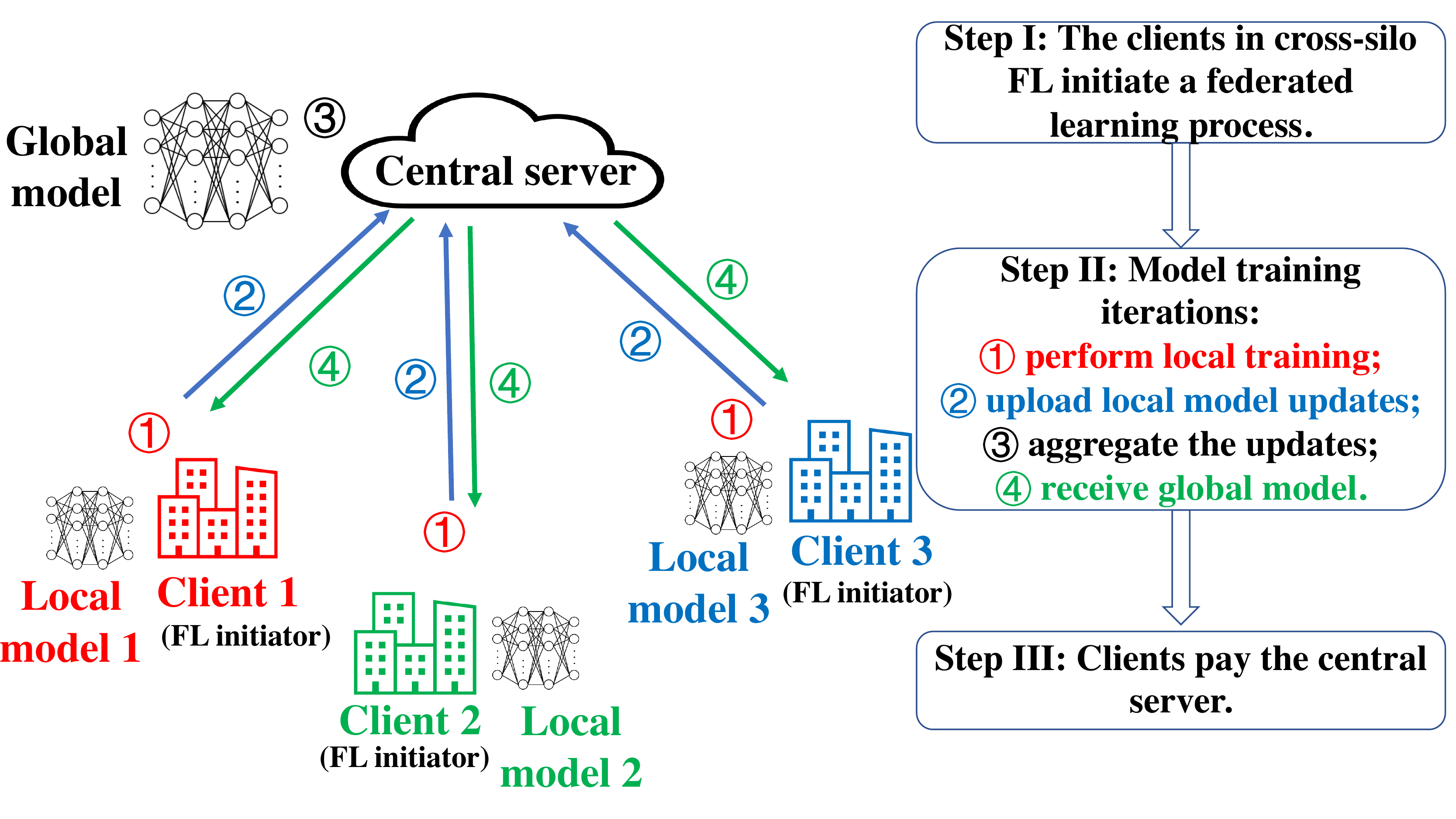}
\end{minipage}
}
\caption{An illustration of (a) cross-device FL and (b) cross-silo FL}
\end{figure}

Depending on the initiator of the FL process and the type of clients, FL is classified into two types: cross-device FL and cross-silo FL \cite{div}. In cross-device FL, as shown in Fig. \ref{crod}, the central server initiates the FL process, and clients are devices who perform local training using their local data. The central server pays each client a reward as the incentive for local training. In cross-silo FL, as shown in Fig. \ref{cros}, clients who are usually companies or organizations of the same interest, initiate the FL process and pay the central server for global aggregation.

In this paper, we study cross-silo FL, which has a wide range of practical applications. For example, WeBank and Swiss Re cooperatively perform data analysis for finance and insurance services \cite{webank}. NVIDIA Clara helps hospitals with different data sets train AI models by FL for mammogram assessment \cite{hospi}. MELLODDY uses cross-silo FL to speed up drug research \cite{mello}. In cross-silo FL, many clients who have the same interest (e.g., training models for finance service or medical diagnosis) train a global model cooperatively, and hence usually can achieve a better model compared with the ones trained only locally \cite{div}. Despite the popular applications of cross-silo FL seen in practice, there is little work analyzing clients' participation behaviors in cross-silo FL theoretically, which is the focus of this paper. 

In cross-silo FL, since clients usually belong to different entities (e.g., companies or organizations), each client may behave selfishly to maximize its own benefit. Specifically, each client selfishly chooses its participation level (i.e., the amount of local data) for the local training steps. The total amount of local data chosen by all clients will affect the global model accuracy, while the model training leads to costs to clients. In practice, clients have heterogeneous valuations for global model accuracy and incur different computation costs \cite{incen2}. Therefore, clients will make a tradeoff between the achieved global model accuracy and the incurred cost. For example, clients with high valuations for global model accuracy and low computation costs may choose a high participation level to achieve a good global model accuracy. On the other hand, clients with low valuations for global model accuracy and high computation costs may choose not to perform local training to reduce the cost, and we call such clients as \emph{free riders}. This motivates us to study clients' selfish participation behaviors in cross-silo FL through a game-theoretic approach. In this paper, we aim at addressing the following fundamental question in cross-silo FL.

\textbf{Key Question 1:} 
\emph{How do heterogeneous clients selfishly choose their participation levels in cross-silo FL?} 

Different from cross-device FL, cross-silo FL usually involves long-term repeated interactions among clients. For example, the MELLODDY project is a long-term project which involves repeated interactions among 17 partners. One reason is that clients' local data may change over time, and hence clients will perform multiple cross-silo FL processes repeatedly to adapt the global model to the time-varying local data sets. For example, hospitals constantly admit new patients and collect treatment data of these cases. Furthermore, different from cross-device FL where the central server chooses different devices to perform different FL tasks and devices aim to maximize their short-term benefits, in cross-silo FL, the same set of clients initiate FL processes repeatedly and train the global model cooperatively. Moreover, clients (e.g., organizations or companies) are usually far-sighted and aim to maximize their long-term benefits. Therefore, it is necessary to analyze clients' long-term interactions in cross-silo FL. However, the existence of free riders is detrimental to the global model for the following two reasons. First, when many clients choose to be free riders, the amount of local data for model training is small, which leads to a bad global model accuracy. Second, the behavior of free riders is unfair to the clients who use their local data to perform local trainings, which will disrupt clients' long-term cooperation. To sustain clients' long-term cooperation in cross-silo FL, we need to reduce the number of free riders. This motivates us to address the second fundamental question in cross-silo FL.

\textbf{Key Question 2:}
\emph{How to minimize the number of free riders to sustain clients' long-term cooperation in cross-silo FL?}

\begin{figure}[t]
\centering
\includegraphics[width=3.5in]{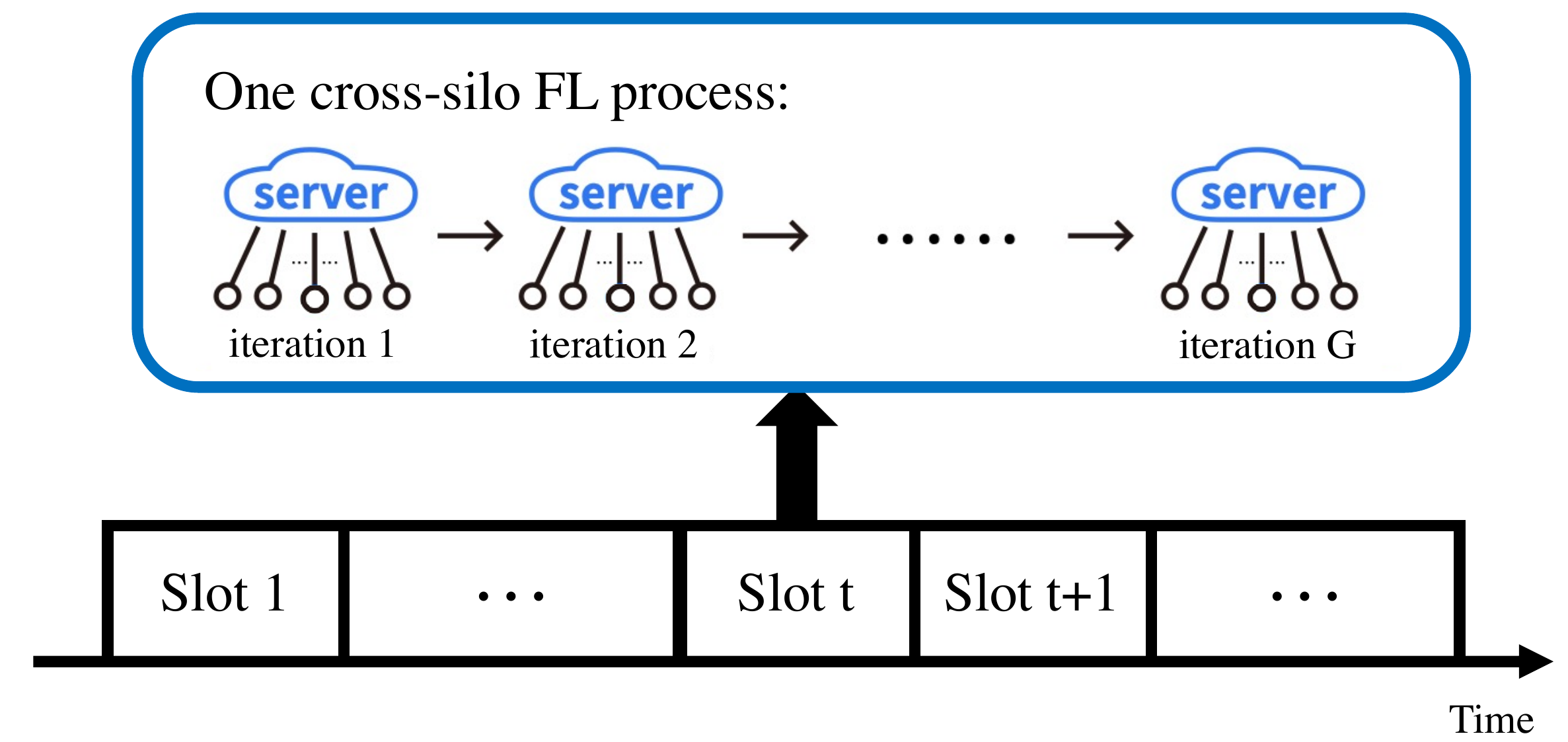}
\centering
\caption{The cross-silo FL processes in an infinite time horizon}\label{fig2}
\end{figure}

\subsection{Contributions}

In this paper, we analyze clients' participation behaviors in cross-silo FL through a game-theoretic approach. We consider an infinite time horizon which is divided into many time slots, as shown in Fig. \ref{fig2}. We model the interactions among clients in the infinite time horizon as a repeated game with the stage game being the selfish participation game in one cross-silo FL process (SPFL) in each time slot. Specifically, in each time slot, clients perform one cross-silo FL process which proceeds in several iterations of local training steps and global aggregation steps. Each time slot can be one month for banks \cite{month} or one week for hospitals \cite{weekly}. Since the local data may change dynamically over time (e.g., banks update the finance data sets monthly or even weekly \cite{month}, and hospitals update the clinical data sets weekly \cite{weekly}), clients need to repeatedly perform cross-silo FL processes to adapt the global model to the dynamic local data. We summarize our main contributions as follows.

\begin{itemize}

\item \emph{Novel Game Analysis in Cross-Silo FL: }To the best of our knowledge, this is the first paper that provides a comprehensive game-theoretic analysis of the long-term participation behaviors of heterogeneous clients in cross-silo FL from the repeated game perspective.
\item \emph{Selfish Participation Game in One Cross-Silo FL Process (SPFL): }For the stage game SPFL, we derive the unique Nash equilibrium (NE), and propose a distributed algorithm for each client to calculate its equilibrium participation strategy. We show that at the NE, clients fall into at most three categories: (i) clients with low valuation-computation ratios being \emph{free riders} who do not perform local model training, (ii) a unique client (if exists) being the \emph{partial contributor} who performs model training with part of its local data, and (iii) clients with high valuation-computation ratios being \emph{contributors} who perform model training with all their local data. 
\item \emph{Infinitely Repeated Game in the Long-Term Cross-Silo FL Processes: }In the infinitely repeated game, we derive a subgame perfect Nash equilibrium (SPNE), i.e., a cooperative strategy enforced by a punishment strategy, which achieves the minimum number of free riders while increasing the amount of local data for model training.  
Deriving the cooperative strategy is challenging since clients are heterogeneous and their decisions are coupled in a highly non-linear manner. We further propose an algorithm to calculate the optimal SPNE which minimizes the number of free riders while maximizing the amount of local data for model training.
\item \emph{Simulation Results: } We conduct extensive numerical evaluations and derive useful insights. First, at the NE of the stage game, the number of free riders increases with the number of clients and the size of the local data set. Our proposed optimal SPNE can effectively reduce the number of free riders by up to $99.3\%$, and increase the amount of local data for model training by up to $82.3\%$. The optimal SPNE performs well even when the number of clients and the size of local data set are large, which shows its scalability. Second, compared with other methods, our proposed optimal SPNE has distinct advantages in the number of contributors and the total amount of training data.
\end{itemize}

\subsection{Related Work}

FL has drawn researchers' attention in recent years. McMahan \emph{et al.} propose and illustrate the effectiveness of FL \cite{first}. In the following, we discuss related work regarding FL optimization, incentive mechanism design, the free rider issue, and cross-silo FL.

\textbf{FL Optimization}: Some papers aim to optimize different objectives in FL, such as the learning time, the model accuracy, and the energy consumption of local devices. Zhong \emph{et al.} propose the \emph{P-Fedavg} algorithm for the Parallel FL system with provable convergence rate \cite{newadd2}. Wang \emph{et al.} propose a GCN-based algorithm to derive the optimal device sampling strategy considering D2D offloading to maximize the training accuracy \cite{newadd4}. Mo \emph{et al.} optimize the energy efficiency by balancing the energy tradeoff between communication and computation in FL \cite{opti4}. Luo \emph{et al.} propose a sampling based algorithm to minimize the total cost with marginal overhead \cite{opti5}.

\textbf{Incentive Mechanism Design}: There are many papers focusing on the incentive mechanism design for FL. For example, Ding \emph{et al.} in \cite{incen1} consider contract design for the central server in cross-device FL under multi-dimensional private information. Sun \emph{et al.} in \cite{newadd1} propose a contract-based incentive mechanism to customize the payment for each participating worker considering personalized privacy preference. Zhan \emph{et al.} in \cite{res-c10} propose a DRL-based incentive mechanism to capture the tradeoff between payment and training time in FL without prior knowledge of edge nodes. Zeng \emph{et al.} in \cite{res-c11} propose an auction-based lightweight incentive mechanism considering multi-dimensional and dynamic edge resources to guide the server to effectively choose clients to participate in FL. Deng \emph{et al.} in \cite{newadd3} propose an auction-based quality-aware system called FAIR to estimate clients’ learning qualities and encourage high-quality clients' participation. Zhan \emph{et al.} in \cite{res-c12} summarize the papers on incentive mechanism design for FL from three aspects: clients’ data contribution, clients’ reputation, and resource allocation. However, none of these papers solves the free-rider problem in federated learning. Our work complements the research in this area.

\textbf{Free Rider Issue}: Free riders are those who benefit from resources or public goods but refuse to pay for them \cite{free1}, which have been extensively studied in peer-to-peer systems \cite{free4}. Several papers study how to identify free riders in FL. Lin \emph{et al.} reveal the free-riding attack in FL and propose a novel anomaly detection technique using autoencoders to identify free riders \cite{free5}. Huang \emph{et al.} design a protection mechanism named Gradient Auditing to detect and punish the free-riding behavior \cite{prox2}. Fraboni \emph{et al.} propose to establish a routine practice to inspect clients' distribution for the detection of free-rider attacks in FL \cite{prox3}. Although these papers notice the existence of free riders in FL, they only focus on identifying free riders rather than addressing the problem effectively. 

\textbf{Cross-silo FL}: There have been some researches analyzing cross-silo FL. Chen \emph{et al.} propose FOCUS for cross-silo FL to solve the problem of noisy local labels \cite{csr1}. Heikkila \emph{et al.} combine additively homomorphic secure summation protocols with differential privacy in cross-silo FL to learn complex models while guaranteeing privacy of local data \cite{csr2}. Majeed \emph{et al.} build a cross-silo horizontal federated model for traffic classification \cite{csr4}. Zhang \emph{et al.} present a system solution BatchCrypt for cross-silo FL to reduce the encryption and communication overhead caused by additively homomorphic encryption \cite{silo1}. Tang \emph{et al.} propose an incentive mechanism considering the public goods feature of cross-silo FL \cite{incen2}. However, none of the above papers analyze the long-term selfish participation behaviors of clients in cross-silo FL.

In summary, existing papers mainly focus on cross-device FL, and none of these papers studies clients' long-term selfish participation behaviors in cross-silo FL. As far as we know, this is the first work that analyzes the long-term participation behaviors of heterogeneous clients in cross-silo FL and proposes a strategy to reduce the number of free riders through the game-theoretic approach. 

We organize our paper as follows. In Section \ref{ii}, we present the system model. Then we analyze the stage game SPFL in Section \ref{iii}. We analyze the infinitely repeated game in Section \ref{iiii}. We present simulation results in Section \ref{iiiii} and conclude in Section \ref{iiiiii}. Due to space constraints, we relegate all the proofs to the supplementary material.

\section{System Model}\label{ii}

We consider a set $\mathcal{N}=\{1,2,\ldots, N\}$ of clients (e.g., companies or organizations) with the same interest participating in the cross-silo FL processes. Each client has some local data which may change over time, and clients will perform multiple cross-silo FL processes repeatedly to adapt the global model to the time-varying local data sets. Since clients are far-sighted and aim to optimize their long-term benefits, we study clients' long-term interactions in cross-silo FL processes in the infinite time horizon which is divided into time slots. In the following, we first describe a cross-silo FL process in one time slot, then model the cost of clients in a cross-silo FL process, and finally show the interactions among clients in two time scales.

\subsection{A Cross-Silo FL Process in One Time Slot}

As shown in Fig. \ref{fig2}, each time slot corresponds to one cross-silo FL process which proceeds in training iterations. We first introduce the objective of cross-silo FL, and then describe the iteration process. 

In a cross-silo FL process, clients of the same interest cooperatively train a global model represented by a parameter vector $\boldsymbol{w}$. Each client $n \in \mathcal{N}$ has a local data set $\mathcal{D}_n$, where the number of local data samples is $D_n \triangleq |\mathcal{D}_n|$.\footnote{Although different clients have different local data, we assume that the data is i.i.d. across all clients \cite{wiopt}.} Each client can choose a subset $\mathcal{X}_n \subseteq \mathcal{D}_n$ of local data for model training. Let $x_n$ denote the size of the chosen subset of local data, i.e., $x_n \triangleq |\mathcal{X}_n|$, and let $\boldsymbol{d}_{ni}$ denote the $i$-th data sample in set $\mathcal{X}_n$. Cross-silo FL aims to find the optimal global model $\boldsymbol{w}^\ast$ that minimizes the global loss function \cite{div}: 
\begin{equation}
L(\boldsymbol{w})=\sum_{n\in\mathcal{N}}\frac{x_n}{\sum_{n'\in\mathcal{N}}x_{n'}}L_n(\boldsymbol{w}).
\end{equation}
Here $L_n(\boldsymbol{w})$ is the local loss function of client $n$:
\begin{equation}
L_n(\boldsymbol{w})=\frac{1}{x_n}\sum_{\boldsymbol{d}_{ni}\in \mathcal{X}_n} l(\boldsymbol{w};\boldsymbol{d}_{ni}).
\end{equation}
Here $l(\boldsymbol{w};\boldsymbol{d}_{ni})$ is the loss function for data sample $\boldsymbol{d}_{ni}$ under $\boldsymbol{w}$. 

To achieve the optimal global model $\boldsymbol{w}^\ast$, cross-silo FL proceeds in training iterations. One widely adopted algorithm to derive $\boldsymbol{w}^\ast$ is the FedAvg algorithm \cite{first}.\footnote{In cross-silo FL, we assume that all clients participate in training iterations, and clients perform synchronous update scheme \cite{div}.} In each training iteration $r$, clients perform local model trainings over the previous global model $\boldsymbol{w}^{r-1}$ with the chosen subsets of local data by using the mini-batch stochastic gradient descent (SGD) method \cite{minibatch}. Clients derive the updated local models $\boldsymbol{w}^r_n, \forall n\in\mathcal{N}$, and send them\footnote{We assume that clients will truthfully report their local models to the central server and the central server can verify clients' contributions (e.g., by the Trusted Execution Environments proposed in \cite{bingo}).} to the central server for global aggregation. The central server derives the updated global model $\boldsymbol{w}^r=\sum_{n\in\mathcal{N}}\frac{x_n}{\sum_{n'\in\mathcal{N}}x_{n'}}\boldsymbol{w}^r_n$ \cite{first}. The above cross-silo FL process causes costs to clients, which we will introduce in detail next.

\subsection{Cost of Clients in a Cross-Silo FL Process}

In cross-silo FL, clients incur the following costs: the model accuracy loss, the computation cost, the communication cost, and the payment to the central server.

\subsubsection{Model Accuracy Loss}

The purpose of clients to perform cross-silo FL is to achieve a global model with good accuracy, i.e., a global model with a small accuracy loss \cite{wiopt}. The global model accuracy loss can be calculated as $L(\boldsymbol{w}^G) - L(\boldsymbol{w}^\ast)$, where $L(\boldsymbol{w}^G)$ and $L(\boldsymbol{w}^\ast)$ are the global losses under parameters $\boldsymbol{w}^G$ and $\boldsymbol{w}^\ast$, respectively, and $G$ is the number of training iterations in a cross-silo FL process. According to \cite{BD}, the expected global model accuracy loss is bounded by $O(1 / \sqrt{BG} + 1/G)$, where $B$ is the total batch size that clients use for model trainings, i.e., $B=\sum_{n\in\mathcal{N}}x_n$.\footnote{Note that in FL, the central server knows clients' chosen amounts of local data for model trainings, i.e., the values of $x_n, \forall n\in\mathcal{N}$, which are the weights in the global aggregation steps \cite{incen2}\cite{wiopt}.} Hence, the expected global model accuracy loss decreases with the amount of local data $\sum_{n\in\mathcal{N}} x_n$ for model trainings and the number of training iterations $G$. In this paper, we denote the model accuracy loss of each client $n\in\mathcal{N}$ as $A(x_n, \boldsymbol{x}_{-n})$, which depends on all clients' chosen amounts of local data for model trainings and can be calculated as follows:\footnote{In this paper, we assume a complete information setting where each client $n\in\mathcal{N}$ knows other clients' chosen amounts of local data, i.e., the values of $x_{n'}, \forall n'\in\mathcal{N}, n' \neq n$. This information can be announced by the central server to clients. Note that our proposed algorithms later in the paper which calculate the equilibrium strategy for each client do not require such information. The analysis in the complete information setting provides useful insights for practical cross-silo FL systems. For the analysis of the incomplete information setting, we will leave it as future work.}
\begin{equation}
A(x_n, \boldsymbol{x}_{-n}) = \frac{1}{\sqrt{(x_n + \sum_{n'\in\mathcal{N},n'\neq n}x_{n'})G}} + \frac{1}{G},
\end{equation}
where $\boldsymbol{x}_{-n} = \{x_1, x_2, \ldots, x_{n-1}, x_{n+1}, \ldots, x_N\}$.

\subsubsection{Computation Cost}

The local model training steps consume local computation resources. In a cross-silo FL process, the CPU energy consumption of each client $n\in\mathcal{N}$, which depends on its chosen amount of local data $x_n$ for model training, can be calculated as \cite{com0}
\begin{equation}
\mathcal{E}(x_n) = \frac{\varsigma_n}{2} \mu_n \vartheta_n^2 x_n,
\end{equation}
where $\varsigma_n$ is a coefficient depending on the client's computing chip architecture, $\mu_n$ is the number of CPU cycles of client $n$ to perform the model training on one local data sample, and $\vartheta_n$ is the CPU processing speed (in cycles per second) of client $n$. We denote the \emph{computation cost coefficient} of client $n$ as $E_n=\frac{\varsigma_n}{2} \mu_n \vartheta_n^2$, and the computation cost of client $n$ can be calculated as $\mathcal{E}(x_n)=E_nx_n$.


\subsubsection{Communication Cost}

In a cross-silo FL process, clients send their local models $\boldsymbol{w}_n^r, \forall n\in\mathcal{N}$, to the central server for aggregation, and receive the updated global model $\boldsymbol{w}^r$ for the next training iteration, for all $r=1,2, \ldots, G$. Since clients in cross-silo FL are companies or organizations, the data transmission between clients and the central server can be through either wired networks (e.g., through the high-speed wired connections \cite{wired}) or wireless networks (e.g., using transmission protocols TDMA, OFDMA or NOMA \cite{res-a1}). When transmitting model updates, clients usually have different data transmission rates which depend on their transmission power and channel coefficients in wireless networks \cite{res-a2}. In this case, clients will incur heterogeneous communication costs (e.g., different transmission delays). We denote the communication cost of each client $n\in\mathcal{N}$ as $C_n$. Since clients' model updates have the same size, $C_n$ does not depend on the participation strategy $x_n$ and is a constant in our paper.

\subsubsection{Payment to the Central Server} 

Clients need to pay the central server for global model aggregation. We assume that the payment of each client is $p$. 

In summary, we define the total cost of each client $n\in \mathcal{N}$ as follows, 
\begin{equation}\label{Fn}
F_n(x_n, \boldsymbol{x}_{-n}) = \rho_nA(x_n, \boldsymbol{x}_{-n}) + \mathcal{E}(x_n) + C_n + p.
\end{equation} 
Here $\rho _n$ denotes client $n$'s valuation for the model accuracy, which describes how important the model accuracy is to the client \cite{incen2}. For example, $\rho_n$ can be a bank's customer churn per unit of model accuracy loss when using the risk control service trained by cross-silo FL, or a pharmaceutical company's unit revenue loss when using the drug research model trained by cross-silo FL.

\subsection{Interactions among Clients in Two Time Scales}\label{DD}

In this subsection, we consider the interactions among clients in two time scales as shown in Fig. 2. Specifically, we describe the behaviors of clients in each time slot and in the infinite time horizon, respectively.

We first analyze the interactions among clients in one time slot. As in \cite{month}\cite{weekly}, the length of a time slot can be a month for banks or a week for hospitals. One time slot corresponds to one cross-silo FL. At the beginning of the time slot, each client $n$ chooses the amount of local data $x_n$ for model training (i.e., its participation level in the cross-silo FL process), to minimize its total cost calculated in \eqref{Fn}, considering the participation behaviors of other clients. When clients are myopic and only care about their costs in the current time slot, we model their participation behaviors as a selfish participation game in one cross-silo FL process (SPFL), to be introduced in detail in Section 3.

We then analyze the interactions among clients in the infinite time horizon which is divided into many time slots. In the infinite time horizon, each client's local data set changes over time, i.e., $\mathcal{D}_n^{t-1} \neq \mathcal{D}_n^t, \forall n\in\mathcal{N}$, where $\mathcal{D}_n^{t-1}$ and $\mathcal{D}_n^t$ are the local data sets of client $n$ in time slot $t-1$ and time slot $t$, respectively. We assume that the number of local data samples remains unchanged, i.e., $\vert \mathcal{D}_n^{t-1} \vert = \vert \mathcal{D}_n^t \vert, \forall n\in\mathcal{N}$.\footnote{For banks or hospitals, the number of users does not change much from month to month, so we assume that the number of local data samples remains the same.} To adapt the global model to the time-varying data sets, clients perform cross-silo FL processes repeatedly. Specifically, at the beginning of each time slot $t$, clients receive the global model $\boldsymbol{w}^{t-1}$ in the previous time slot, and then perform a cross-silo FL process using their local data (as described in the selfish participation game SPFL) to derive the global model $\boldsymbol{w}^t$. In the infinite time horizon, clients are far-sighted, and each client $n$ chooses the amount of local data for model training to minimize its long-term discounted total cost. We model clients' long-term selfish participation behaviors as an infinitely repeated game with the stage game being SPFL, which will be introduced in detail in Section 4.

\section{Stage Game Analysis}\label{iii}

In this section, we analyze the stage game, i.e., the selfish participation game in one cross-silo FL process (SPFL). We derive the unique Nash equilibrium of the stage game SPFL, and design a distributed algorithm for each client to compute its equilibrium participation strategy. 

\subsection{Stage Game Modeling}\label{subsec1}

In a cross-silo FL process, clients selfishly choose their participation levels to minimize their own costs. We model the behaviors of clients in each time slot as a selfish participation game as follows:

\begin{game}[Selfish Participation Game in a Cross-Silo FL Process (SPFL)]\label{game11}
$ $
\begin{itemize}
 \item Players: the set $\mathcal{N}$ of clients.
 \item Strategies: each client $n\in\mathcal{N}$ chooses the amount of local data $x_n \in [0, D_n]$ for model training.
 \item Objectives: each client $n\in\mathcal{N}$ aims to minimize its total cost $F_n(x_n, \boldsymbol{x}_{-n})$ defined in \eqref{Fn}.
\end{itemize}
\end{game}

Here $D_n$ is the total number of data samples in client $n$'s local data set $\mathcal{D}_n$, i.e., $D_n=|\mathcal{D}_n|$.

\subsection{Nash Equilibrium of SPFL}

Next we define the Nash equilibrium (NE) of SPFL and derive the unique NE later. 

\begin{definition}[Nash Equilibrium]

A Nash equilibrium of Game \ref{game11} is a strategy profile $\boldsymbol{x}^\ast=\{x_n^\ast: \forall n\in\mathcal{N}\}$ such that for each client $n\in\mathcal{N}$,
\begin{equation*}
F_n(x_n^\ast,\boldsymbol{x}_{-n}^\ast) \leq F_n(x_n,\boldsymbol{x}_{-n}^\ast), \mbox{ for all } x_n\in [0,D_n]. 
\end{equation*}

\end{definition}

At the NE, each client's strategy is the best response to other clients' strategies. In Lemma \ref{lem1}, we characterize the best response of each client $n\in\mathcal{N}$ that minimizes its total cost given all other clients' strategies $\boldsymbol{x}_{-n}$. 

\begin{lemma}\label{lem1}

The best response of each client $n\in\mathcal{N}$ in Game \ref{game11} is 
\begin{equation}
\begin{aligned}
&x_n^{\rm BR}(\boldsymbol{x}_{-n})= \\
&\min \left\{  D_n, \max \left\{ \sqrt[3]{\frac{\rho_n^2}{4GE_n^2}} - \sum_{n'\in\mathcal{N},n'\neq n} x_{n'}, 0 \right\}  \right\}.
\end{aligned}
\end{equation}
\end{lemma}

\begin{proof}
See Appendix A in the supplementary material. 
\end{proof}

Client $n$'s best response $x_n^{\rm BR}(\boldsymbol{x}_{-n})$ increases with its valuation $\rho_n$ for model accuracy, while decreases with its computation cost coefficient $E_n$, the number of iterations $G$, and the total amount of local data $\sum_{n'\in\mathcal{N},n'\neq n} x_{n'}$ that other clients choose for model training. Intuitively, when client $n$ has a high valuation for the global model accuracy, client $n$ may choose a large amount of local data for model training to achieve a good global model. On the contrary, when client $n$ has a large computation cost coefficient $E_n$, client $n$ may choose a small amount of local data for training to reduce the computation cost. Furthermore, when the number of iterations $G$ is large, clients can achieve a good global model with a small amount of local data as shown in (3). Similarly, when the total amount of local data from other clients is large, the accuracy loss in (3) is small and client $n$ will choose a small amount of local data for model training to reduce the computation cost.

Next we derive the NE of Game 1. For simplicity of analysis, we assume that $D_n=D,\forall n\in\mathcal{N}$. For each client $n\in\mathcal{N}$, we define $h_n \triangleq \sqrt[3]{\frac{\rho_n^2}{4GE_n^2}}$, which depends on $\frac{\rho_n}{E_n}$, the \emph{valuation-computation ratio} (i.e., the ratio of the valuation parameter to the computation cost coefficient). Without loss of generality, we assume that $\frac{\rho_1}{E_1} \leq \frac{\rho_2}{E_2} \leq \cdots \leq \frac{\rho_N}{E_N}$. 

\begin{theorem}\label{tho1}
A Nash equilibrium $\boldsymbol{x}^\ast$ always exists in Game \ref{game11}, and falls into one of the following two cases. 
\begin{itemize}
\item Case I: If there exists a critical client $k\in\mathcal{N}$ that satisfies $(N-k)D \leq h_k \leq (N+1-k)D$, then the NE $\boldsymbol{x}^\ast=\{x_n^\ast: \forall n \in \mathcal{N}\}$ is: 
\begin{equation}
x_n^\ast=
\left\{
\begin{aligned}
& 0, & \mbox{ if } n < k;\\
& h_k - (N-k)D, & \mbox{ if } n = k;\\
& D, & \mbox{ if } n>k.
\end{aligned}
\right.
\end{equation}
\item Case II: If the critical client $k$ in Case I does not exist, then there must exist a client $m \in \mathcal{N}$ that satisfies $h_m < (N-m)D < h_{m+1}$, and in this case, the NE $\boldsymbol{x}^\ast=\{x_n^\ast: \forall n \in \mathcal{N}\}$ is: 
\begin{equation}
x_n^\ast=
\left\{
\begin{aligned}
& 0, & \mbox{ if } n \leq m;\\
& D, & \mbox{ if } n>m.
\end{aligned}
\right.
\end{equation}
\end{itemize}
\end{theorem}

\begin{proof}
See Appendix B in the supplementary material.
\end{proof}

Theorem 1 shows that the NE of Game 1 falls into two cases depending on whether a critical client $k$ exists or not. If such a critical client $k$ exists, clients can be divided into three categories at equilibrium: (i) clients whose valuation-computation ratios are lower than that of the critical client $k$ choose to be \emph{free riders} that do not perform local model training, (ii) the unique critical client $k$ chooses to be a \emph{partial contributor} who performs model training with part of its local data, and (iii) clients whose valuation-computation ratios are higher than that of the critical client $k$ choose to be \emph{contributors} who perform model training with all their local data. If the critical client $k$ does not exist, clients can be divided into two categories at equilibrium: (i) clients with low valuation-computation ratios choose to be free riders, and (ii) clients with high valuation-computation ratios choose to be contributors. Note that at equilibrium of Game 1, it is not possible that all clients choose to be free riders, as shown in the following corollary.

\begin{corollary}
At equilibrium of Game 1, the client with the highest valuation-computation ratio chooses a positive amount of local data for model training, i.e., $x_N^\ast > 0$. 
\end{corollary}

The significance of the above result is to establish that at equilibrium of Game 1, there will always be a positive amount of local data chosen by clients to perform model training in cross-silo FL. Specifically, the client $N$ who has the highest valuation-computation ratio will never choose to be a free rider. Even if all other clients choose to be free riders at equilibrium, client $N$ will choose a positive amount of local data for model training, otherwise its model accuracy loss in (3) and its total cost in (5) will go to infinity.

At equilibrium, however, clients with low valuation-computation ratios may choose to be free riders. This is a result of the tradeoff between the model accuracy loss and the computation cost. Specifically, given the global model aggregated from the local models of the (partial) contributors, clients with low valuation-computation ratios choose not to perform local trainings to avoid the computation costs. The behaviors of free riders only minimize their own costs, but cannot improve the global model accuracy. As discussed in Section \ref{BM}, the existence of free riders hampers clients' long-term participation in cross-silo FL. Later in Section \ref{iiii}, we will analyze the interactions among clients in the infinite time horizon, and design a cooperative strategy to reduce the number of free riders.

Next we discuss the uniqueness of the NE of Game 1. 

\begin{theorem}\label{tho2}
Game 1 admits a unique Nash equilibrium if $\frac{\rho_1}{E_1} < \frac{\rho_2}{E_2} < \cdots < \frac{\rho_N}{E_N}$.
\end{theorem}

\begin{proof}
See Appendix C in the supplementary material.
\end{proof}

Here $\frac{\rho_1}{E_1} < \frac{\rho_2}{E_2} < \cdots < \frac{\rho_N}{E_N}$ is a mild constraint which can be satisfied in most cases in practice.\footnote{The constraint that $\{ \frac{\rho_n}{E_n}: \forall n\in\mathcal{N} \}$ can be ranked in a strictly ascending order is mild, because clients usually have different valuations $\rho_n$ for global model accuracy as well as heterogeneous computation cost coefficients $E_n$.  When the constraint $\frac{\rho_1}{E_1} < \frac{\rho_2}{E_2} < \cdots < \frac{\rho_N}{E_N}$ is not satisfied, the Nash equilibrium may not be unique. Consider an example with $N = 2$ clients, and $\frac{\rho_1}{E_1} = \frac{\rho_2}{E_2}$, $h_1 = h_2 = 10, D_1=D_2=10$. In this example, $(x_1^\ast, x_2^\ast)=(4, 6)$ is a Nash equilibrium of Game 1, and $(x_1^\ast, x_2^\ast)=(5, 5)$ is also a Nash equilibrium of Game 1.} To reach the NE of Game 1, we prove in Appendix D that the stage game SPFL is a concave game and the best response update process converges to the NE of Game 1. Next we will design a distributed algorithm to calculate the equilibrium participation strategy for each client.

\begin{algorithm}[t]
\LinesNumbered
\SetAlgoLined
\begin{small}
\KwIn{$\rho_n, E_n, \forall n\in\mathcal{N}, G, N, D$}
\KwOut{$x_n^*, \forall n\in\mathcal{N} $}
Each client $n\in\mathcal{N}$ reports the value of $\frac{\rho_n}{E_n}$ to the central server.

The central server sorts the clients in an ascending order of $\frac{\rho_n}{E_n}, \forall n\in\mathcal{N}$, and sends each client its index $n$ as well as the total number of clients $N$.

\For{$n = 1:N$}{
Calculate $h_n = \sqrt[3]{\frac{\rho_n^2}{4GE_n^2}}$\;
\uIf{$h_n<(N-n)D$}{
$x_n^\ast = 0$\;
}\uElseIf{
$h_n>(N-n+1)D$}
{$x_n^* = D$\;}
\Else
{$x_n^* = h_n - (N - n)D$.}
}

\end{small}
\caption{A Distributed Algorithm to Compute Equilibrium Strategy for Each Client in Game 1}
\label{code:recentEnd}
\end{algorithm}

\subsection{A Distributed Algorithm to Compute Equilibrium Strategy }

In this section, we design a distributed algorithm (in Algorithm 1) for clients to compute the equilibrium strategy of Game 1. In the distributed algorithm, each client calculates its equilibrium participation strategy based on its own information, without knowing other clients' participation decisions or parameter information (e.g., the valuation parameters for global model accuracy and the computation cost coefficients). 

To calculate the equilibrium strategy, each client $n\in\mathcal{N}$ needs to know the total number of clients $N$ and the ranking of its valuation-computation ratio $\frac{\rho_n}{E_n}$ in all clients (i.e., the index $n$), which can be obtained by reporting its ratio $\frac{\rho_n}{E_n}$ to the central server for ordering\footnote{We assume that each client $n\in\mathcal{N}$ will truthfully report its ratio $\frac{\rho_n}{E_n}$ to the central server. How to incentivize clients to truthfully reveal their information \cite{tru2} is beyond the scope of this paper.} (Lines 1-2 in Algorithm 1). The equilibrium strategy of each client $n\in\mathcal{N}$ depends on the value of $h_n$. If $h_n < (N-n)D$, client $n$ chooses to be a free rider, i.e., $x_n^\ast=0$ (Lines 5-6 in Algorithm 1). If $h_n > (N-n+1)D$, client $n$ chooses to be a contributor, i.e., $x_n^\ast=D$ (Lines 7-8 in Algorithm 1). Otherwise, client $n$ chooses to be a partial contributor, i.e., $x_n^\ast=h_n-(N-n)D$ (Lines 9-10 in Algorithm 1). Note that clients do not need to know whether the critical client $k$ exists at equilibrium, and Algorithm 1 covers both cases of the equilibrium described in Theorem 1.

Algorithm 1 has a linear complexity of $\mathcal{O}(N)$. Specifically, computing the value of $h_n$ and the value of $x_n^\ast$ for each client $n\in\mathcal{N}$ in Lines 4-11 has a complexity of $\mathcal{O}(1)$. With the for loop operation in Line 3, the overall complexity of Algorithm 1 is $\mathcal{O}(N)$. Thus it is scalable with the number of clients in cross-silo FL.

As discussed in Section \ref{DD}, the interactions among clients in the infinite time horizon are inherently a repeated process. Since clients' local data may change over time, they need to adapt the global model to the local data sets constantly. In the next section, we analyze clients' long-term selfish participation behaviors.

\section{Repeated Game Analysis}\label{iiii}

In this section, we analyze the long-term interactions among clients in cross-silo FL processes. We model clients' selfish participation behaviors as a repeated game in the infinite time horizon, with the stage game being the SPFL (i.e., Game 1). We first model the repeated game, and then derive a cooperative participation strategy that can minimize the number of free riders while increasing the amount of local data for model training. We derive a subgame perfect Nash equilibrium (SPNE) of the repeated game which enforces the cooperative participation strategy by a punishment strategy. Finally, we propose an algorithm to calculate the optimal SPNE that can minimize the number of free riders while maximizing the amount of local data for model training.

\subsection{Repeated Game Modeling}

As shown in Fig. \ref{fig2}, we consider clients' interactions in the infinite time horizon which consists of infinitely many time slots. Each time slot corresponds to one cross-silo FL process. In the infinite time horizon, each client chooses its participation level (i.e., the amount of local data for model training) for each cross-silo FL process to minimize its long-term discounted total cost. 

We first calculate the long-term discounted total cost for each client. We denote the participation strategy profile of all clients in time slot $t$ as $\boldsymbol{x}^t=\{x^t_n: \forall n\in\mathcal{N}\}$. For each client $n\in\mathcal{N}$, its local data set changes over time while keeping the size unchanged, i.e., $\mathcal{D}_n^{t-1} \neq \mathcal{D}_n^{t}, \vert \mathcal{D}_n^{t-1} \vert = \vert \mathcal{D}_n^{t} \vert=D_n$. Therefore, we have $x_n^t\in[0,D_n]$ for each time slot $t$. We denote the strategy profile history up to time slot $t$ as 

\begin{center}
$s^t \triangleq \{\boldsymbol{x}^0, \boldsymbol{x}^1, \ldots, \boldsymbol{x}^{t-1}\}$.
\end{center}
Then for each client $n\in\mathcal{N}$, its long-term discounted total cost is
\begin{equation}\label{c7}
{\mathcal{C}}_n(s^\infty) = \sum\limits_{t=0}^{\infty}\delta_n^t F_n(\boldsymbol{x}^t),
\end{equation} 
where $\delta_n \in [0,1)$ is the discount factor \cite{folk} of client $n$, and $F_n(\boldsymbol{x}^t)$ is the cost of client $n$ in time slot $t$ calculated in (5). 

We model the interactions among clients in the infinite time horizon as a repeated game as follows.

\begin{game}[Repeated Game in Infinitely Many Cross-Silo FL Processes]\label{gamere}
$ $
\begin{itemize}
\item Players: the set $\mathcal{N}$ of clients.
\item Strategies: each client $n \in \mathcal{N}$ chooses the amount of local data $x_n^t \in [0, D_n]$ for model training in each time slot $t$.
\item Histories: the strategy profile history $s^t$ till time slot $t$.
\item Objectives: each client $n\in\mathcal{N}$ aims to minimize its long-term discounted total cost ${\mathcal{C}}_n(s^\infty)$ defined in \eqref{c7}. 

\end{itemize}
\end{game}

In the repeated game (Game 2), clients train the global model repeatedly to adapt the model to clients' time-varying local data sets by participating in infinitely many cross-silo FL processes. We aim to find the subgame perfect Nash equilibrium (SPNE) of Game \ref{gamere}. According to Folk Theorem \cite{folk}, under proper discount factors $\delta_n, \forall n \in \mathcal {N}$, any feasible and individually rational participation strategy profile can be the SPNE of the infinitely repeated game \cite{newadd5}. Since the success of cross-silo FL depends on the active long-term participation of clients, we will characterize the SPNE that can minimize the number of free riders while increasing the chosen amount of local data for model training. Next in Section \ref{4b}, we derive a cooperative participation strategy in each cross-silo FL process, i.e., a strategy profile where clients cooperate to minimize the number of free riders. Later in Section \ref{4c}, we show that the cooperative participation strategy can be enforced as the SPNE of Game 2 by a punishment strategy.

\subsection{Cooperative Participation Strategy}\label{4b}

Now we derive the cooperative participation strategy $\boldsymbol{x}^{coop,t}$ in each time slot $t$. The cooperative participation strategy aims to minimize the number of free riders while reducing the cost of each client compared with the cost at the NE of the stage game SPFL, i.e., $F_n(\boldsymbol{x}^{coop, t}) < F_n(\boldsymbol{x}^{*, t})$ for all $n \in N$. Here $\boldsymbol{x}^{\ast,t}$ is the NE of the stage game SPFL in time slot $t$. In this subsection, we focus on a time slot $t$. When there is no confusion, we ignore the time index and write $\boldsymbol{x}^{coop,t}$ and $\boldsymbol{x}^{\ast,t}$ as $\boldsymbol{x}^{coop}$ and $\boldsymbol{x}^{\ast}$, respectively.

We first analyze the cooperative participation strategy for the case where the critical client $k$ exists at the NE of Game 1 (as shown in Case I in Theorem 1). For each client $n\in\mathcal{N}$, we define $B_n \triangleq \frac{2\sqrt{G((N-k)D+x_k^\ast)^3}}{k-n}$ and $O_n \triangleq \frac{D\sqrt{G}}{\frac{1}{\sqrt{(N-k)D+x_k^*}}-\frac{1}{\sqrt{(N-n)D+x_k^*}}}$, where $x_k^\ast$ is the participation strategy of the critical client $k$ at the NE of Game 1. In this case, we derive the cooperative strategy as follows.

\begin{theorem}\label{kkk}

When the critical client $k$ exists at the NE of Game 1 (as shown in Case I in Theorem 1), we find the client $l$ which satisfies
\begin{equation}\label{l2l2}
l=\min \left\{ n\in\mathcal{N}: n < k, \mbox{ and } \frac{\rho_n}{E_n}>B_n \right\}.
\end{equation}
Then the cooperative participation strategy $\boldsymbol{x}^{coop}=\{x_n^{coop}: \forall n \in \mathcal{N}\}$ that minimizes the number of free riders is: 
\begin{equation}
x_n^{coop}=
\left\{
\begin{aligned}
& 0,  &\mbox{ if } &n < l;\\
& x^{coop}, &\mbox{ if } &l \leq n < k;\\
& x_k^*, &\mbox{ if }  &n = k;\\
& D, &\mbox{ if } &n>k.
\end{aligned}
\right.
\end{equation}
Here $x^{coop}$ is a cooperative participation level that satisfies $x^{coop} \leq x_l^{th}$, and $x_l^{th}$ is the maximum amount of local data that client $l$ can choose under the cooperative strategy, calculated as follows,
\begin{equation}\label{coopk}
x_l^{th} = 
\left\{
\begin{aligned}
& x_l^{th}\left(\frac{\rho_l}{E_l}\right), &\mbox{ if }  &B_l < \frac{\rho_l}{E_l} \leq O_l; \\
& D, &\mbox{ if }  &\frac{\rho_l}{E_l} > O_l.
\end{aligned}
\right.
\end{equation}
Here $x_l^{th}\left(\frac{\rho_l}{E_l}\right)$, which depends on the valuation-computation ratio $\frac{\rho_l}{E_l}$, is the unique non-zero solution to the following implicit equation: 
\begin{equation}\label{imp}
\frac{\rho_l}{E_l} = \frac{\sqrt{G}x_l^{th}}{\frac{1}{\sqrt{(N-k)D+x_k^*}}-\frac{1}{\sqrt{(N-k)D+x_k^*+(k-l) x_l^{th}}}}.
\end{equation}
\end{theorem}

\begin{proof}
See Appendix E in the supplementary material.
\end{proof}

Theorem \ref{kkk} shows that when the critical client $k$ exists at the NE of Game 1, we can derive a cooperative participation strategy under which clients fall into at most four categories: (i) the free riders at the NE of Game 1 who satisfy $\frac{\rho_n}{E_n} \leq B_n$ still choose to be \emph{free riders} under the cooperative strategy, (ii) the free riders at the NE of Game 1 who satisfy $\frac{\rho_n}{E_n} > B_n$ choose to be \emph{converted contributors} who perform model training with a positive amount\footnote{We assume that for fairness, converted contributors choose the same amount of local data for model training. Fairness is an important problem in FL \cite{fair0}, which is beyond the scope of this paper.} of local data $x^{coop}$, (iii) the unique critical client $k$ chooses to be a \emph{partial contributor} who performs model training with the amount of local data $x_k^\ast$, and (iv) the contributors at the NE of Game 1 still choose to be \emph{contributors} who perform model training with all their local data.

Compared with the NE of Game 1 in Theorem 1, Theorem \ref{kkk} shows that under the cooperative participation strategy, we have the following improved results. First, the number of free riders is reduced from $k-1$ to $l-1$ where $l \leq k$. Second, the amount of local data chosen by the converted contributors increases from $0$ to $x^{coop}>0$. The following corollary shows that under certain conditions, the cooperative strategy can reduce the number of free riders to $0$. 

\begin{corollary}
Under the cooperative strategy in Theorem 3, no client chooses to be a free rider if 
$$\frac{\rho_1}{E_1}> B_1 \triangleq \frac{2\sqrt{G((N-k)D+x_k^\ast)^3}}{k-1}.$$
Furthermore, all clients choose to be contributors, i.e., $x_n^{coop}=D, \forall n\in\mathcal{N}$, if 
$$\frac{\rho_1}{E_1}> O_1 \triangleq \frac{D\sqrt{G}}{\frac{1}{\sqrt{(N-k)D+x_k^*}}-\frac{1}{\sqrt{(N-1)D+x_k^*}}}.$$
\end{corollary}

The above result shows that when client 1's ratio $\frac{\rho_1}{E_1}$ is high, all clients choose to be (converted/partial) contributors under the cooperative participation strategy, and there is no free rider in cross-silo FL. 

Next, we analyze the cooperative participation strategy for the case where the critical client $k$ does not exist at the NE of Game 1 (as shown in Case II in Theorem 1). For each client $n\in\mathcal{N}$, we define $B'_n \triangleq \frac{2\sqrt{G(N-m)^3 D^3}}{m-n+1}$ and $O'_n \triangleq \frac{D\sqrt{G}}{\frac{1}{\sqrt{(N-m)D}}-\frac{1}{\sqrt{(N-n+1)D}}}$. Here $m$ is the index of the client who satisfies $h_m < (N-m)D < h_{m+1}$. In this case, we derive the cooperative strategy as follows.

\begin{theorem}\label{mmm}
When there is a client $m$ that satisfies $h_m < (N-m)D < h_{m+1}$ at the NE of Game 1 (as shown in Case II in Theorem 1), we find the client $l$ which satisfies 
\begin{equation}\label{l1l1}
l=\min \left\{n\in\mathcal{N}: n \leq m, \mbox{ and } \frac{\rho_n}{E_n}>B'_n\right\}.
\end{equation}
Then the cooperative participation strategy $\boldsymbol{x}^{coop}=\{x_n^{coop}: \forall n \in \mathcal{N}\}$ that minimizes the number of free riders is:
\begin{equation}
x_n^{coop}=
\left\{
\begin{aligned}
& 0, &\mbox{ if } &n < l;\\
& x^{coop}, &\mbox{ if } &l \leq n \leq m;\\
& D, &\mbox{ if } &n>m.
\end{aligned}
\right.
\end{equation}
Here $x^{coop}$ is a cooperative participation level that satisfies $x^{coop}\leq x_l^{th}$, and $x_l^{th}$ is the maximum amount of local data that client $l$ can choose under the cooperative strategy, calculated as follows,
\begin{equation}\label{coopm}
x_l^{th} = 
\left\{
\begin{aligned}
& x_l^{th}\left(\frac{\rho_l}{E_l}\right), &\mbox{ if }  &B'_l < \frac{\rho_l}{E_1} \leq O'_l; \\
& D, &\mbox{ if }  &\frac{\rho_l}{E_l} > O'_l.
\end{aligned}
\right.
\end{equation}
Here $x_l^{th}\left(\frac{\rho_l}{E_l}\right)$, which depends on the valuation-computation ratio $\frac{\rho_l}{E_l}$, is the unique non-zero solution to the following implicit equation: 
\begin{equation}\label{fun}
\frac{\rho_l}{E_l} = \frac{\sqrt{G}x_l^{th}}{\frac{1}{\sqrt{(N-m)D}}-\frac{1}{\sqrt{(N-m)D+(m-l+1)x_l^{th}}}}.
\end{equation}

\end{theorem}

\begin{proof}

See Appendix F in the supplementary material.
\end{proof}

Theorem \ref{mmm} shows that when the critical client $k$ does not exist at the NE of Game 1, we can derive a cooperative participation strategy under which clients fall into at most three categories: (i) the free riders at the NE of Game 1 who satisfy $\frac{\rho_n}{E_n} \leq B'_n$ still choose to be \emph{free riders} under the cooperative strategy, (ii) the free riders at the NE of Game 1 who satisfy $\frac{\rho_n}{E_n} > B'_n$ choose to be \emph{converted contributors} who perform model training with a positive amount of local data $x^{coop}$, and (iii) the contributors at the NE of Game 1 still choose to be \emph{contributors}.

Similar as Theorem 3, Theorem 4 shows that compared with the NE of Game 1 in Theorem 1, we have the following improved results under the cooperative participation strategy. First, the number of free riders is reduced from $m$ to $l-1$, where $l-1 \leq m$. Second, the amount of local data chosen by the converted contributors increases from $0$ to $x^{coop}>0$. The following corollary shows that under certain conditions, the cooperative strategy can reduce the number of free riders to $0$. 

\begin{corollary}
Under the cooperative strategy in Theorem 4, no client chooses to be a free rider if
$$\frac{\rho_1}{E_1}> B'_1 \triangleq \frac{2\sqrt{G((N-m)D)^3}}{m}.$$
Furthermore, all clients choose to be contributors, i.e., $x_n^{coop}=D, \forall n\in\mathcal{N}$, if 
$$\frac{\rho_1}{E_1}> O'_1 \triangleq \frac{D\sqrt{G}}{\frac{1}{\sqrt{(N-m)D}}-\frac{1}{\sqrt{ND}}}.$$
\end{corollary}

Similar as Corollary 2, Corollary 3 shows that when client 1's ratio $\frac{\rho_1}{E_1}$ is high, all clients choose to be (converted/partial) contributors under the cooperative participation strategy, and there is no free rider in cross-silo FL. 

In summary, Theorem 3 and Theorem 4 characterize the cooperative participation strategies that can minimize the number of free riders while increasing the amount of local data for model training in cross-silo FL, for both cases of the NE of Game 1 in Theorem 1. We next discuss how to enforce the cooperative participation strategy as the SPNE of Game 2 by a punishment strategy. 

\subsection{Subgame Perfect Nash Equilibrium}\label{4c}

In this section, we aim to derive the condition under which clients behave according to the cooperative participation strategy $\boldsymbol{x}^{coop}$ at the SPNE of Game 2. 

In the infinitely repeated game, we can enforce the cooperative participation strategy $\boldsymbol{x}^{coop}$ as the SPNE by a punishment strategy $\boldsymbol{x}^{pun}$. In this paper, we adopt Friedman punishment \cite{fri}, where clients play the NE of the stage game SPFL in Theorem 1 if any client deviates from the cooperative participation strategy (i.e., $\boldsymbol{x}^{pun}=\boldsymbol{x}^\ast$). Note that although the cooperative participation strategy can reduce the cost of each client in each time slot, i.e., $F_n(\boldsymbol{x}^{coop,t}) < F_n(\boldsymbol{x}^{\ast,t}), \forall n\in\mathcal{N}$, it is not an equilibrium strategy of the selfish participation game SPFL. In other words, given other clients' strategies under $\boldsymbol{x}^{coop}$, a client who is a (converted/partial) contributor has the incentive to decrease its chosen amount of local data for model training to reduce its cost. To punish the client who deviates from the cooperative participation strategy, other clients resort to the NE of the stage game SPFL. 

We use the one-stage deviation principle \cite{fri2} to characterize clients' behaviors at the SPNE. Specifically, each client $n\in\mathcal{N}$ plays the cooperative participation strategy at the SPNE, if its long-term discounted total cost under cooperation is no larger than that under deviation, i.e., 
\begin{equation}
\sum\limits_{t=0}^{\infty}\delta_n^tF_n(\boldsymbol{x}^{coop})\leq F_n(\boldsymbol{x}_n^{least})+\sum\limits_{t=1}^{\infty}\delta_n^tF_n(\boldsymbol{x}^{pun}).
\end{equation}
Here $\boldsymbol{x}_n^{least}$ is the participation strategy profile under which other clients keep cooperation while client $n$ deviates to minimize its own cost in one time slot. Hence we have 
\begin{equation}\label{eq19}
\delta_n \geq \delta_n^{th}(\boldsymbol{x}^{coop}) \triangleq \frac{F_n(\boldsymbol{x}^{coop})-F_n(\boldsymbol{x}_n^{least})}{F_n(\boldsymbol{x}^{pun})-F_n(\boldsymbol{x}_n^{least})},
\end{equation}
where $\delta_n^{th}(\boldsymbol{x}^{coop})$ is the threshold discount factor of client $n$ to play the cooperative strategy profile $\boldsymbol{x}^{coop}$. If each client $n\in\mathcal{N}$ is patient enough, i.e., the discount factor satisfies $\delta_n \geq \delta_n^{th}(\boldsymbol{x}^{coop})$, all clients will play the cooperative participation strategy at the SPNE. For simplicity, we assume that all clients have the same discount factor, i.e., $\delta_n=\delta, \forall n \in\mathcal{N}$, as in \cite{book}.

Next we show how to enforce the cooperative strategies derived in Theorem 3 and Theorem 4 as the SPNE in the repeated game. 

\begin{theorem}\label{sss}
Consider the following strategy profile: all clients choose the cooperative participation strategy $\boldsymbol{x}^{coop}$ (calculated in Theorem \ref{kkk} or Theorem \ref{mmm}) until a client deviates, and then all clients play the NE $\boldsymbol{x}^\ast$ of the stage game SPFL (calculated in Theorem 1) in all future time slots. Such a strategy profile is the SPNE of Game 2 if 
\begin{equation}\label{eq20}
\delta \geq \delta^{th}(\boldsymbol{x}^{coop}) \triangleq \max \{\delta_n^{th}(\boldsymbol{x}^{coop}): \forall n\in \mathcal{N}\}.
\end{equation}
\end{theorem}

\begin{proof}
See Appendix G in the supplementary material.
\end{proof}

Theorem 5 shows that when clients are patient enough, i.e., $\delta \geq \delta^{th}(\boldsymbol{x}^{coop})$, clients play the cooperative strategy $\boldsymbol{x}^{coop}$ at the SPNE, and no client has the incentive to decrease its chosen amount of local data.

Note that the cooperative strategy $\boldsymbol{x}^{coop}$ characterized in Theorem 3 and Theorem 4 is not unique, since converted contributors can choose any cooperative participation level $x^{coop} \leq x_l^{th}$. Therefore, the SPNE in Theorem 5 is not unique. Next we design an algorithm to calculate the optimal SPNE that can minimize the number of free riders while maximizing the amount of local data for model training.

\subsection{Optimal SPNE}

In this section, we aim to derive the optimal SPNE that can minimize the number of free riders while maximizing the amount of local data for model training. In the following, we first define the optimal SPNE of Game 2, and then present the procedures to compute the optimal SPNE. 

We first analyze the optimal SPNE $\boldsymbol{x}^{os}=\{x_n^{os}: \forall n\in\mathcal{N}\}$ for the case where the critical client $k$ exists at the NE of Game 1 (as shown in Case I in Theorem 1). Since the optimal SPNE minimizes the number of free riders, it follows the solution structure $\boldsymbol{x}^{coop}$ shown in (11) in Theorem 3, and can be written as:
\begin{equation}\label{eq21}
x_n^{os}=
\left\{
\begin{aligned}
& 0, &\mbox{ if } &n < l;\\
& x_{cc}^{os}, &\mbox{ if } &l \leq n < k;\\
& x_k^{*}, &\mbox{ if } &n=k;\\
& D, &\mbox{ if } &n>k.
\end{aligned}
\right.
\end{equation}
Here $x_{cc}^{os}$ is converted contributors' strategy under the optimal SPNE which satisfies $0 \leq x_{cc}^{os} \leq x_l^{th}$, and $x_l^{th}$ is defined in \eqref{coopk}, \eqref{imp}. We can see from \eqref{eq21} that only converted contributors' strategy $x_{cc}^{os}$ in the optimal SPNE $\boldsymbol{x}^{os}$ is unknown.

Now we derive $x_{cc}^{os}$. Since $\boldsymbol{x}^{os}$ is the SPNE of Game 2, it satisfies the one-stage deviation principle described in Theorem 5, i.e., $\delta \geq \delta^{th}(\boldsymbol{x}^{os})$. Furthermore, the optimal SPNE maximizes the amount of local data for model training, i.e., $(k-l) x_{cc}^{os} + x_k^\ast +(N-k)D$. In summary, $x_{cc}^{os}$ is the optimal solution to the following optimization problem with respect to $x_{cc}$: 
\begin{equation}\label{prob:osk}
\begin{aligned}
\mbox{max}~ & ~~  (k-l) x_{cc} + x_k^\ast +(N-k)D \\
\mbox{s.t.}~ & ~~  0 \leq x_{cc} \leq x_l^{th}, \\
~ & ~~ \delta \geq \delta^{th}(\boldsymbol{x}^{os}). 
\end{aligned}
\end{equation}

Although the above optimization problem has a simple linear objective with respect to $x_{cc}$, we cannot derive the closed-form expression for the optimal solution $x_{cc}^{os}$ since the constraint $\delta \geq \delta^{th}(\boldsymbol{x}^{os})$ is complicated due to the following two reasons. First, similar to \eqref{eq20}, $\delta^{th}(\boldsymbol{x}^{os})$ takes the maximum value among $\delta_n^{th}(\boldsymbol{x}^{os}), \forall n\in\mathcal{N}$, and the max operator is a non-differentiable function. Second, similar to \eqref{eq19}, $\delta_n^{th}(\boldsymbol{x}^{os})$ involves not only $\boldsymbol{x}^{os}$ but also $\boldsymbol{x}_n^{least}$, and hence the function $\delta_n^{th}(\boldsymbol{x}^{os})$ changes with $x_{cc}$ in a highly non-linear manner (as shown in Appendix G in the supplementary material). In order to find the optimal solution $x_{cc}^{os}$, we can use greedy algorithms to solve the univariate optimization problem \eqref{prob:osk}. 

\begin{algorithm}[t]
\LinesNumbered
\SetAlgoLined
\begin{small}
\KwIn{$\rho_n, E_n, C_n, \forall n\in\mathcal{N}$, $G, N, D, p$}
\KwOut{The optimal SPNE $\boldsymbol{x}^{os}=\{x_n^{os}: \forall n\in\mathcal{N}\}$}
\uIf{$\exists k\in\mathcal{N}$ such that $(N-k)D \leq h_k \leq (N+1-k)D$}{
Find the client $l$ where $l=\min\{n\in\mathcal{N}: 1 \leq n \leq k-1, \frac{\rho_n}{E_n} > B_n\}$\;
Compute the maximum amount of local data $x_l^{th}$ that converted contributors can choose by solving (13)\;
Solve the optimization problem \eqref{prob:osk}\;
Announce to each client $n\in\mathcal{N}$ the optimal SPNE:
$
x_n^{os}=
\left\{
\begin{aligned}
& 0, &\mbox{ if } &n < l;\\
& x_{cc}^{os}, &\mbox{ if } &l \leq n < k;\\
& x_k^{*}, &\mbox{ if } &n=k;\\
& D, &\mbox{ if } &n>k.
\end{aligned}
\right.
$
}\ElseIf{
$\exists m\in\mathcal{N}$ such that $h_m < (N-m)D < h_{m+1}$}{
Find the client $l$ where $l=\min\{n\in\mathcal{N}: 1 \leq n \leq m, \frac{\rho_n}{E_n} > B_n^{'}\}$\; 
Compute the maximum amount of local data $x_l^{th}$ that converted contributors can choose by solving (17)\;
Solve the optimization problem \eqref{prob:osm}\;
Announce to each client $n\in\mathcal{N}$ the optimal SPNE:
$
x_n^{os}=
\left\{
\begin{aligned}
& 0, &\mbox{ if } &n < l;\\
& x_{cc}^{os}, &\mbox{ if } &l \leq n \leq m;\\
& D, &\mbox{ if } &n>m.
\end{aligned}
\right.
$

}
\end{small}
\caption{An Algorithm to Calculate the Optimal SPNE of Game 2}\label{al2}
\end{algorithm}

Next we analyze the optimal SPNE $\boldsymbol{x}^{os}=\{x_n^{os}: \forall n\in\mathcal{N}\}$ for the case where the critical client $k$ does not exist at the NE of Game 1 (as shown in Case II in Theorem 1). Similarly, the optimal SPNE follows the solution structure $\boldsymbol{x}^{coop}$ shown in (15) in Theorem 4, and can be written as:
\begin{equation}\label{eq22}
x_n^{os}=
\left\{
\begin{aligned}
& 0, &\mbox{ if } &n < l;\\
& x_{cc}^{os}, &\mbox{ if } &l \leq n \leq m;\\
& D, &\mbox{ if } &n>m.
\end{aligned}
\right.
\end{equation}
Here $x_{cc}^{os}$ is converted contributors' strategy under the optimal SPNE which satisfies $0 \leq x_{cc}^{os} \leq x_l^{th}$, and $x_l^{th}$ is defined in \eqref{coopm}, \eqref{fun}. We can see from \eqref{eq22} that only converted contributors' strategy $x_{cc}^{os}$ in the optimal SPNE $\boldsymbol{x}^{os}$ is unknown. 
We can derive $x_{cc}^{os}$ by solving the following optimization problem with respect to $x_{cc}$: 
\begin{equation}\label{prob:osm}
\begin{aligned}
\mbox{max}~ & ~~  (m-l+1) x_{cc}  +(N-m)D \\
\mbox{s.t.}~ & ~~  0 \leq x_{cc} \leq x_l^{th}, \\
~ & ~~ \delta \geq \delta^{th}(\boldsymbol{x}^{os}). 
\end{aligned}
\end{equation}
Similarly, we can use greedy algorithms to solve the above univariate optimization problem.

Algorithm 2 shows the procedures for the central server to calculate and announce the optimal SPNE of Game 2.\footnote{Note that the change in the number of clients participating in the federated learning process due to battery power and occasional disconnections in a mobile computing environment will not affect Algorithm 2.} To calculate the optimal SPNE of Game 2, the central server first needs to check whether the critical client $k$ exists (Line 1 and Line 6 in Algorithm 2). Then the central server calculates the minimum number of free riders in cross-silo FL under the optimal SPNE, which depends on each client's valuation-computation ratio $\frac{\rho_n}{E_n}$ and the value of $B_n$ or $B'_n$ (Line 2 and Line 7 in Algorithm 2). 
The central server then calculates the maximum amount of local data $x_l^{th}$ that the converted contributors can choose under the optimal SPNE, by solving (13) or (17) (Line 3 and Line 8 in Algorithm 2). Finally the central server solves the optimization problem \eqref{prob:osk} or \eqref{prob:osm} and announces the optimal SPNE strategy $x_n^{os}$ to each client $n\in\mathcal{N}$ (Lines 4-5 and Lines 9-10 in Algorithm 2). 

Algorithm 2 has a linear complexity of $\mathcal{O}(D)$. More specifically, checking whether the critical client $k$ exists (Line 1 and Line 6) has a complexity of $\mathcal{O}(N)$. Calculating the number of free riders (Line 2 and Line 7) has a complexity of $\mathcal{O}(N)$. Calculating the maximum amount of local data $x_l^{th}$ (Line 3 and Line 8) has a complexity of $\mathcal{O}(\log D)$. Solving the optimization problem (Line 4 and Line 9) has a complexity of $\mathcal{O}(D)$. Since the number of clients in cross-silo FL is always small, we usually have $N < D$ \cite{div}. Therefore, the overall complexity of Algorithm 2 is $\mathcal{O}(D)$. 


\section{Simulation Results}\label{iiiii}

In this section, we first evaluate the performance of our proposed optimal SPNE and show its scalability. We then compare the optimal SPNE with two related methods.

We conduct simulations on the dataset MNIST using FedAvg \cite{first}. The dataset MNIST contains $60000$ figures of handwritten digits, which has been widely used in many existing works on FL, e.g., \cite{res-c1}. We assume that there are $N = 100$ clients since the number of clients in cross-silo FL is always small \cite{div}. Because clients in cross-silo FL are usually companies or organizations which have plenty of data, we assume that each client has a local data set with $D = 10000$ local data samples randomly chosen from MNIST \cite{res-c2}. We assume that the number of iterations in one cross-silo FL process is $G=50$ under which the FedAvg algorithm can converge \cite{res-c3}. For simplicity, we assume that the computation cost coefficient of each client is $E = 0.985 \times 10^{-8}$ CNY.\footnote{According to Amazon Lambda Pricing \cite{res-c4}, each request is charged $1.36 \times 10^{-6}$ CNY, and the price per millisecond is $1.42 \times 10^{-8}$ CNY. Our assumption is roughly consistent with \cite{res-c4}.} We assume that the valuation for global model accuracy of client $n$ is $\rho_n = n \times \frac{100}{N}$, and the discount factor of each client is $\delta=0.8$.

\subsection{Performance and Scalability}

In this subsection, we evaluate the performance of our proposed optimal SPNE in terms of two important metrics: the number of reduced free riders $N_f$ and the total data contribution ratio $R_d$. The number of reduced free riders is the difference between the number of free riders at the NE of the stage game and the number of free riders at the proposed optimal SPNE. The total data contribution ratio is the ratio between the total amount of local data for model training at the optimal SPNE and that at the NE of the stage game, i.e., $R_d = \sum_{n\in\mathcal{N}}x_n^{os} / \sum_{n\in\mathcal{N}}x_n^*$. To show the scalability of our proposed SPNE, we will show how the above two metrics change with the number of clients $N$ and the size of each client's local data set $D$. 

\begin{figure}[t]
 \centering
 \begin{minipage}[h]{0.48\linewidth}\label{figN1}
 \centering
 \includegraphics[width=1.02\textwidth]{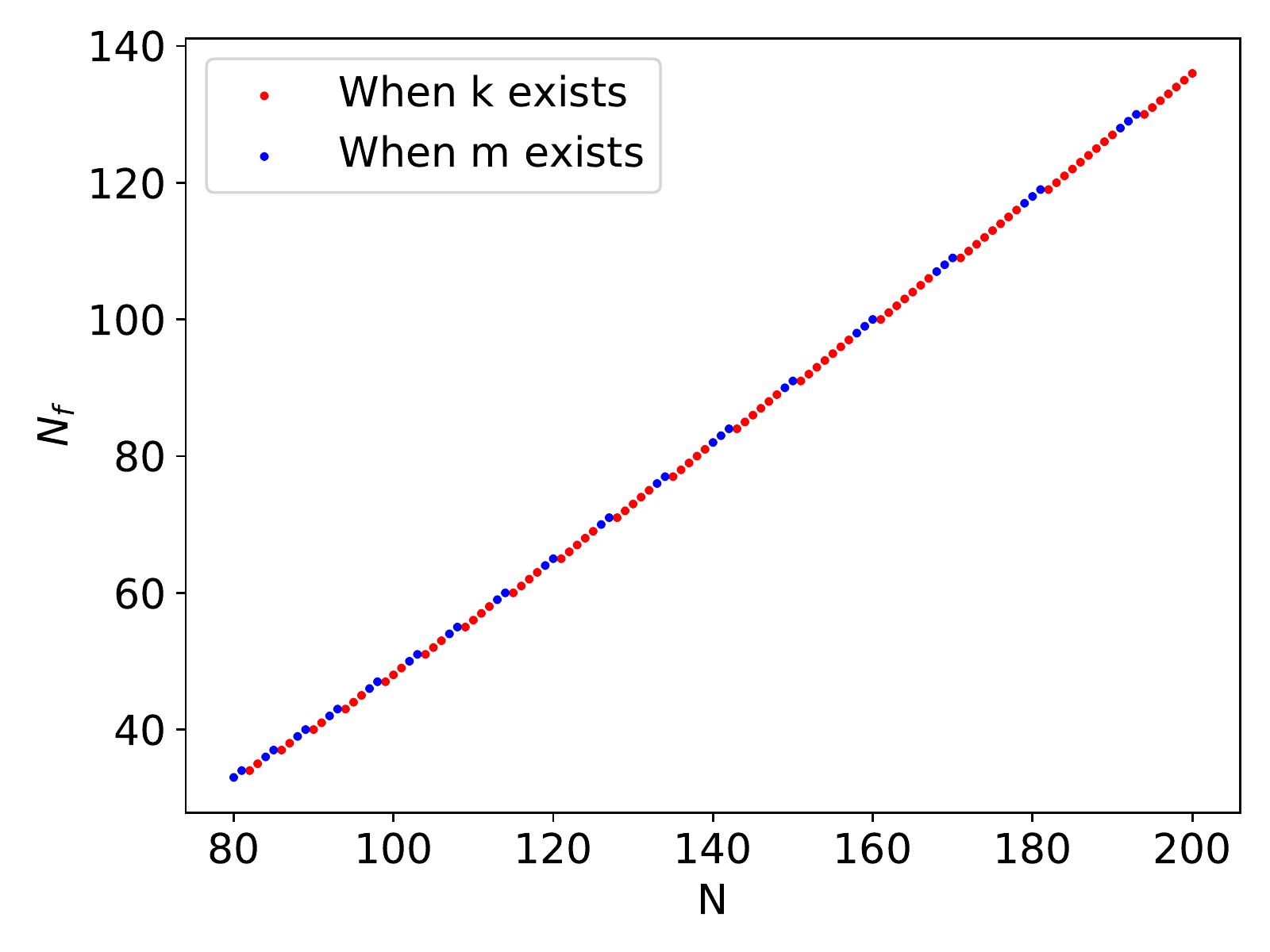}
 \caption{$N_f$ under different $N$}
 \end{minipage}%
 \begin{minipage}[h]{0.48\linewidth}\label{figN2}
 \centering
 \includegraphics[width=1.02\textwidth]{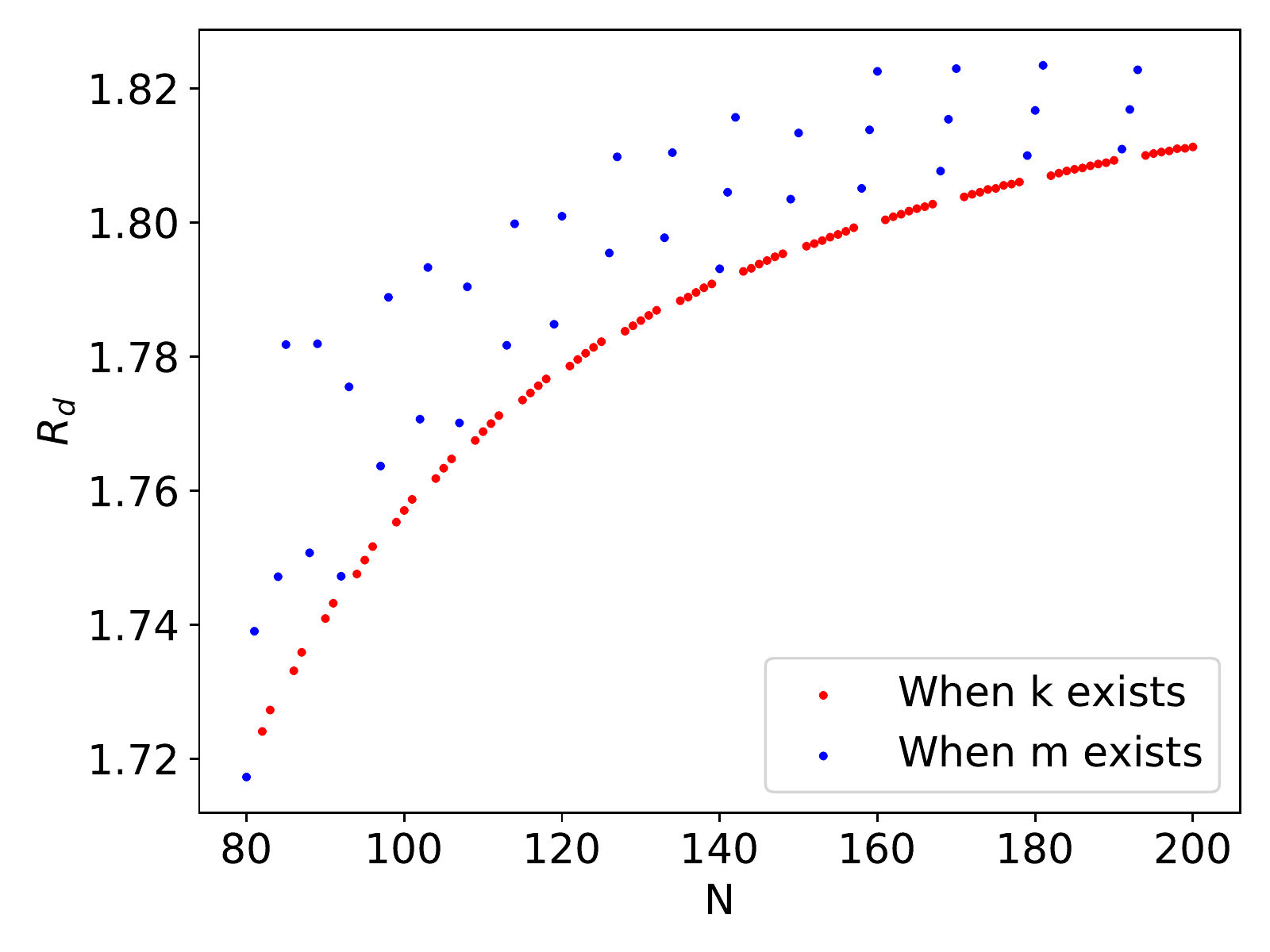}
 \caption{$R_d$ under different $N$}
 \end{minipage}%
\end{figure}

\textbf{The impact of the number of clients $N$.} We show how the number of clients $N$ affects the number of reduced free riders $N_f$ and the data contribution ratio $R_d$ in Fig. 3 and Fig. 4 respectively.

Fig. 3 shows that the number of reduced free riders increases with $N$. In Fig. 3, the red dots represent the case where the critical client $k$ exists, and the blue dots represent the case where the critical client $k$ does not exist. As $N$ increases, at the NE of the stage game SPFL, the global model trained by the contributors with high valuation-computation ratios is good enough, and hence the contributors with low valuation-computation ratios may become free riders to reduce their total costs. Thus more clients will choose to be free riders with the increase of $N$. Our proposed optimal SPNE can convert most free riders into \emph{converted contributors} and thus solves the free-rider problem effectively even when $N$ is large.

Fig. 4 shows how the data contribution ratio $R_d$ changes with $N$. Specifically, when the critical client $k$ exists, the data contribution ratio $R_d$ generally increases with $N$. As $N$ increases, the optimal SPNE can make the the majority of free riders choose to be converted contributors and thus increase the data contribution ratio. When $k$ does not exist, under the same number of free riders, the data contribution ratio decreases with $N$ (the blue dots). The reason is that under the same number of free riders, a larger $N$ indicates more contributors, and hence a larger amount of local data for model training, which leads to a global model with a higher model accuracy at the NE of stage game. Therefore, the converted contributors under the optimal SPNE choose a smaller amount of local data for model training to reduce the computation cost. Our proposed optimal SPNE can increase the amount of total data for model training compared with the NE of the stage game, even when $N$ is large.

\emph{In summary, when more clients participate in cross-silo FL, the selfish participation behavior leads to more free riders. Our proposed optimal SPNE can effectively reduce the number of free riders by up to $99.3\%$,\footnote{In Fig. 3, when $N=200$, we can calculate that $k=138$ and $l=2$. In other words, there are 137 free riders at the NE of the stage game while there is only one free rider at the optimal SPNE. We reduce the number of free riders by $99.3\%$, i.e., $\frac{137-1}{137}=99.3 \%$.} and increase the amount of local data chosen by clients to perform model training by up to $82.3\%$.\footnote{In Fig. 4, when $N=181$, we can calculate that $m=120$. Furthermore, the total amount of local data for model training at the optimal SPNE is $1112299$, and the total amount of local data for model training at the NE of the stage game is $610000$. In this case, the total data contribution ratio is $\frac{1112299}{610000}=1.823$. Equivalently, the amount of local data chosen by clients to perform model training is increased by $82.3\%$.}
Our proposed optimal SPNE performs well even when $N$ is large, which shows the scalability of our work.}

\begin{figure}[t]
 \centering
 \begin{minipage}[h]{0.48\linewidth}\label{figD1}
 \centering
 \includegraphics[width=1.028\textwidth]{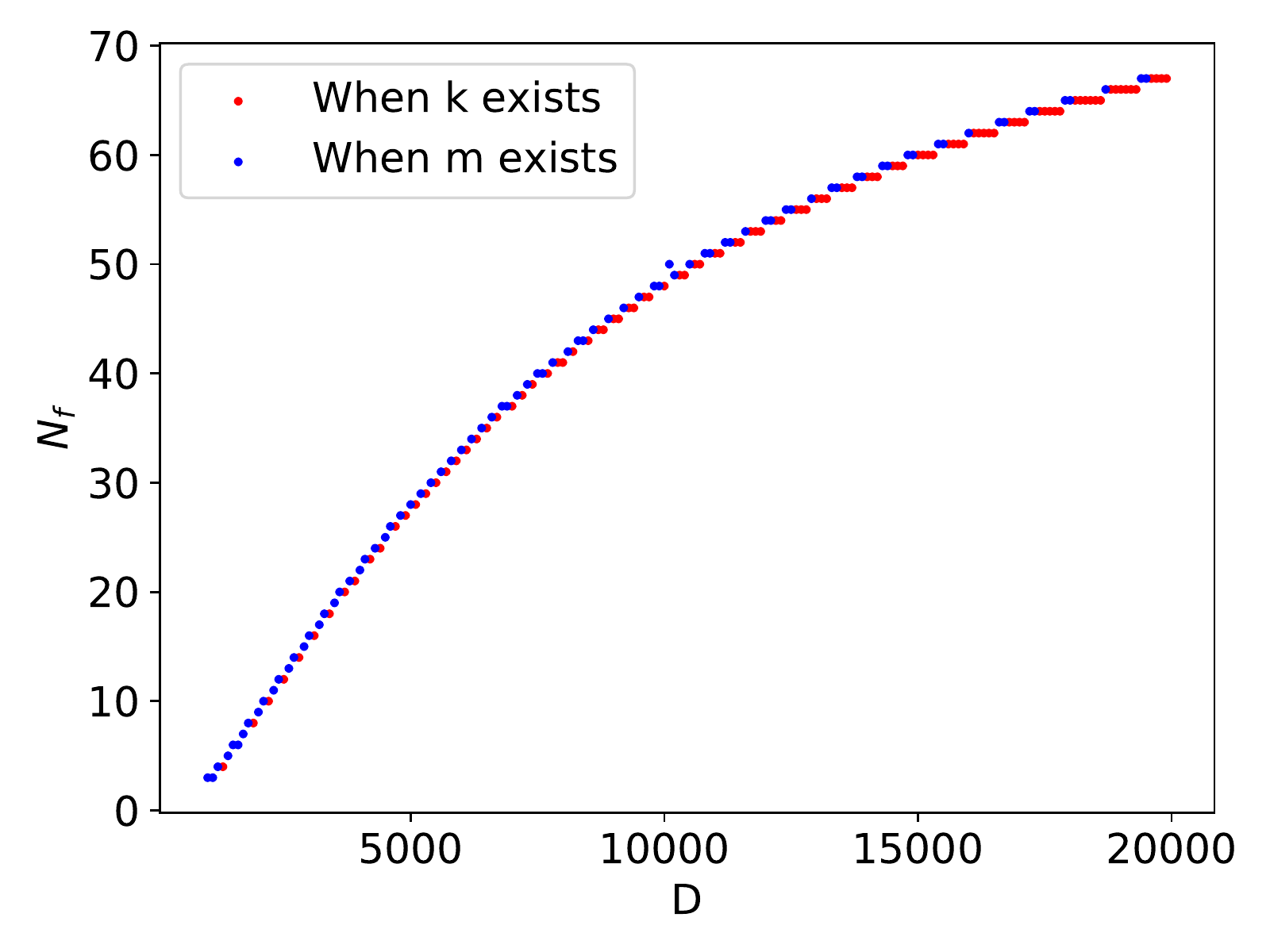}
 \caption{$N_f$ under different $D$}
 \end{minipage}%
 \begin{minipage}[h]{0.48\linewidth}\label{figD2}
 \centering
 \includegraphics[width=1.028\textwidth]{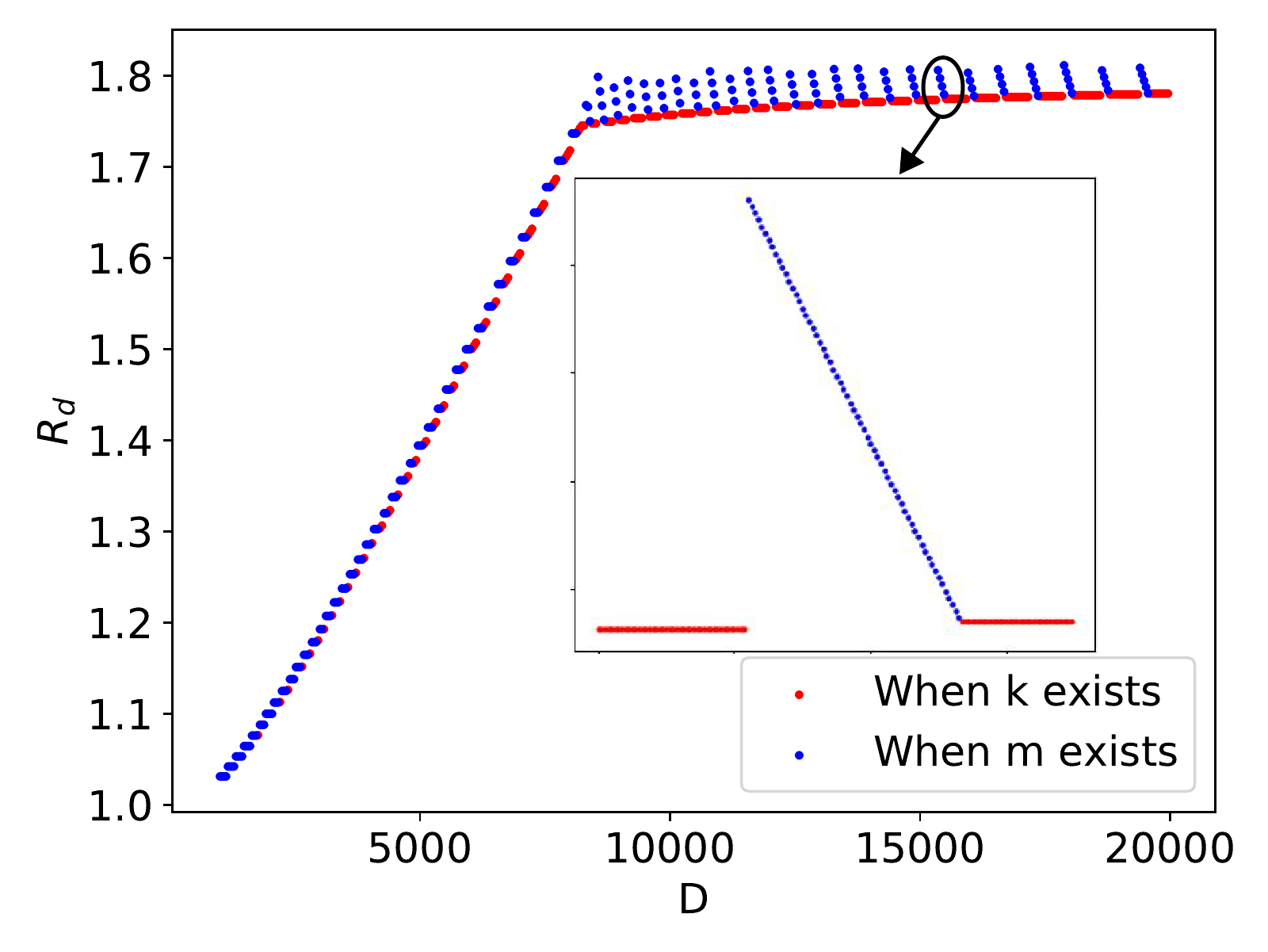}
 \caption{$R_d$ under different $D$}
 \end{minipage}%
\end{figure}

\textbf{The impact of the size of each client's local data set $D$.} We show how the parameter $D$ affects the number of reduced free riders $N_f$ and the data contribution ratio $R_d$ in Fig. 5 and Fig. 6 respectively.

Fig. 5 shows that the number of reduced free riders increases with $D$. When $D$ is small, to achieve a good global model with a high model accuracy requires many clients to choose to be contributors to perform model training, and hence there are few free riders at the NE of the stage game. When $D$ is large, a small number of contributors can train a good global model since each client has a large amount of local data to perform model training, and hence there are many free riders at the NE of the stage game. Our proposed optimal SPNE can convert most free riders into \emph{converted contributors} and thus solves the free-rider problem effectively, even when $D$ is large.

Fig. 6 shows how the data contribution ratio $R_d$ changes with $D$. Specifically, the data contribution ratio $R_d$ generally increases with $D$. When $D$ is small, the curve increases quickly with $D$, and when $D$ is large, $R_d$ increases slowly in a zigzag pattern. When $D$ is small, there are few free riders at the NE of the stage game, and hence there is no much room to increase the amount of local data for model training. As $D$ increases, under the optimal SPNE, converted contributors choose a positive amount of local data to perform model training, which increases the amount of chosen data by up to $81.2\%$. The subfigure in Fig. 6 shows the zigzag shape of the curve under the same number of free riders. Note that the cooperative strategy of the critical client $k$ is $x_k^{*}$ which decreases with $D$ as calculated in Theorem 1, while the cooperative strategy $x^{coop}$ of the converted contributors increases with $D$. Hence the total amount of local data remains unchanged at the beginning of the curve. When $k$ does not exist, under the same number of free riders, the data contribution ratio decreases with $D$. The reason is that a larger $D$ leads to less accuracy loss, so the converted contributors will use less data for training to reduce the computation cost. Our proposed optimal SPNE can increase the amount of total data for model training compared with the NE of the stage game, even when $D$ is large.

\emph{In summary, when each client in cross-silo FL has more local data samples, the selfish participation behavior leads to more free riders. Our proposed optimal SPNE can effectively reduce the number of free riders by up to $98.5\%$, and increase the data contribution by up to $81.2\%$. Our proposed optimal SPNE performs well even when $D$ is large, which shows its scalability.}

We also evaluate the performances of our proposed optimal SPNE under other system parameters such as the number of iterations $G$, the computation cost coefficient $E$, and the discount factor $\delta$. Due to space limit, we put detailed discussions in Appendix H.

\subsection{Comparisons with Related Methods}

In this subsection, we compare our proposed solution with two related methods. One is a free-rider detection method called DAGMM proposed in \cite{DAGMM}, which can effectively detect free riders and only aggregate the updates of (partial) contributors at the server. The other is a contract-based incentive mechanism for FL proposed in \cite{incen1}. In the following, we show the comparisons between our method and the above two methods under different system parameters (i.e., the number of iterations $G$ and the discount factor $\delta$).\footnote{We also show the comparisons under different values of $N, D$ and $E$ in Appendix H.}

\begin{figure}[t]
\centering
\begin{minipage}[h]{0.48 \linewidth}
\centering
\includegraphics[width=1.04\textwidth]{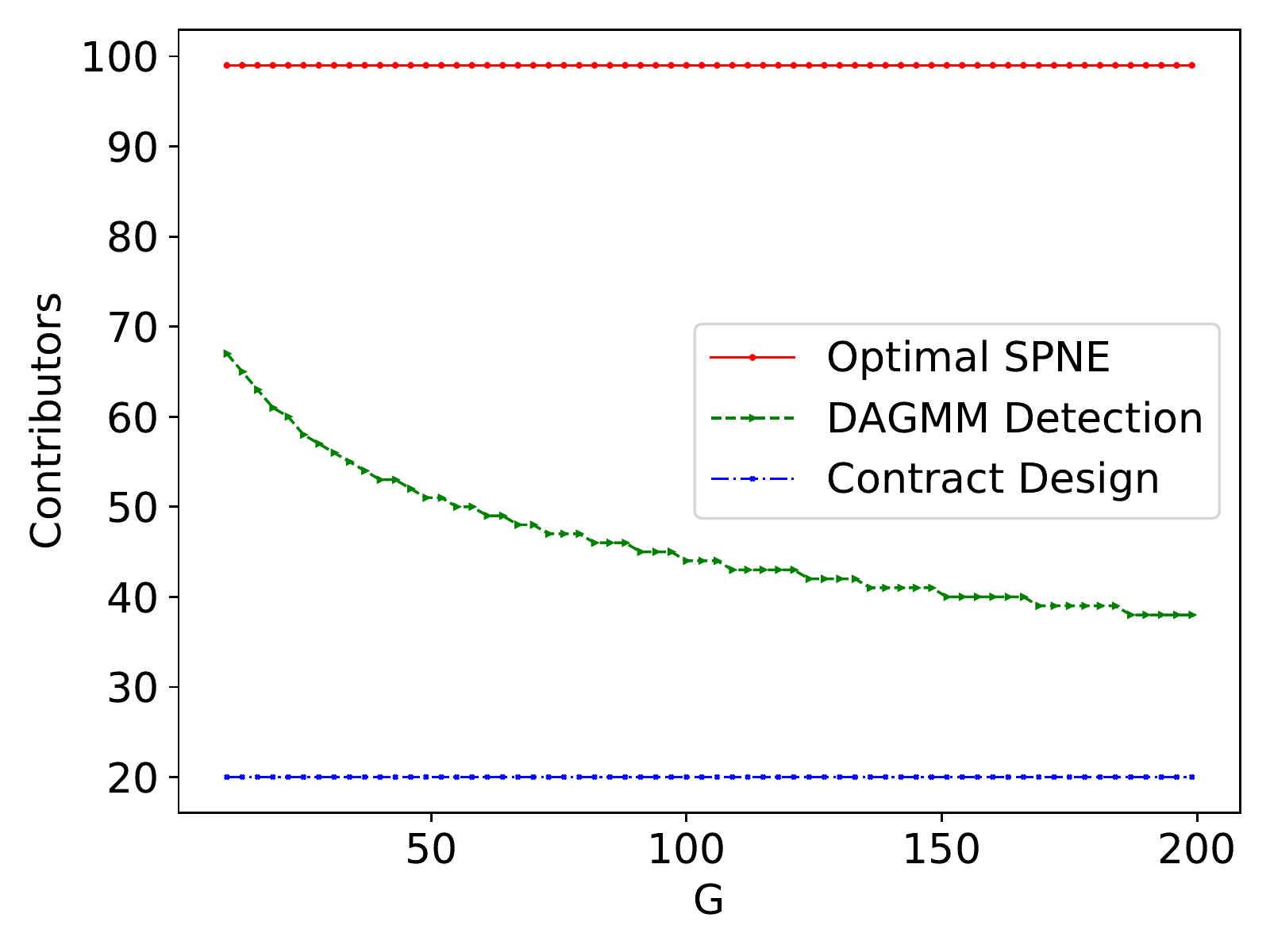}
  \caption{The number of contributors under different $G$}\label{CG}
\end{minipage}
\begin{minipage}[h]{0.48 \linewidth}
\centering
\includegraphics[width=1.04\textwidth]{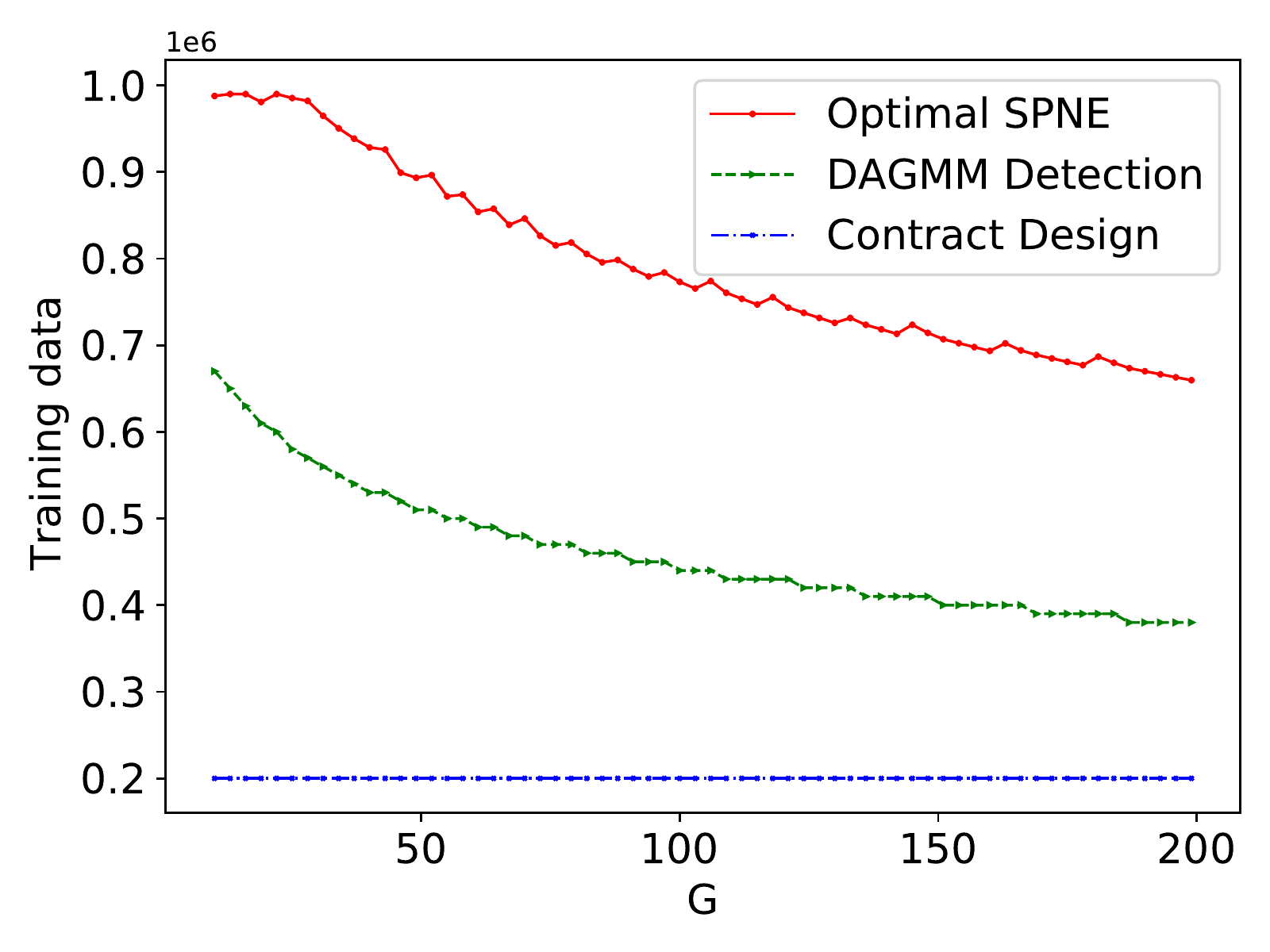}
  \caption{The total amount of training data under different $G$}\label{TG}
\end{minipage}
\end{figure}

Fig. \ref{CG} and Fig. \ref{TG} show the comparisons of the three methods under different values of $G$. A larger $G$ means more iterations in one cross-silo FL process, in which case clients can achieve a global model with a higher model accuracy. Fig. \ref{CG} shows that among the three methods, our method achieves the largest number of contributors under different values of $G$. Specifically, under our method, the number of contributors remains unchanged and almost all clients participate in the training. Under the DAGMM detection method, the number of contributors decreases with $G$. Under the contract design method, only $20\%$ clients participate in the training. The reason is that under a larger $G$, the global model trained by contributors with high valuation-computation ratios is good enough, and hence contributors with low valuation-computation ratios for the global model accuracy may change to be free riders. Therefore, the number of contributors under the DAGMM detection method decreases with $G$. Under the contract design method, since the change in $G$ does not affect the number of clients that are provided with a positive contract item, the number of contributors remains unchanged. Our proposed method can effectively motivate almost all free riders to become converted contributors. 

Fig. \ref{TG} shows that our method has a larger amount of training data compared with the other two methods. Specifically, under our method, the total amount of training data first remains unchanged, and then gradually decreases. The reason is that a smaller $G$ leads to a smaller threshold discount factor of clients. So when $G$ is small, converted contributors will participate in training with all local data. As $G$ increases, converted contributors can use only part of local data in training. Under the DAGMM detection method, since the number of contributors decreases with $G$, the total amount of training data also decreases with $G$. Under the contract design method, since the change in $G$ does not affect the number of contributors, the total amount of training data remains unchanged.

\begin{figure}[t]
\centering
\begin{minipage}[h]{0.49 \linewidth}
\centering
\includegraphics[width=1.05\textwidth]{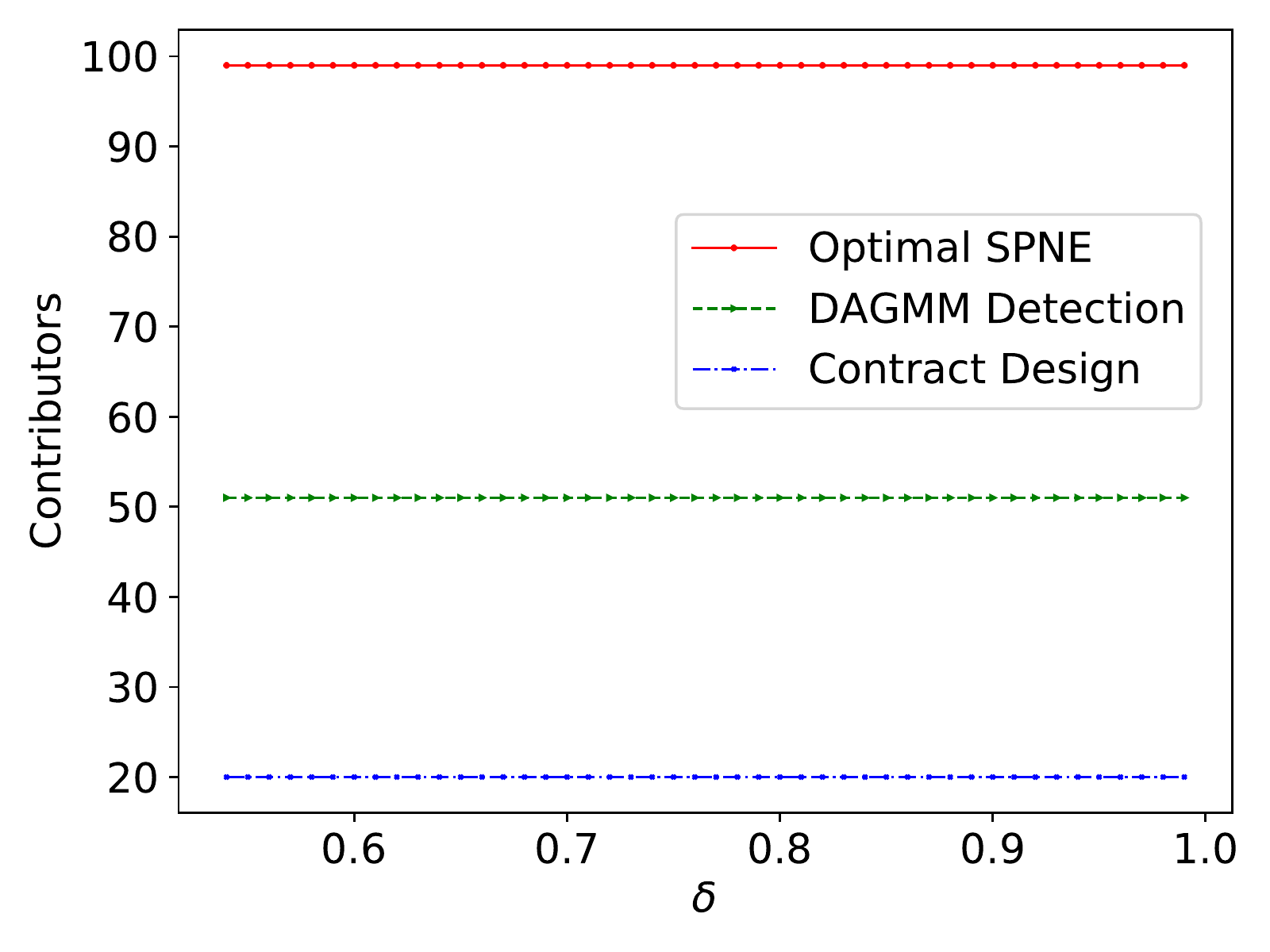}
  \caption{The number of contributors under different $\delta$}\label{Cd}
\end{minipage}
\begin{minipage}[h]{0.49 \linewidth}
\centering
\includegraphics[width=1.05\textwidth]{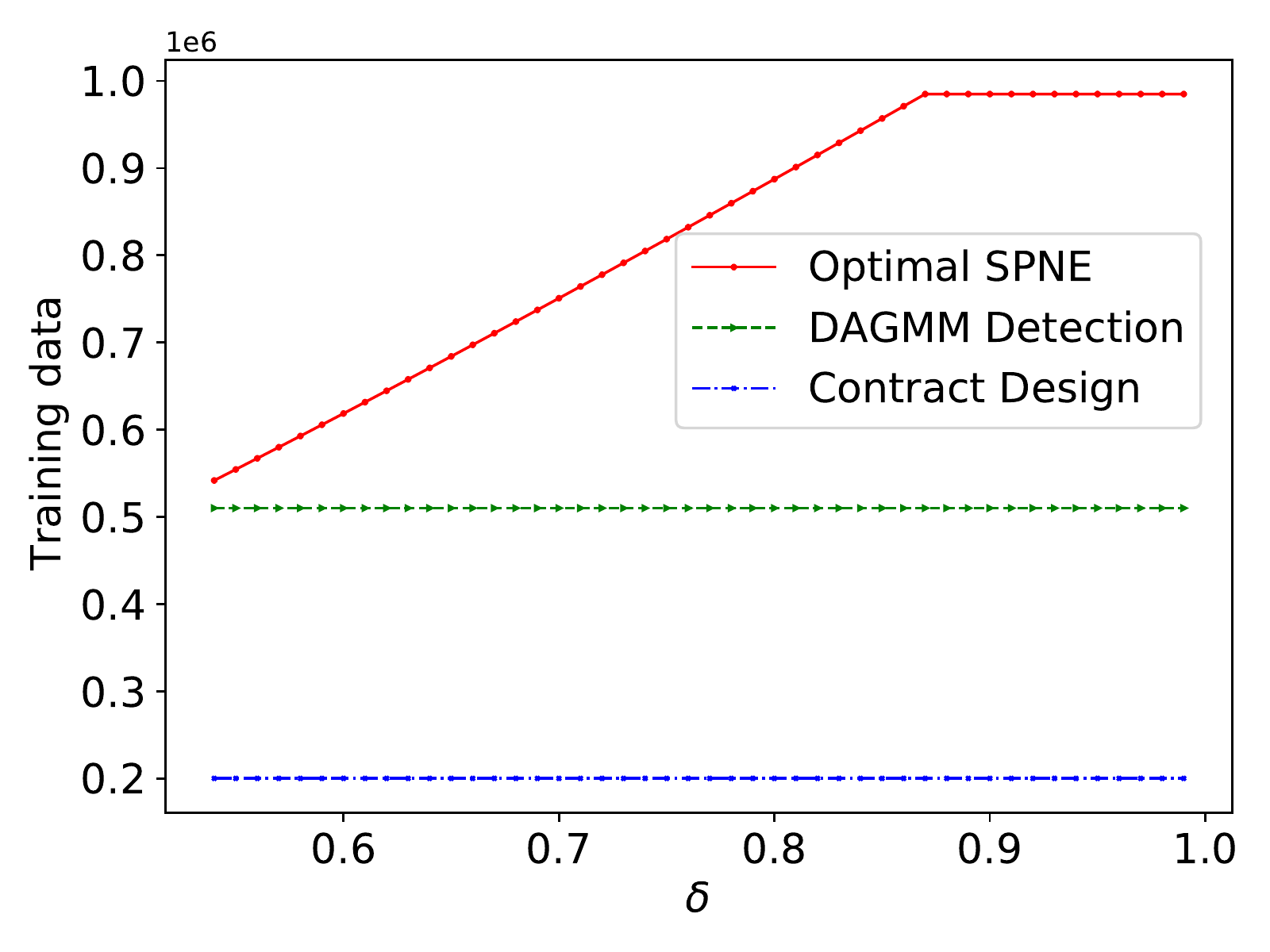}
  \caption{The total amount of training data under different $\delta$}\label{Td}
\end{minipage}
\end{figure}

Fig. \ref{Cd} and Fig. \ref{Td} show the comparisons of the three methods under different values of $\delta$. Fig. \ref{Cd} shows that the number of contributors under the three methods does not change with the discount factor $\delta$. The reason is that the number of free riders depends on the valuation-computation ratios $\{\frac{\rho_n}{E_n}: \forall n\in\mathcal{N}\}$, the bounds $\{B_n: \forall n\in\mathcal{N}\}$, and the local data set sizes $\{D_n: \forall n\in\mathcal{N}\}$, and is independent of $\delta$. Our method can motivate almost all free riders to become converted contributors, and hence our method has distinct advantage over the other two methods.

Fig. \ref{Td} shows that our method has a larger amount of training data compared with the other two methods. Specifically, under our method, the total amount of training data first increases with $\delta$, and then remains to be a constant when $\delta$ is larger than $0.87$. A larger discount factor $\delta$ indicates that converted contributors are more patient, and hence they are more willing to choose a larger amount of local data to reduce the long-term discounted total cost. When converted contributors are patient enough, i.e., $\delta \geq 0.87$, they will perform model training with all their local data. Under the DAGMM detection method and the contract design method, the numbers of contributors remain unchanged, and the contributors under these two methods participate in training with all their local data. Thus the amounts of training data of these two methods also remain unchanged.

\section{Conclusion}\label{iiiiii}

In this work, we analyze clients' selfish participation behaviors in cross-silo FL. We model the interactions among clients in a single cross-silo FL process as a selfish participation game and derive its unique Nash equilibrium. We show that clients' selfish participation behaviors lead to free riders at equilibrium. The existence of free riders hampers clients' long-term participation in cross-silo FL. We model clients' interactions in the long-term cross-silo FL processes as a repeated game. We derive the optimal SPNE that can minimum the number of free riders while maximizing the amount of local data for model training. Simulation results show that our derived optimal SPNE can effectively reduce the number of free riders by up to $99.3\%$ and increase the amount of data for training by up to $82.3\%$. For future work, it is interesting to analyze the interactions among clients in the incomplete information scenario where clients do not know the participation strategies of other clients or in the asymmetric information scenario where the valuation for global model accuracy and the computation cost coefficient are the private information of each client.

\ifCLASSOPTIONcaptionsoff
  \newpage
\fi

\bibliographystyle{IEEEtran}
\bibliography{reference}

\begin{thebibliography}{10}
\providecommand{\url}[1]{#1}
\csname url@samestyle\endcsname
\providecommand{\newblock}{\relax}
\providecommand{\bibinfo}[2]{#2}
\providecommand{\BIBentrySTDinterwordspacing}{\spaceskip=0pt\relax}
\providecommand{\BIBentryALTinterwordstretchfactor}{4}
\providecommand{\BIBentryALTinterwordspacing}{\spaceskip=\fontdimen2\font plus
\BIBentryALTinterwordstretchfactor\fontdimen3\font minus
  \fontdimen4\font\relax}
\providecommand{\BIBforeignlanguage}[2]{{%
\expandafter\ifx\csname l@#1\endcsname\relax
\typeout{** WARNING: IEEEtran.bst: No hyphenation pattern has been}%
\typeout{** loaded for the language `#1'. Using the pattern for}%
\typeout{** the default language instead.}%
\else
\language=\csname l@#1\endcsname
\fi
#2}}
\providecommand{\BIBdecl}{\relax}
\BIBdecl

\bibitem{5G}
H.~Li, K.~Ota, and M.~Dong, ``Learning {IoT} in edge: Deep learning for the
  internet of things with edge computing,'' \emph{IEEE network}, vol.~32,
  no.~1, pp. 96--101, 2018.

\bibitem{google}
Y.~Zhan, P.~Li, Z.~Qu, D.~Zeng, and S.~Guo, ``A learning-based incentive
  mechanism for federated learning,'' \emph{IEEE Internet of Things Journal},
  vol.~7, no.~7, pp. 6360--6368, 2020.

\bibitem{FL}
Y.~Huang, L.~Chu, Z.~Zhou, L.~Wang, J.~Liu, J.~Pei, and Y.~Zhang,
  ``Personalized cross-silo federated learning on non-iid data,'' in
  \emph{Proceedings of the AAAI Conference on Artificial Intelligence}, 2021.

\bibitem{div}
P.~Kairouz, H.~B. McMahan, B.~Avent, A.~Bellet, M.~Bennis, A.~N. Bhagoji,
  K.~Bonawitz, Z.~Charles, G.~Cormode, R.~Cummings \emph{et~al.}, ``Advances
  and open problems in federated learning,'' \emph{Foundations and Trends in
  Machine Learning}, vol.~14, pp. 1--210, 2021.

\bibitem{webank}
L.~Insurance, ``Swiss {Re} and {Chinese} firm {WeBank} partner to explore use
  of {AI} in reinsurance,''
  \url{https://www.lifeinsuranceinternational.com/news/swiss-re-webank/},
  Accessed Apr. 11, 2021.

\bibitem{hospi}
E.~research project, ``Mammogram assessment with {NVIDIA Clara} federated
  learning,''
  \url{https://blogs.nvidia.com/blog/2020/04/15/federated-learning-mammogram-assessment/},
  Accessed Apr. 11, 2021.

\bibitem{mello}
C.~Hale, ``The {MELLODDY} project,''
  \url{https://www.fiercebiotech.com/special-report/melloddy-project}, Accessed
  Apr. 11, 2021.

\bibitem{incen2}
M.~Tang and V.~W. Wong, ``An incentive mechanism for cross-silo federated
  learning: A public goods perspective,'' in \emph{IEEE INFOCOM}.\hskip 1em
  plus 0.5em minus 0.4em\relax IEEE, 2021, pp. 1--10.

\bibitem{month}
P.~J. Miller and D.~M. Chin, ``Using monthly data to improve quarterly model
  forecasts,'' \emph{Federal Reserve Bank of Minneapolis Quarterly Review},
  vol.~20, pp. 16--28, 1996.

\bibitem{weekly}
N.~Van~Leersum, H.~Snijders, D.~Henneman, N.~Kolfschoten, G.~Gooiker,
  M.~Ten~Berge, E.~Eddes, M.~Wouters, R.~Tollenaar, W.~Bemelman \emph{et~al.},
  ``The dutch surgical colorectal audit,'' \emph{European Journal of Surgical
  Oncology}, vol.~39, no.~10, pp. 1063--1070, 2013.

\bibitem{first}
B.~McMahan, E.~Moore, D.~Ramage, S.~Hampson, and B.~A. y~Arcas,
  ``Communication-efficient learning of deep networks from decentralized
  data,'' in \emph{Artificial Intelligence and Statistics}.\hskip 1em plus
  0.5em minus 0.4em\relax PMLR, 2017, pp. 1273--1282.

\bibitem{newadd2}
Z.~Zhong, Y.~Zhou, D.~Wu, X.~Chen, M.~Chen, C.~Li, and Q.~Z. Sheng, ``P-fedavg:
  parallelizing federated learning with theoretical guarantees,'' in \emph{IEEE
  INFOCOM}.\hskip 1em plus 0.5em minus 0.4em\relax IEEE, 2021, pp. 1--10.

\bibitem{newadd4}
S.~Wang, M.~Lee, S.~Hosseinalipour, R.~Morabito, M.~Chiang, and C.~G. Brinton,
  ``Device sampling for heterogeneous federated learning: Theory, algorithms,
  and implementation,'' in \emph{IEEE INFOCOM}.\hskip 1em plus 0.5em minus
  0.4em\relax IEEE, 2021, pp. 1--10.

\bibitem{opti4}
X.~Mo and J.~Xu, ``Energy-efficient federated edge learning with joint
  communication and computation design,'' \emph{arXiv preprint
  arXiv:2003.00199}, 2020.

\bibitem{opti5}
B.~Luo, X.~Li, S.~Wang, J.~Huang, and L.~Tassiulas, ``Cost-effective federated
  learning design,'' in \emph{IEEE INFOCOM}, 2021, pp. 1--10.

\bibitem{incen1}
N.~Ding, Z.~Fang, and J.~Huang, ``Optimal contract design for efficient
  federated learning with multi-dimensional private information,'' \emph{IEEE
  Journal on Selected Areas in Communications}, vol.~39, no.~1, pp. 186--200,
  2020.

\bibitem{newadd1}
P.~Sun, H.~Che, Z.~Wang, Y.~Wang, T.~Wang, L.~Wu, and H.~Shao, ``Pain-fl:
  Personalized privacy-preserving incentive for federated learning,''
  \emph{IEEE Journal on Selected Areas in Communications}, vol.~39, no.~12, pp.
  3805--3820, 2021.

\bibitem{res-c10}
Y.~Zhan and J.~Zhang, ``An incentive mechanism design for efficient edge
  learning by deep reinforcement learning approach,'' in \emph{IEEE
  INFOCOM}.\hskip 1em plus 0.5em minus 0.4em\relax IEEE, 2020, pp. 2489--2498.

\bibitem{res-c11}
R.~Zeng, S.~Zhang, J.~Wang, and X.~Chu, ``Fmore: An incentive scheme of
  multi-dimensional auction for federated learning in {MEC},'' in \emph{2020
  IEEE 40th International Conference on Distributed Computing Systems
  (ICDCS)}.\hskip 1em plus 0.5em minus 0.4em\relax IEEE, 2020, pp. 278--288.

\bibitem{newadd3}
Y.~Deng, F.~Lyu, J.~Ren, Y.-C. Chen, P.~Yang, Y.~Zhou, and Y.~Zhang, ``Fair:
  Quality-aware federated learning with precise user incentive and model
  aggregation,'' in \emph{IEEE INFOCOM}.\hskip 1em plus 0.5em minus 0.4em\relax
  IEEE, 2021, pp. 1--10.

\bibitem{res-c12}
Y.~Zhan, J.~Zhang, Z.~Hong, L.~Wu, P.~Li, and S.~Guo, ``A survey of incentive
  mechanism design for federated learning,'' \emph{IEEE Transactions on
  Emerging Topics in Computing}, 2021.

\bibitem{free1}
W.~J. Baumol, ``Welfare economics and the theory of the state,'' in \emph{The
  encyclopedia of public choice}.\hskip 1em plus 0.5em minus 0.4em\relax
  Springer, 2004, pp. 937--940.

\bibitem{free4}
M.~Feldman, C.~Papadimitriou, J.~Chuang, and I.~Stoica, ``Free-riding and
  whitewashing in peer-to-peer systems,'' \emph{IEEE Journal on Selected Areas
  in Communications}, vol.~24, no.~5, pp. 1010--1019, 2006.

\bibitem{free5}
J.~Lin, M.~Du, and J.~Liu, ``Free-riders in federated learning: Attacks and
  defenses,'' \emph{arXiv preprint arXiv:1911.12560}, 2019.

\bibitem{prox2}
J.~Huang, R.~Talbi, Z.~Zhao, S.~Boucchenak, L.~Y. Chen, and S.~Roos, ``An
  exploratory analysis on users' contributions in federated learning,''
  \emph{arXiv preprint arXiv:2011.06830}, 2020.

\bibitem{prox3}
Y.~Fraboni, R.~Vidal, and M.~Lorenzi, ``Free-rider attacks on model aggregation
  in federated learning,'' in \emph{International Conference on Artificial
  Intelligence and Statistics}.\hskip 1em plus 0.5em minus 0.4em\relax PMLR,
  2021, pp. 1846--1854.

\bibitem{csr1}
Y.~Chen, X.~Yang, X.~Qin, H.~Yu, B.~Chen, and Z.~Shen, ``Focus: Dealing with
  label quality disparity in federated learning,'' \emph{arXiv preprint
  arXiv:2001.11359}, 2020.

\bibitem{csr2}
M.~A. Heikkil{\"a}, A.~Koskela, K.~Shimizu, S.~Kaski, and A.~Honkela,
  ``Differentially private cross-silo federated learning,'' \emph{arXiv
  preprint arXiv:2007.05553}, 2020.

\bibitem{csr4}
U.~Majeed, L.~U. Khan, and C.~S. Hong, ``Cross-silo horizontal federated
  learning for flow-based time-related-features oriented traffic
  classification,'' in \emph{2020 21st Asia-Pacific Network Operations and
  Management Symposium (APNOMS)}.\hskip 1em plus 0.5em minus 0.4em\relax IEEE,
  2020, pp. 389--392.

\bibitem{silo1}
C.~Zhang, S.~Li, J.~Xia, W.~Wang, F.~Yan, and Y.~Liu, ``Batchcrypt: Efficient
  homomorphic encryption for cross-silo federated learning,'' in
  \emph{Proceedings of the 2020 USENIX Annual Technical Conference}, 2020, pp.
  493--506.

\bibitem{wiopt}
N.~Ding, Z.~Fang, and J.~Huang, ``Incentive mechanism design for federated
  learning with multi-dimensional private information,'' in \emph{2020 18th
  International Symposium on Modeling and Optimization in Mobile, Ad Hoc, and
  Wireless Networks}.\hskip 1em plus 0.5em minus 0.4em\relax IEEE, 2020, pp.
  1--8.

\bibitem{minibatch}
O.~Dekel, R.~Gilad-Bachrach, O.~Shamir, and L.~Xiao, ``Optimal distributed
  online prediction using mini-batches.'' \emph{Journal of Machine Learning
  Research}, vol.~13, no.~1, 2012.

\bibitem{bingo}
X.~Zhang, F.~Li, Z.~Zhang, Q.~Li, C.~Wang, and J.~Wu, ``Enabling execution
  assurance of federated learning at untrusted participants,'' in \emph{IEEE
  INFOCOM}.\hskip 1em plus 0.5em minus 0.4em\relax IEEE, 2020, pp. 1877--1886.

\bibitem{BD}
M.~Li, T.~Zhang, Y.~Chen, and A.~J. Smola, ``Efficient mini-batch training for
  stochastic optimization,'' in \emph{Proceedings of the 20th ACM SIGKDD
  international conference on Knowledge discovery and data mining}, 2014, pp.
  661--670.

\bibitem{com0}
N.~H. Tran, W.~Bao, A.~Zomaya, M.~N. Nguyen, and C.~S. Hong, ``Federated
  learning over wireless networks: Optimization model design and analysis,'' in
  \emph{IEEE INFOCOM}.\hskip 1em plus 0.5em minus 0.4em\relax IEEE, 2019, pp.
  1387--1395.

\bibitem{wired}
G.~H. Lee and S.-Y. Shin, ``Federated learning on clinical benchmark data:
  Performance assessment,'' \emph{Journal of medical Internet research},
  vol.~22, no.~10, p. e20891, 2020.

\bibitem{res-a1}
M.~Diamanti, G.~Fragkos, E.~E. Tsiropoulou, and S.~Papavassiliou, ``Unified
  user association and contract-theoretic resource orchestration in {NOMA}
  heterogeneous wireless networks,'' \emph{IEEE Open Journal of the
  Communications Society}, vol.~1, pp. 1485--1502, 2020.

\bibitem{res-a2}
H.~Sun, X.~Ma, and R.~Q. Hu, ``Adaptive federated learning with gradient
  compression in uplink {NOMA},'' \emph{IEEE Transactions on Vehicular
  Technology}, vol.~69, no.~12, pp. 16\,325--16\,329, 2020.

\bibitem{tru2}
G.~Gao, M.~Xiao, J.~Wu, L.~Huang, and C.~Hu, ``Truthful incentive mechanism for
  nondeterministic crowdsensing with vehicles,'' \emph{IEEE Transactions on
  Mobile Computing}, vol.~17, no.~12, pp. 2982--2997, 2018.

\bibitem{folk}
S.~Shakkottai and R.~Srikant, ``Economics of network pricing with multiple
  isps,'' \emph{IEEE/ACM Transactions On Networking}, vol.~14, no.~6, pp.
  1233--1245, 2006.

\bibitem{newadd5}
Q.~Ma, J.~Huang, T.~Ba{\c{s}}ar, J.~Liu, and X.~Chen, ``Reputation and pricing
  dynamics in online markets,'' \emph{IEEE/ACM Transactions on Networking},
  vol.~29, no.~4, pp. 1745--1759, 2021.

\bibitem{fair0}
M.~Mohri, G.~Sivek, and A.~T. Suresh, ``Agnostic federated learning,'' in
  \emph{International Conference on Machine Learning}.\hskip 1em plus 0.5em
  minus 0.4em\relax PMLR, 2019, pp. 4615--4625.

\bibitem{fri}
J.~W. Friedman, ``A non-cooperative equilibrium for supergames,'' \emph{The
  Review of Economic Studies}, vol.~38, no.~1, pp. 1--12, 1971.

\bibitem{fri2}
E.~Hendon, H.~J. Jacobsen, and B.~Sloth, ``The one-shot-deviation principle for
  sequential rationality,'' \emph{Games and Economic Behavior}, vol.~12, no.~2,
  pp. 274--282, 1996.

\bibitem{book}
G.~J. Mailath and L.~Samuelson, \emph{Repeated games and reputations: long-run
  relationships}.\hskip 1em plus 0.5em minus 0.4em\relax Oxford university
  press, 2006.

\bibitem{res-c1}
X.~Li, K.~Huang, W.~Yang, S.~Wang, and Z.~Zhang, ``On the convergence of fedavg
  on non-iid data,'' in \emph{International Conference on Learning
  Representations}, 2019.

\bibitem{res-c2}
S.~Francis, I.~Tenison, and I.~Rish, ``Towards causal federated learning for
  enhanced robustness and privacy,'' \emph{arXiv preprint arXiv:2104.06557},
  2021.

\bibitem{res-c3}
Y.~Liu, Y.~Kang, X.~Zhang, L.~Li, Y.~Cheng, T.~Chen, M.~Hong, and Q.~Yang, ``A
  communication efficient collaborative learning framework for distributed
  features,'' \emph{arXiv preprint arXiv:1912.11187}, 2019.

\bibitem{res-c4}
Amazon, ``Amazon lambda pricing,''
  \url{https://www.amazonaws.cn/en/lambda/pricing/}, Accessed Jun. 22, 2021.

\bibitem{DAGMM}
B.~Zong, Q.~Song, M.~R. Min, W.~Cheng, C.~Lumezanu, D.~Cho, and H.~Chen, ``Deep
  autoencoding {Gaussian} mixture model for unsupervised anomaly detection,''
  in \emph{International conference on learning representations}, 2018.

\bibitem{res-b1}
J.~B. Rosen, ``Existence and uniqueness of equilibrium points for concave
  n-person games,'' \emph{Econometrica: Journal of the Econometric Society},
  pp. 520--534, 1965.

\end{thebibliography}

\begin{IEEEbiography}[{\includegraphics[width=1in,height=1.25in]{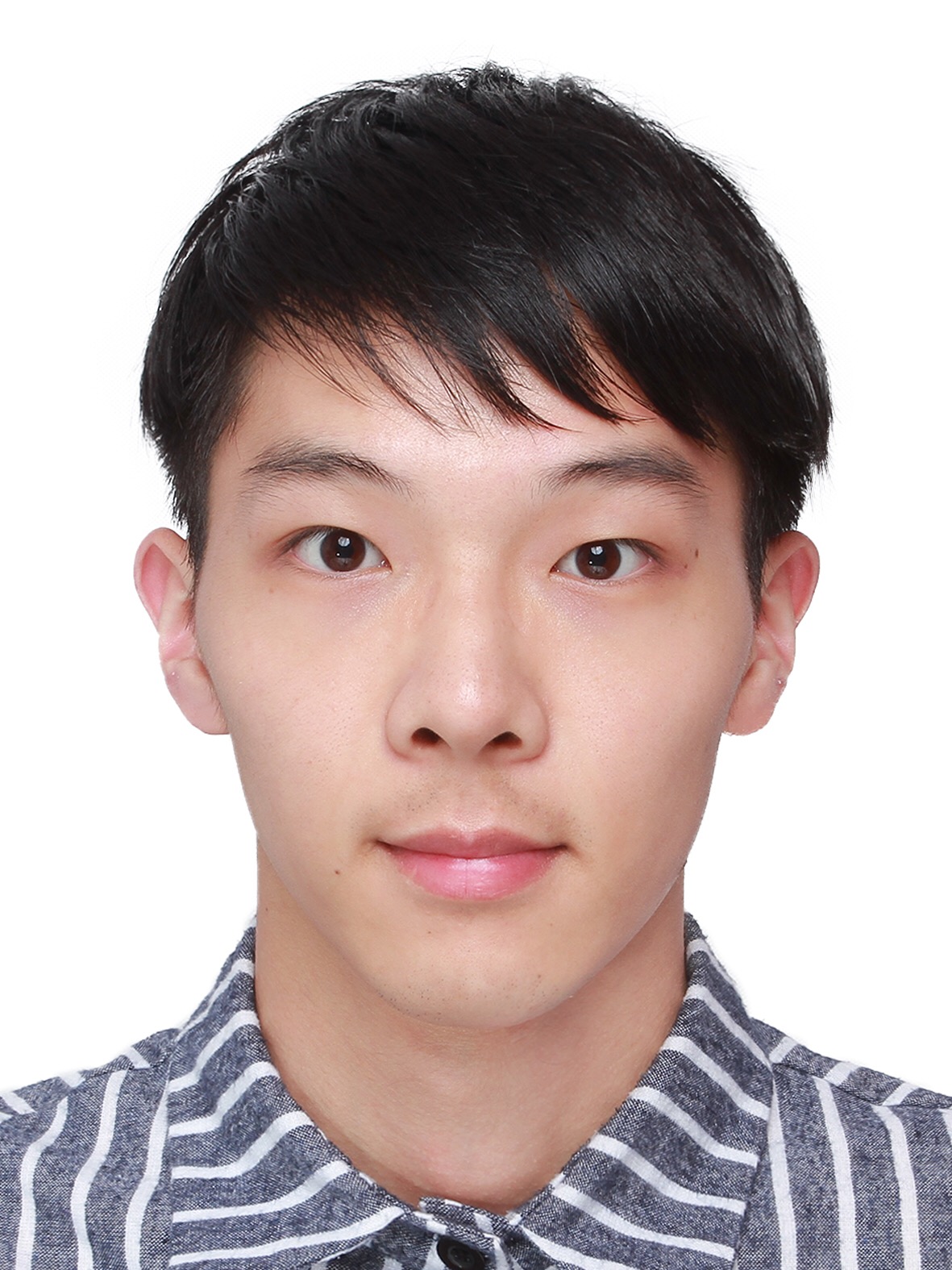}}]{Ning Zhang}
 received the B.E. degree from the School of Intelligent Systems Engineering, Sun Yat-sen University (SYSU), Guangzhou, China in 2020. He is currently pursuing his M.S. degree in the School of Intelligent Systems Engineering, SYSU. His primary research interests include game theory, incentive mechanism and federated learning.
\end{IEEEbiography}

\begin{IEEEbiography}[{\includegraphics[width=1in,height=1.25in]{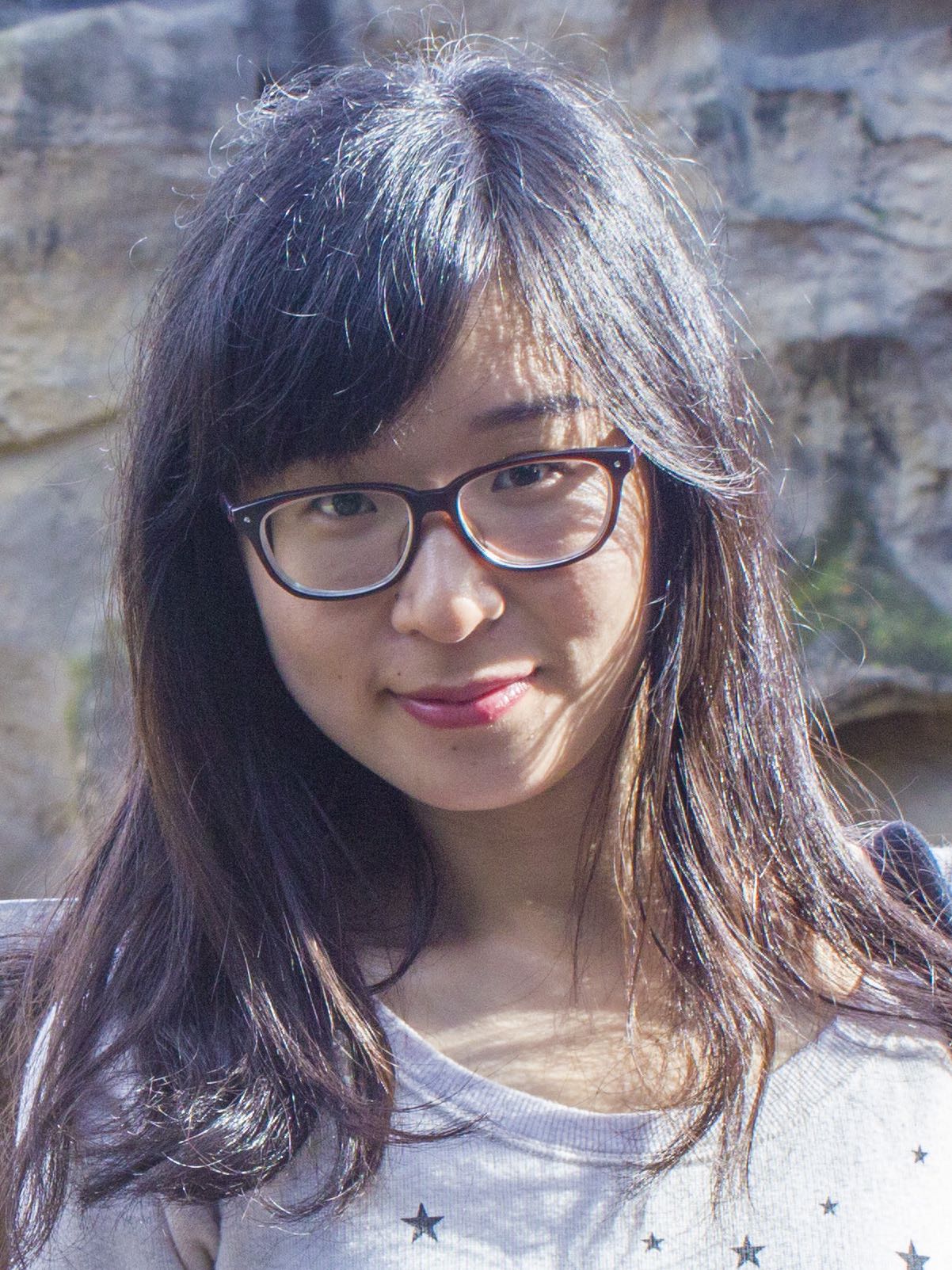}}]{Qian Ma}
 is an Associate Professor of School of Intelligent Systems Engineering, Sun Yat-sen University. She worked as a Postdoc Research Associate at Northeastern University during 2018-2019. She received the Ph.D. degree in the Department of Information Engineering from the Chinese University of Hong Kong in 2017, and the B.S. degree from Beijing University of Posts and Telecommunications (China) in 2012. Her research interests lie in the field of network optimziation and economics. She is the recipient of the Best Paper Award from the IEEE International Symposium on Modeling and Optimization in Mobile, Ad Hoc and Wireless Networks (WiOpt) in 2021 and the recipient of the Best Student Paper Award from the IEEE International Symposium on Modeling and Optimization in Mobile, Ad Hoc and Wireless Networks (WiOpt) in 2015.
\end{IEEEbiography}

\begin{IEEEbiography}[{\includegraphics[width=1in,height=1.25in]{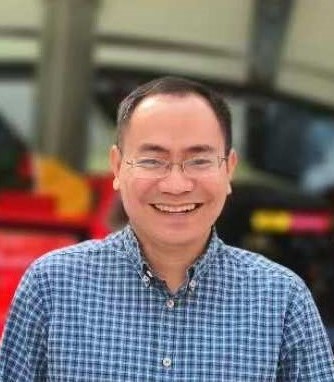}}]{Xu Chen}
 is a Full Professor with Sun Yat-sen University, Guangzhou, China, and the Vice Director of National and Local Joint Engineering Laboratory. He received the Ph.D. degree in information engineering from the Chinese University of Hong Kong in 2012, and worked as a Postdoctoral Research Associate at Arizona State University, Tempe, USA, from 2012 to 2014, and a Humboldt Scholar Fellow at the Institute of Computer Science of the University of Goettingen, Germany from 2014 to 2016. He received the prestigious Humboldt research fellowship awarded by the Alexander von Humboldt Foundation of Germany, 2014 Hong Kong Young Scientist Runner-up Award,  2017 IEEE Communication Society Asia-Pacific Outstanding Young Researcher Award, 2017 IEEE ComSoc Young Professional Best Paper Award, Honorable Mention Award of 2010 IEEE international conference on Intelligence and Security Informatics (ISI), Best Paper Runner-up Award of 2014 IEEE International Conference on Computer Communications (INFOCOM), and Best Paper Award of 2017 IEEE Intranational Conference on Communications (ICC). He is currently an Area Editor of the IEEE OPEN JOURNAL OF THE Communications Society, an Associate Editor of the IEEE TRANSACTIONS WIRELESS COMMUNICATIONS, IEEE TRANSACTIONS VEHICULAR TECHNOLOGY, and IEEE INTERNET OF THINGS JOURNAL.
\end{IEEEbiography}

\clearpage

\appendices

\section{Proof of Lemma 1}\label{lem11}

Take the first-order derivative of (5) and let it to be 0, that is, 
\begin{equation}\label{fo1}
E_n-\frac{\rho_n}{2\sqrt{G(x_n+\sum_{n'\in\mathcal{N},n'\neq n}x_{n'})^3}}=0.
\end{equation}

So we can get,
\begin{equation}
x_n=\sqrt[3]{\frac{\rho_n^2}{4GE_n^2}}-\sum_{n'\in\mathcal{N},n'\neq n}x_{n'}.
\end{equation}

Since the second-order derivative of (5) satisfies 

\begin{equation}
\frac{3\rho_n}{4\sqrt{G(x_n+\sum_{n'\in\mathcal{N},n'\neq n}x_{n'})^5}}>0,
\end{equation}

the total cost of each client $n$, calculated in (5), decreases monotonically in the interval $(-\infty, \sqrt[3]{\frac{\rho_n^2}{4GE_n^2}}-\sum_{n'\in\mathcal{N},n'\neq n}x_{n'})$ and increases monotonically in the interval $(\sqrt[3]{\frac{\rho_n^2}{4GE_n^2}}-\sum_{n'\in\mathcal{N},n'\neq n}x_{n'}, +\infty)$.

Since the strategy of client $n$ is in the range $[0, D_n]$, we know that the best response of each client $n$ to minimize its total cost is:
\begin{equation}
\begin{aligned}
&x_n^{BR}(\boldsymbol{x}_{-n}) = \min \\
&\left\{  D_n, \max \left\{ \sqrt[3]{\frac{\rho_n^2}{4GE_n^2}} - \sum_{n'\in\mathcal{N},n'\neq n} x_{n'}, 0 \right\}  \right\}.
\end{aligned}
\end{equation}

\section{Proof of Theorem 1}\label{tho11}

We will prove that under the Nash equilibrium, no client can get a lower cost by changing its strategy while keeping other clients' strategies unchanged. Here we denote $\sqrt[3]{\frac{\rho_n^2}{4GE_n^2}}-\sum_{n'\in\mathcal{N},n'\neq n}x_{n'})$ as $x_n^{opt}$ for convenience and $x_n^{opt}$ is the stationary point of the cost function of client $n$.

When the critical client $k$ exists, for client $k$, the stationary point of client $k$'s cost function is $x_k^{opt} = \sqrt[3]{\frac{\rho_k^2}{4GE_k^2}} - (N - k)D$, because $(N-k)D \leq \sqrt[3]{\frac{\rho_k^2}{4GE_k^2}} \leq (N-k+1)D$, we can get $0 \leq x_k^{opt} \leq D$, the best response of client $k$ is $x_k^{BR}(\boldsymbol{x}_{-k}) = \sqrt[3]{\frac{\rho_k^2}{4GE_k^2}} - (N - k)D$. 

Similarly for client $k - 1$, $x_{k-1}^{opt} = \sqrt[3]{\frac{\rho_{k-1}^2}{4GE_{k-1}^2}} - (N - k)D - x_k^{BR}(\boldsymbol{x}_{-k}) = \sqrt[3]{\frac{\rho_{k-1}^2}{4GE_{k-1}^2}} - \sqrt[3]{\frac{\rho_k^2}{4GE_k^2}} < 0$, the best response of client $k-1$ is $x_{k-1}^{BR}(\boldsymbol{x}_{-(k-1)}) = 0$. For client $k - 2, k - 3, \cdots, 1$, similarly we can get $x_1^{BR}(\boldsymbol{x}_{-1}) = x_2^{BR}(\boldsymbol{x}_{-2}) = \cdots = x_{k-2}^{BR}(\boldsymbol{x}_{-(k-2)}) = 0$. 

For client $k + 1$, $x_{k+1}^{opt} = \sqrt[3]{\frac{\rho_{k+1}^2}{4GE_{k+1}^2}} - (N - k - 1)D - x_k^{BR}(\boldsymbol{x}_{-k}) = \sqrt[3]{\frac{\rho_{k+1}^2}{4GE_{k+1}^2}} - \sqrt[3]{\frac{\rho_k^2}{4GE_k^2}} + D > D$, the best response of client $k + 1$ is $x_{k+1}^{BR}(\boldsymbol{x}_{-(k+1)}) = D$. For client $k + 2, k + 3, \cdots, N$, similarly we can get $x_{k+2}^{BR}(\boldsymbol{x}_{-(k+2)}) = x_{k+3}^{BR}(\boldsymbol{x}_{-(k+3)}) = \cdots = x_N^{BR}(\boldsymbol{x}_{-N}) = D$. 

When the critical client $k$ does not exist. There must exist a client $m$ as mentioned in Case 2 in Theorem 1. 

For client $m$, $x_m^{opt} = \sqrt[3]{\frac{\rho_m^2}{4GE_m^2}} - (N - m)D$. Because $\sqrt[3]{\frac{\rho_m^2}{4GE_m^2}} < (N-m)D < \sqrt[3]{\frac{\rho_{m+1}^2}{4GE_{m+1}^2}}$, $x_m^{opt} < 0$, so the best response of client $m$ is $x_m^{BR}(\boldsymbol{x}_{-m}) = 0$. 

For client $m-1$, $x_{m-1}^{opt} = \sqrt[3]{\frac{\rho_{m-1}^2}{4GE_{m-1}^2}} - (N - m)D < \sqrt[3]{\frac{\rho_m^2}{4GE_m^2}} - (N - m)D < 0$, the best response of client $m-1$ is $x_{m-1}^{BR}(\boldsymbol{x}_{-(m-1)}) = 0$. Similarly, $x_1^{BR}(\boldsymbol{x}_{-1}) = x_2^{BR}(\boldsymbol{x}_{-2}) = \cdots = x_m^{BR}(\boldsymbol{x}_{-m}) = 0$. 

For client $m+1$, $x_{m+1}^{opt} = \sqrt[3]{\frac{\rho_{m+1}^2}{4GE_{m+1}^2}} - (N - m -1)D > D$, the best response of client $m+1$ is $x_{m+1}^{BR}(\boldsymbol{x}_{-(m+1)}) = D$. 

For client $m+2$, $x_{m+2}^{opt} = \sqrt[3]{\frac{\rho_{m+2}^2}{4GE_{m+2}^2}} - (N - m - 2)D - x_{m+1}^{BR}(\boldsymbol{x}_{-(m+1)}) = \sqrt[3]{\frac{\rho_{m+2}^2}{4GE_{m+2}^2}} - \sqrt[3]{\frac{\rho_{m+1}^2}{4GE_{m+1}^2}} + D > D$, the best response of client $m+2$ is $x_{m+2}^{BR}(\boldsymbol{x}_{-(m+2)}) = D$. Similarly, $x_{m+1}^{BR}(\boldsymbol{x}_{-(m+1)}) = x_{m+2}^{BR}(\boldsymbol{x}_{-(m+2)}) = \cdots = x_N^{BR}(\boldsymbol{x}_{-N}) = D$.

We can see that under this strategy, each client's strategy is the best response to other clients' strategies, so this strategy is the Nash equilibrium of Game 1.

\begin{table}[t]
\newcommand{\tabincell}[2]{\begin{tabular}{@{}#1@{}}#2\end{tabular}}
\centering
\caption{Key Notations}
\begin{tabular}{|c||p{6.5cm}|}
\hline
\textbf{Symbol}  &  \textbf{Physical Meaning} \\
\hline
$\mathcal{N}$  & The set of clients participating in cross-silo FL  \\
\hline
$\mathcal{D}_n$ & The set of local data of client $n \in \mathcal{N}$   \\
\hline
$D_n$ & The number of local data samples in set $\mathcal{D}_n$  \\
\hline
$\mathcal{X}_n$ & The subset of local data that client $n$ chooses to \\
& perform model training  \\ 
\hline
$x_n$ & The size of the chosen subset $\mathcal{X}_n$  \\
\hline
$\boldsymbol{x}$ &  The vector of clients' chosen strategies \\
\hline
$B$ & The total amount of chosen local data, i.e., \\
& $B=\sum_{n\in\mathcal{N}}x_n$ \\ 
\hline
$\boldsymbol{d}_{ni}$ & The $i$-th data sample in set $\mathcal{X}_n$  \\
\hline
$\boldsymbol{w}$ & The parameter vector of the global model  \\
\hline
$L(\boldsymbol{w})$ & The global loss function under $\boldsymbol{w}$ \\
\hline
$L_n(\boldsymbol{w})$ & The local loss function of client $n$ under $\boldsymbol{w}$  \\
\hline
$l(\boldsymbol{w};\boldsymbol{d}_{ni})$ & The loss function for data sample $\boldsymbol{d}_{ni}$ under $\boldsymbol{w}$ \\
\hline
$G$ & The number of iterations in a cross-silo FL process \\
\hline
$A(\boldsymbol{x})$ & The model accuracy loss under $\boldsymbol{x}$ \\
\hline
$\mathcal{E}(x_n)$ & The computation cost of client $n$ \\
\hline
$C_n$ & The communication cost of client $n$\\
\hline
$p$ & The payment that each client pays the central server \\
\hline
$\rho_n$ & Client $n$'s valuation for the model accuracy \\
\hline
$F_n(x_n, \boldsymbol{x}_{-n})$ & The total cost of each client $n\in \mathcal{N}$ \\
\hline
\end{tabular}
\label{table:Notation}
\end{table}

\section{Proof of Theorem 2}\label{tho22}

We will prove Theorem 2 by mathematical induction. For convenience, we denote $\sqrt[3]{\frac{\rho_n^2}{4GE_n^2}}$ as $h_n$. We will prove that no matter what initial strategies clients choose, they strategies will eventually converge to the Nash equilibrium.

(1) There are two clients $y_1, y_2$, $y_1 < y_2$, whose initial strategy is $a$ and $b$ respectively, $a, b$ can be any value in $[0, D]$. For client $y_1$, he will calculate the stationary point of his cost function $x_{y_1}^{opt} = h_{y_1} - b$, now we suppose $h_{y_1} - b \in [0, D]$, so the best response of client $y_1$ is $x_{y_1}^{BR}(\boldsymbol{x}_{-{y_1}}) = h_{y_1} - b$. Then client $y_2$ will calculate the stationary point of his cost function $x_{y_2}^{opt} = h_{y_2} - h_{y_1} + b$, here we still assume $h_{y_2} - h_{y_1} + b \in [0, D]$, so the best response of client $y_2$ is $x_{y_2}^{BR}(\boldsymbol{x}_{-{y_2}}) = h_{y_2} -h_{y_1} + b$. Then client $y_1$ will update his best response $x_{y_1}^{BR}(\boldsymbol{x}_{-{y_1}}) = h_{y_1} - h_{y_2} + h_{y_1} - b$. Client $y_2$ will update his best response $x_{y_2}^{BR}(\boldsymbol{x}_{-{y_2}}) = h_{y_2} - h_{y_1} + h_{y_2} - h_{y_1} + b$. As this process goes on, we can see that $x_{y_1}^{BR}(\boldsymbol{x}_{-{y_1}})$ continues to decrease from $h_{y_1} - b$ and $x_{y_2}^{BR}(\boldsymbol{x}_{-{y_2}})$ continues to increase from $b$ until one client reaches the boundary, i.e., $0$ or $D$. Let's assume that client $y_1$ reaches the boundary $0$ first, i.e., $x_{y_1}^{BR}(\boldsymbol{x}_{-{y_1}}) = 0, x_{y_2}^{BR}(\boldsymbol{x}_{-{y_2}}) < D$. So $x_{y_2}^{opt} = h_{y_2} - 0 = h_{y_2} = x_{y_2}^{BR}(\boldsymbol{x}_{-{y_2}}), x_{y_1}^{opt} = h_{y_1} - h_{y_2} < 0, x_{y_1}^{BR}(\boldsymbol{x}_{-{y_1}}) = 0$. It eventually converges to $x_{y_1}^* = 0, x_{y_2}^* = h_{y_2}$. Similarly, if $y_2$ reaches the boundary $D$ first, i.e., $x_{y_1}^{BR}(\boldsymbol{x}_{-{y_1}}) > 0, x_{y_2}^{BR}(\boldsymbol{x}_{-{y_2}}) = D$). So $x_{y_1}^{opt} = h_{y_1} - D > 0$, $x_{y_1}^{BR}(\boldsymbol{x}_{-{y_1}}) = h_{y_1} - D$, $x_{y_2}^{opt} = h_{y_2} - h_{y_1} + D > D$, $x_{y_2}^{BR}(\boldsymbol{x}_{-{y_2}}) = D$. It eventually converges to $x_{y_1}^* = h_{y_1} - D, x_{y_2}^* = D$. It illustrates that no matter which client choose which value, it will always converge to the Nash equilibrium under the condition $\rho_1 < \rho_2$. 

(2) Suppose when there are $n$ clients who satisfies $\frac{\rho_1}{E_1} < \frac{\rho_2}{E_2} < \cdots < \frac{\rho_N}{E_N}$, there initial strategies are $a_1, a_2, \cdots, a_N \in [0, D]$ respectively. Their strategy will eventually converge to the Nash equilibrium:

$x_1^* = x_2^* = \cdots = x_{k-1}^* = 0, x_k^* = h_k - (N - k)D, x_{k+1}^* = x_{k+2}^* = \cdots = x_N^* = D$
(Similarly if $k$ don't exist).

(3) When there are $n + 1$ clients who satisfies $\frac{\rho_1}{E_1} < \frac{\rho_2}{E_2} < \cdots < \frac{\rho_N}{E_N} < \frac{\rho_{N+1}}{E_{N+1}}$, their initial strategies are $a_1, a_2, \cdots, a_N, a_{N+1} \in [0, D]$. According to (2), the last $N$ clients will eventually converge to the Nash equilibrium beacuse they satisfy the condition $\frac{\rho_2}{E_2} < \cdots < \frac{\rho_N}{E_N} < \frac{\rho_{N+1}}{E_{N+1}}$. When client $y_{1}$ adds to the list, let's see what will happens from the converge point of the last $N$ clients.

We have discussed that there may have two situations. We first consider that there exists a critical client $k + 1$ in the interval. The last $N$ clients will eventually converge to the form 

\begin{equation}
x_n^\ast=
\left\{
\begin{aligned}
& 0, &\mbox{ if } &y_2 \leq n < y_{k+1};\\
& h_{k+1} - (N-k)D, &\mbox{ if } &n = y_{k+1};\\
& D, &\mbox{ if } &y_{k+1} < n \leq N.
\end{aligned}
\right.
\end{equation}

When clients $y_1$ adds to the list, $x_{y_1}^{opt} = h_{y_1} - (N-k)D - x_{y_{k+1}}^*$. According to (2), we know that $x_{y_2}^{opt} = h_{y_2} - (N-k)D - x_{y_{k+1}}^* < 0$. We know that $h_{y_2} > h_{y_1}$, so $x_{y_1}^{opt} < 0, x_{y_1}^{BR}(\boldsymbol{x}_{-{y_1}}) = 0$. Then calculate the best response of the other $n$ clients, $x_{y_2}^{opt} = h_{y_2} - (N-k)D - x_{y_{k+1}}^* < 0$, $x_{y_2}^{BR}(\boldsymbol{x}_{-{y_2}}) = 0, \cdots$, $x_{y_{k+1}}^{opt} = h_{y_{k+1}} - (N-k)D \in [0, D]$, $x_{y_{k+1}}^{BR}(\boldsymbol{x}_{-{y_{k+1}}}) = h_{y_{k+1}} - (N-k)D$, $x_{y_{k+2}}^{opt} = h_{y_{k+2}} - (N-k)D - x_{y_{k+1}}^* + D > D$, $x_{y_{k+2}}^{BR}(\boldsymbol{x}_{-{y_{k+2}}}) = D, \cdots$, $x_{y_{N+1}}^{opt} = h_{y_{N+1}} - (N-k)D - x_{y_{k+1}}^* + D > D$, $x_{y_{N+1}}^{BR}(\boldsymbol{x}_{-{y_{N+1}}}) = D$, we eventually get 

\begin{equation}
x_n^\ast=
\left\{
\begin{aligned}
& 0, &\mbox{ if } &y_1 \leq n < y_{k+1};\\
& h_{k+1} - (N-k)D, &\mbox{ if } &n = y_{k+1};\\
& D, &\mbox{ if } &y_{k+1} < n \leq N.
\end{aligned}
\right.
\end{equation}

Similarly, when we consider that there does not exist the critical client $k + 1$ in the interval. The game will eventually converge to 

\begin{equation}
x_n^\ast=
\left\{
\begin{aligned}
& 0, &\mbox{ if } &y_1 \leq n \leq y_{m+1};\\
& D, &\mbox{ if } &y_{m+1} < n \leq N.
\end{aligned}
\right.
\end{equation}

According to (3), when there are $n + 1$ clients who satisfy $\frac{\rho_1}{N_1} < \frac{\rho_2}{E_2} < \cdots < \frac{\rho_N}{E_N} < \frac{\rho_{N+1}}{E_{N+1}}$, no matter no matter what initial strategies clients choose, their strategies will eventually converge to the unique Nash equilibrium. So far, we have proved the uniqueness of Nash equilibrium.

\section{Proof of concave game and convergence}\label{res-b2prove}

We will discuss the convergence of the best response update process. we first prove that the stage game SPFL is a concave game.

\begin{definition}[Concave Game]\label{res-pp}
A game with a finite number of players and continuous action sets is a concave game if (i) every player's action set is compact and convex, and (ii) every player's payoff (or utility) function is concave in its own strategy \cite{res-b1}.
\end{definition}

We next show that the stage game SPFL is a concave game.

\begin{lemma}\label{res-p}
Game 1 (i.e., the stage game SPFL) is a concave game. 
\end{lemma}

\begin{proof}
To prove that Game 1 is a concave game, we prove that Game 1 satisfies the conditions in Definition \ref{res-pp}. Specifically, the strategy space of client $n$ is $[0, D_n], \forall n \in \mathcal{N}$, which is compact and convex. The cost function that client $n$ aims to minimize is $F_n(x_n, \boldsymbol{x}_{-n}) = \rho_nA(x_n, \boldsymbol{x}_{-n})+\mathcal{E}_n(x_n)+C_n+p$. Taking the second-order derivative of the cost function $F_n(x_n, \boldsymbol{x}_{-n})$, we get $F_n^{''}(x_n, \boldsymbol{x}_{-n})= \frac{3\rho_n}{4\sqrt{G(x_n+ \boldsymbol{x}_{-n})^5}} \geq 0$. Thus each client's cost function is convex in its own strategy. This completes our proof. 
\end{proof}

Concave games have a nice property that when the NE is unique, the best response update process is globally asymptotically stable (see Theorem 8 and Theorem 9 in \cite{res-b1}). Furthermore, starting from any feasible strategy profile, the best response update process always converges to the unique NE of the concave game \cite{res-b1}. We have proven that the stage game SPFL admits a unique NE if $\frac{\rho_1}{E_1} < \frac{\rho_2}{E_2} < \cdots < \frac{\rho_N}{E_N}$, and we show the best response update process for clients in Algorithm \ref{res-alo} in the appendix. Specifically, all clients begin with a random participation strategy (Line 1). Then each client updates its participation strategy (i.e., best response update) to minimize its total cost given other clients' strategies (Line 4) until convergence. 

\begin{proposition}
The best response update process in Algorithm \ref{res-alo} in the appendix converges to the Nash equilibrium of Game 1. 
\end{proposition}

\begin{algorithm}[t]
\LinesNumbered
\SetAlgoLined
\begin{small}
\textbf{Initialization} Set iteration round $r=0$. Each client $n \in \mathcal{N}$ starts with a random participation strategy $x_n(r=0) \in [0, D]$\;
\Repeat{$\boldsymbol{x}$ converges}{
\ForEach{$n \in \mathcal{N}$}{
\textbf{Best response update:} client $n$ updates its participation strategy according to: $x_n(r) = \mathop{\arg\min}\limits_{x_n \in [0, D]} F_n(x_n, \boldsymbol{x}_{-n}(r)) = \min \left\{  D_n, \max \left\{ \sqrt[3]{\frac{\rho_n^2}{4GE_n^2}} - \sum_{n'\in\mathcal{N},n'\neq n} x_{n'}(r), 0 \right\}  \right\}$.
}
Update the strategy profile and the iteration round: $\boldsymbol{x}(r+1) = \boldsymbol{x}(r)$ and $r=r+1$.
}

\end{small}
\caption{Best Response Update Algorithm}\label{res-alo}
\label{code:recentEnd}
\end{algorithm}

\section{Proof of Theorem 3.}\label{kkk1}

First, we prove a lemma that guarantees $B_n \leq O_n$ and $B_n^{'} \leq O_n^{'}$.

\begin{lemma}\label{B}
For $k \in \mathcal{N}$, we have $B_n \leq O_n$, and for $m \in \mathcal{N}$, we have $B_n^{'} \leq O_n^{'}$.
\end{lemma}

\begin{proof}
Set the function $f(k)$:
\begin{equation}
f(k) = 3Dk + 2\sqrt{\frac{((N-k)D+x_k^*)^3}{(N-n)D + x_k^*}}.
\end{equation}

Take the first derivative of $f(k)$:
\begin{equation}
f{'}(k) = 3D -3D \sqrt{\frac{(N-k)D +x_k^*}{(N-n)D + x_k^*}} \geq 0 \ k \in \mathcal{N}.
\end{equation}

The minimum value of $f(k)$ in the domain $[n, N]$ is $f(n)$:
\begin{equation}
f(n) = (2N+n)D +2x_k^*.
\end{equation}

Then we know that $f(k) \geq f(n)$, so
\begin{equation}\label{yixiang}
3Dk + 2\sqrt{\frac{((N-k)D+x_k^*)^3}{(N-n)D + x_k^*}} \geq (2N+n)D +2x_k^*.
\end{equation}

Deal with the inequality \eqref{yixiang}, we can get
\begin{equation}
\frac{D\sqrt{G}}{\frac{1}{\sqrt{(N-k)D+x_k^*}}-\frac{1}{\sqrt{(N-n)D+x_k^*}}}  \geq \frac{2\sqrt{G((N-k)D+x_k^*)^3}}{k-n}.
\end{equation}
\end{proof}
When the critical client $k$ does not exist, the process is similarly to the above, and we will not repeat it.

Then we will prove that under the cooperative strategy, each client's cost is less than that at the NE of Game 1.

We denote $F_l(\boldsymbol{x}^{coop})-F_l(\boldsymbol{x}^*)$ as $H_l(x^{coop})$, which is the difference between the cost under cooperative strategy and the cost at the NE of Game 1. If $H_l(x^{coop}) <0$, the cost of each client under cooperative strategy is less than that at the NE of Game 1. Specifically, 
$H_l(x^{coop}) = \frac{\rho_l}{\sqrt{G}}(\frac{1}{\sqrt{(k-l)x^{coop} + x_k^* + (N-k)D}} - \frac{1}{\sqrt{(N-k)D+x_k^*}})+E_lx^{coop}$.

Take the first derivative of $H_l(x^{coop})$:

\noindent$H_l^{'} (x^{coop}) = E_l - \frac{\rho_l(k-l)}{2\sqrt{G((k-l)x^{coop}+x_k^*+(N-k)D)^3}}$.

When $\frac{\rho_l}{E_l} > B_l$, $H_l^{'}(0) < 0, H_l(0) = 0, H_l^{'}(+\infty) >0$. It's easy to know that $H_l^{'}(x^{coop})$ is a monotonically increasing function of $x^{coop}$, so the curve of $H_l(x^{coop})$ is a curve that monotonically decreases and then monotonically increases, like a parabola with an upper opening. When $\frac{\rho_l}{E_l} > O_l, H_l(D) > 0$, the unique intersection of $H_l(x^{coop})$ and the $x^{coop}$ axis at $(0, + \infty)$ is larger than $D$, so the maximum amount of local data that client $l$ can choose under the cooperative strategy is $x_l^{th}=D$. The $k-l$ converted contributors can choose any same strategy on $(0, D]$ as their cooperative strategy. This strategy can reduce their cost. When $B_l < \frac{\rho_l}{E_l} \leq O_l$, $H_l(D) \geq 0$, the unique intersection of $H(x^{coop})$ and the $x^{coop}$ axis at $[0, + \infty)$ is the solution of the implicit function (13), we set it as $x_l^{th}(\frac{\rho_l}{E_l})$. So $x_l^{th}=x_l^{th}(\frac{\rho_l}{E_l})$, the $k-l$ converted contributors can choose any strategy on $(0, x_l^{th}(\frac{\rho_l}{E_l}))$ as their cooperative strategy. This strategy can reduce their cost.

When $\frac{\rho_l}{E_l} \leq B_l$, it means $H_l^{'}(0) \geq 0, H_l(0) = 0$, $H_l(x^{coop})$ is a monotonically increasing function of $x^{coop}$, so there is no intersection at $(0, + \infty)$, all free riders should choose 0 as their strategy. If $\frac{\rho_1}{E_1} \leq B_1, \frac{\rho_2}{E_2} \leq B_2, \cdots, \frac{\rho_{k-1}}{E_{k-1}} \leq B_{k-1}$, it means that there will not exist a cooperative strategy that can reduce the cost of each client at the same time, so all client will not cooperate and play the NE of the stage game. The reason of this situation is clients has low valuation for the accuracy of the global model. With the increase of training data, the local model is good enough and the global model has small improvement. The reduction of accuracy loss is very small (i.e., the difference of accuracy loss between the NE of stage game and cooperative strategy of repeated game), while the impact of computation cost is very significant. Any increase in data will lead to the increase of total cost.

Now we will illustrate why each client can reduce their cost under the cooperative strategy.

According to Theorem 3, clients $1, 2, \cdots, l-1$ will choose 0 as their strategy. Because $(k-l)x^{coop} + x_k^{*} + (N-k)D > (N-k)D+x_k^*$, we can get
\begin{equation}
\begin{split}
&H_1(x^{coop}) = \frac{\rho_1}{\sqrt{G}}(\frac{1}{\sqrt{(k-l)x^{coop} + x_k^* + (N-k)D}}\\
& - \frac{1}{\sqrt{(N-k)D+x_k^*}}) < 0.
\end{split}
\end{equation}
Similarly, $H_2(x^{coop})<0, \cdots, H_{l-1}^{coop} < 0$.

Because $\frac{\rho_l}{E_l} > B_l$, we can get 
\begin{equation}
\begin{split}
&H_l(x^{coop}) = \frac{\rho_l}{\sqrt{G}}(\frac{1}{\sqrt{(k-l)x^{coop} + x_k^* + (N-k)D}} \\
&- \frac{1}{\sqrt{(N-k)D+x_k^*}})+E_lx^{coop} \leq 0.
\end{split}
\end{equation}

By transplanting the inequality, we can get

\begin{equation}
\frac{\rho_l}{E_l} \geq \frac{\sqrt{G}x^{coop}}{\frac{1}{\sqrt{(N-k)D+x_k^*}}-\frac{1}{\sqrt{(k-l)x^{coop} + x_k^* + (N-k)D}}}.
\end{equation}

Because $\frac{\rho_l}{E_l} < \frac{\rho_{l+1}}{E_{l+1}} < \cdots < \frac{\rho_{k-1}}{E_{k-1}}$, we can get

\begin{equation}
\frac{\rho_{l+1}}{E_{l+1}} > \frac{\rho_l}{E_l} \geq \frac{\sqrt{G}x^{coop}}{\frac{1}{\sqrt{(N-k)D+x_k^*}}-\frac{1}{\sqrt{(k-l)x^{coop} + x_k^* + (N-k)D}}}.
\end{equation}

By transplanting the inequality, we can get

\begin{equation}
\begin{split}
&\frac{\rho_{l+1}}{\sqrt{G}}(\frac{1}{\sqrt{(k-l)x^{coop} + x_k^* + (N-k)D}} \\
&- \frac{1}{\sqrt{(N-k)D+x_k^*}})+E_{l+1}x^{coop} = H_{l+1}(x^{coop}) < 0.
\end{split}
\end{equation}

Note that the difference between $H_l(x^{coop})$ and $H_{l+1}(x^{coop})$ lies in $\rho_l$ and $E_l$. $k-l$ remains unchanged because it represents the number of converted contributors, which is unrelated with the subscript of function $H_l(x^{coop})$. 

Similarly, we can get $H_{l+2}(x^{coop}) < 0, \cdots, H_{k-1}(x^{coop}) < 0$.

For client $k$, his cooperative strategy is $x_k^*$, we can get 
\begin{equation}
\begin{split}
H_k(x^{coop}) = &\frac{\rho_k}{\sqrt{G}}(\frac{1}{\sqrt{(k-l)x^{coop} + x_k^* + (N-k)D}} - \\
& \frac{1}{\sqrt{(N-k)D+x_k^*}}) <0.
\end{split}
\end{equation}

For client $k+1$, his cooperative strategy is $D$, we can get
\begin{equation}
\begin{split}
&H_{k+1}(x^{coop}) = \frac{\rho_{k+1}}{\sqrt{G}}(\frac{1}{\sqrt{(k-l)x^{coop} + x_k^* + (N-k)D}} \\
&- \frac{1}{\sqrt{(N-k)D+x_k^*}})< 0.
\end{split}
\end{equation}
Similarly, $H_{k+2}(x^{coop})<0, \cdots, H_N(x^{coop})< 0$.

\section{Proof of Theorem 4.}\label{mmm1}

We will prove that under the cooperative strategy, each client can reduce their cost under the cooperative strategy.

Similarly to the proof of Theorem 3, $F_l(\boldsymbol{x}^{coop})-F_l(\boldsymbol{x}^*) = H_l(x^{coop})$ which is the difference between the cost under cooperative strategy and the cost at the NE of the Game 1. Specifically, $H_l(x^{coop}) = \frac{\rho_l}{\sqrt{G}}(\frac{1}{\sqrt{(m-l+1)x^{coop} +(N-m)D}}-\frac{1}{\sqrt{(N-m)D}})+E_lx^{coop}$.

Take the first derivative of $H_l(x^{coop})$:

\noindent$H_l^{'} (x^{coop}) = E_l - \frac{(m-l+1)\rho_l}{2\sqrt{G((N-m)D+(m-l+1)x_1^{coop})^3}}$.

When $\frac{\rho_l}{E_l} > B_l^{'}$, $H_l^{'}(0) < 0, H_l(0) = 0, H_l^{'}(+\infty) >0$. It's easy to know that $H_l^{'}(x^{coop})$ is a monotonically increasing function of $x^{coop}$, so the curve of $H_l(x^{coop})$ is a curve that monotonically decreases and then monotonically increases, like a parabola with an upper opening. When $\frac{\rho_l}{E_l} > O_l^{'}, H_l(D) < 0$, the unique intersection of $H_l(x^{coop})$ and the $x^{coop}$ axis at $(0, + \infty)$ is larger than $D$, so the maximum amount of local data that client $l$ can choose under the cooperative strategy is $x_l^{th}=D$. The $m-l+1$ converted contributors can choose any same strategy on $(0, D]$ as their cooperative strategy. This strategy can reduce their cost. When $B_l^{'} < \frac{\rho_l}{E_l} \leq O_l^{'}$, $H_l(D) \geq 0$, the unique intersection of $H_l(x^{coop})$ and the $x^{coop}$ axis at $(0, + \infty)$ is the solution of the implicit function (17), we set it as $x_l^{th}(\frac{\rho_l}{E_l})$, so $x_l^{th} = x_l^{th}(\frac{\rho_l}{E_l})$. The $m-l+1$ converted contributors can choose any strategy on $(0, x_l^{th}(\frac{\rho_l}{E_l}))$ as their cooperative strategy. This strategy can reduce their cost.

When $\frac{\rho_l}{E_l} \leq B_l^{'}$, it means $H_l^{'}(0) \geq 0, H_l(0) = 0$, $H_l(x^{coop})$ is a monotonically increasing function of $x^{coop}$, so there is no intersection at $(0, + \infty)$, all free riders should choose 0 as their strategy. If $\frac{\rho_1}{E_1} \leq B_1^{'}, \frac{\rho_2}{E_2} \leq B_2^{'}, \cdots, \frac{\rho_m}{E_m} \leq B_m^{'}$, it means that there will not exist a cooperative strategy to reduce the cost of each client at the same time, so all client will not cooperate and play the NE of Game 1. The reason is the same as in the proof of Theorem 3.

According to Theorem 4, clients $1, 2, \cdots, l-1$ will choose 0 as their strategy. Because $(m-l+1)x^{coop} + (N-m)D > (N-m)D$, we can get
\begin{equation}
\begin{split}
H_1(x^{coop})= &\frac{\rho_1}{\sqrt{G}}(\frac{1}{\sqrt{(m-l+1)x^{coop} + (N-m)D}} \\
&- \frac{1}{\sqrt{(N-m)D}}) < 0.
\end{split}
\end{equation}
Similarly, $H_2(x^{coop})<0, \cdots, H_{l-1}(x^{coop}) < 0$.

Because $\frac{\rho_l}{E_l} > B_l^{'}$, we can get 
\begin{equation}
\begin{split}
H_l(x^{coop}) = &\frac{\rho_l}{\sqrt{G}}(\frac{1}{\sqrt{(m-l+1)x^{coop} + (N-m)D}} \\
&- \frac{1}{\sqrt{(N-m)D}})+E_lx^{coop} \leq 0.
\end{split}
\end{equation}

By transplanting the inequality, we can get

\begin{equation}
\frac{\rho_l}{E_l} \geq \frac{\sqrt{G}x^{coop}}{\frac{1}{\sqrt{(N-m)D}}-\frac{1}{\sqrt{(m-l+1)x^{coop} + (N-m)D}}}.
\end{equation}

Because $\frac{\rho_l}{E_l} < \frac{\rho_{l+1}}{E_{l+1}} < \cdots < \frac{\rho_m}{E_m}$, we can get

\begin{equation}
\frac{\rho_{l+1}}{E_{l+1}} > \frac{\rho_l}{E_l} \geq \frac{\sqrt{G}x^{coop}}{\frac{1}{\sqrt{(N-m)D}}-\frac{1}{\sqrt{(m-l+1)x^{coop} + (N-m)D}}}.
\end{equation}

By transplanting the inequality, we can get

\begin{equation}
\begin{split}
&\frac{\rho_{l+1}}{\sqrt{G}}(\frac{1}{\sqrt{(m-l+1)x^{coop} + (N-m)D}} \\
&- \frac{1}{\sqrt{(N-m)D}})+E_{l+1}x^{coop} = H_{l+1}(x^{coop}) < 0.
\end{split}
\end{equation}
Similarly, we can get $H_{l+2}(x^{coop}) < 0, \cdots, H_m(x^{coop}) <0$.

For client $m+1$, his cooperative strategy is $D$, we can get
\begin{equation}
\begin{split}
H_{m+1}(x^{coop}) = &\frac{\rho_{m+1}}{\sqrt{G}}(\frac{1}{\sqrt{(m-l+1)x^{coop} + (N-m)D}} \\
&- \frac{1}{\sqrt{(N-m)D}}) < 0 .
\end{split}
\end{equation}
Similarly, $H_{m+2}(x^{coop})<0, \cdots, H_N(x^{coop}) < 0$.

\section{Proof of Theorem 5}\label{lem3}

Here we assume cooperative free riders choose $x^{coop}$ as their cooperative strategy.

For client $n$, when client $n$ utilizes the cooperation of other clients to achieve the least cost at current time slot. The cost function of client $n$ is

\begin{equation}
\begin{split}
&F_n(\boldsymbol{x}_n^{least}) = \\
&\frac{\rho_n}{\sqrt{G}}(\frac{1}{\sqrt{(k-l-1)x^{coop}+x_k^{*}+(N-k)D+x}} \\
&+\frac{1}{\sqrt{G}})+E_nx+C_n+p.
\end{split}
\end{equation}

Set $F_n(\boldsymbol{x}_n^{least})^{'} = 0$, we can get

\begin{equation}
\begin{split}
x &= \sqrt[3]{\frac{\rho_n^2}{4GE_n^2}}-(k-l-1)x^{coop}-(N-k)D-x_k^{*} \\
& \leq \sqrt[3]{\frac{\rho_n^2}{4GE_n^2}}-\sqrt[3]{\frac{\rho_k^2}{4GE_k^2}} -(k-l-1)x^{coop} \\
& <0.
\end{split}
\end{equation}

Similarly, for client $l+1,\cdots, k-1$, the point at which the cost function minimizes is $0$. So we can calculate the threshold discount factor $\delta_l^{th}(\boldsymbol{x}^{coop})$ of client $l$:

\begin{equation}
\begin{aligned}
&\delta_l^{th}(\boldsymbol{x}^{coop}) = \frac{F_l(\boldsymbol{x}^{coop})-F_l(\boldsymbol{x}_n^{least})}{F_l(\boldsymbol{x}^{pun})-F_l(\boldsymbol{x}_n^{least})} \\
&= \frac{\frac{1}{\sqrt{(k-l)x^{coop}+x_k^{*}+(N-k)D}}-\frac{1}{\sqrt{(k-l-1)x^{coop}+x_k^{*}+(N-k)D}}}{\frac{1}{\sqrt{x_k^*+(N-k)D}}-\frac{1}{\sqrt{(k-l-1)x^{coop}+x_k^{*}+(N-k)D}}} \\
&+\frac{E_lx^{coop}\sqrt{G}}{\rho_l (\frac{1}{\sqrt{x_k^*+(N-k)D}}-\frac{1}{\sqrt{(k-l-1)x^{coop}+x_k^{*}+(N-k)D}})}.
\end{aligned}
\end{equation}

It's easy to see that the threshold discount factor of $\delta_l^{th}(\boldsymbol{x}^{coop})$ decreases with the increase of $\frac{\rho_l}{E_l}$, so we can know that the threshold discount factor of client $l$ is larger than the threshold discount factor of other converted contributors, i.e., $\delta_l^{th}(\boldsymbol{x}^{coop}) > \delta_{l+1}^{th}(\boldsymbol{x}^{coop}) > \cdots > \delta_{k-1}^{th}(\boldsymbol{x}^{coop})$.

For client $k$, similar to client $l$, we can get the point at which the cost function of client $k$ minimizes, which we denote as $x_k^{least}$:

\begin{equation}
x_k^{least} = 
\left\{
\begin{aligned}
& 0, & \mbox{ if } & x_k^*-(k-l)x^{coop} \leq 0; \\
& x_k^*-(k-l)x^{coop}, & \mbox{ if } & 0< x_k^*-(k-l)x^{coop}<D.
\end{aligned}
\right.
\end{equation}

The cost function of client $k$ under this strategy profile is: 
\begin{equation}
\begin{split}
&F_k(\boldsymbol{x}_k^{least}) =\\
&\frac{\rho_k}{\sqrt{G}}(\frac{1}{(k-l)x^{coop}+x_k^{least}+(N-k)D} +\frac{1}{\sqrt{G}})+\\
&E_kx_k^{least}+C_k+p.
\end{split}
\end{equation} 

Then we calculate the threshold discount factor of client $k$:

\begin{equation}
\delta_k^{th}(\boldsymbol{x}^{coop}) = \frac{F_k(\boldsymbol{x}^{coop})-F_k(\boldsymbol{x}_k^{least})}{F_k(\boldsymbol{x}^{pun})-F_k(\boldsymbol{x}_k^{least})}.
\end{equation}

For client $k+1$, we can get the point at which the cost function of client $k+1$ minimizes, which we denote as $x_{k+1}^{least}$, here we denote $\sqrt[3]{\frac{\rho_{k+1}^2}{4GE_{k+1}^2}}-(k-l)x^{coop} -x_k^{*}-(N-k-1)D$ as $x_{k+1}^{min}$ for convenience, 

\begin{equation}
x_{k+1}^{least} = 
\left\{
\begin{aligned}
& 0,  & \mbox{ if } &x_{k+1}^{min} \leq 0; \\
& x_{k+1}^{min},  & \mbox{ if } &0<x_{k+1}^{min}<D; \\
& D, & \mbox{ if } &x_{k+1}^{min} \geq D.
\end{aligned}
\right.
\end{equation}

The cost function of client $k+1$ under this strategy profile is: 
\begin{equation}
\begin{split}
&F_{k+1}(\boldsymbol{x}^{least}) =\\
&\frac{\rho_{k+1}}{\sqrt{G}}(\frac{1}{(k-l)x^{coop}+x_k^{*}+(N-k-1)D+x_{k+1}^{least}} +\frac{1}{\sqrt{G}})+\\
&E_{k+1}x_{k+1}^{least}+C_{k+1}+p.
\end{split}
\end{equation} 

Then we calculate the threshold discount factor of client $k+1$:

\begin{equation}
\delta_{k+1}^{th}(\boldsymbol{x}^{coop}) = \frac{F_{k+1}(\boldsymbol{x}^{coop})-F_{k+1}(\boldsymbol{x}_{k+1}^{least})}{F_{k+1}(\boldsymbol{x}^{pun})-F_{k+1}(\boldsymbol{x}_{k+1}^{least})}.
\end{equation}

Similarly, for client $k+2, \cdots, N$, we can get their threshold discount factor $\delta_{k+2}^{th}(\boldsymbol{x}^{coop}), \cdots, \delta_N^{th}(\boldsymbol{x}^{coop})$ respectively.

When the critical client $k$ does not exist, the process is similarly to the above, and we will not repeat it.

\section{Additional performances}\label{addper}

We first show the impact of several system parameters on $N_f$ and $R_d$, including the number of iterations $G$, the computation cost coefficient $E$, and the discount factor $\delta$. We also discuss the impact of the distribution of clients’ valuation parameters on the number of reduced free riders.

\begin{figure}[t]
 \centering
 \begin{minipage}[h]{0.48\linewidth}
 \centering
 \includegraphics[width=1.02\textwidth]{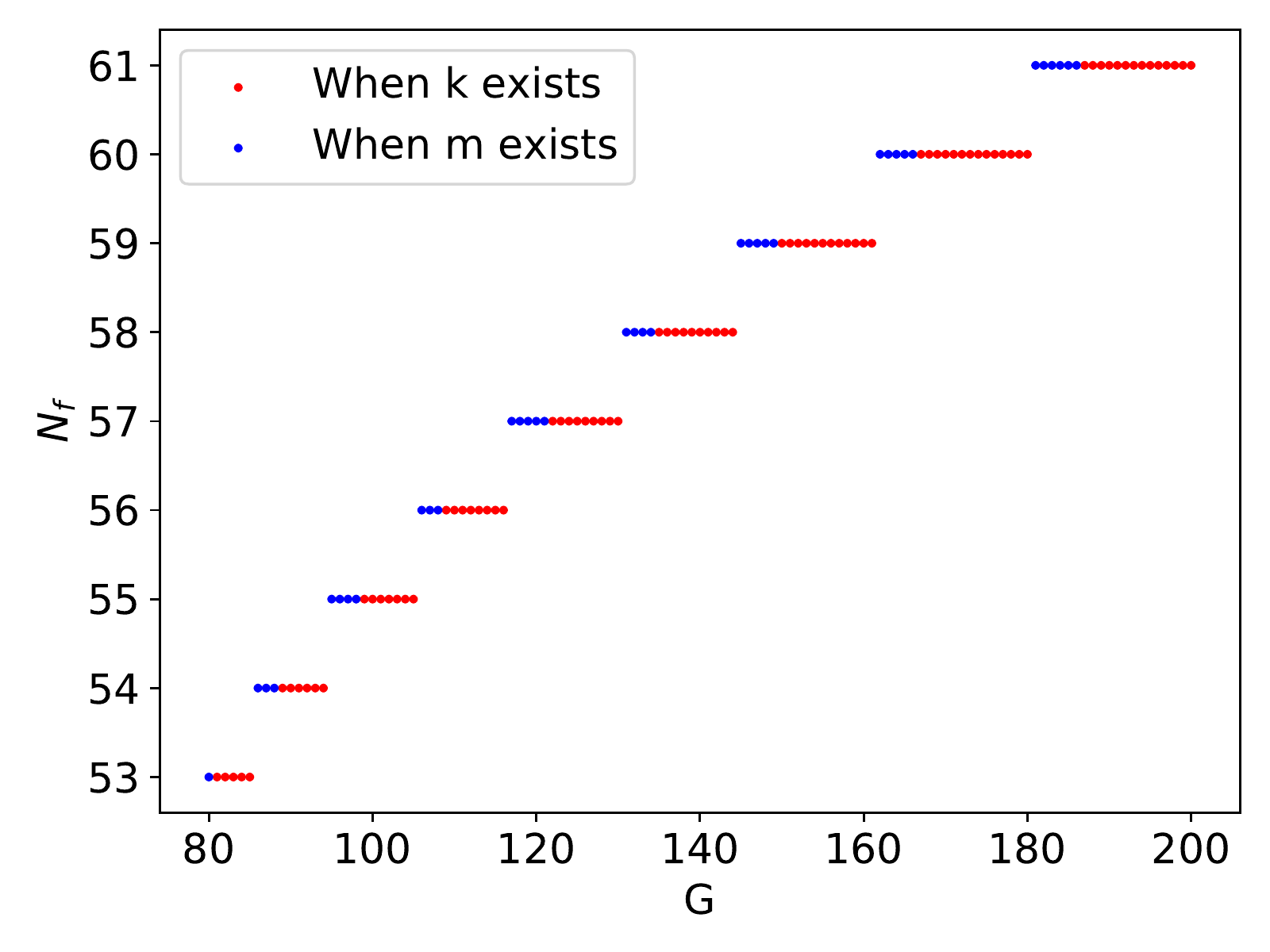}
 \caption{$N_f$ under different $G$}\label{figG1}
 \end{minipage}%
 \begin{minipage}[h]{0.48\linewidth}
 \centering
 \includegraphics[width=1.02\textwidth]{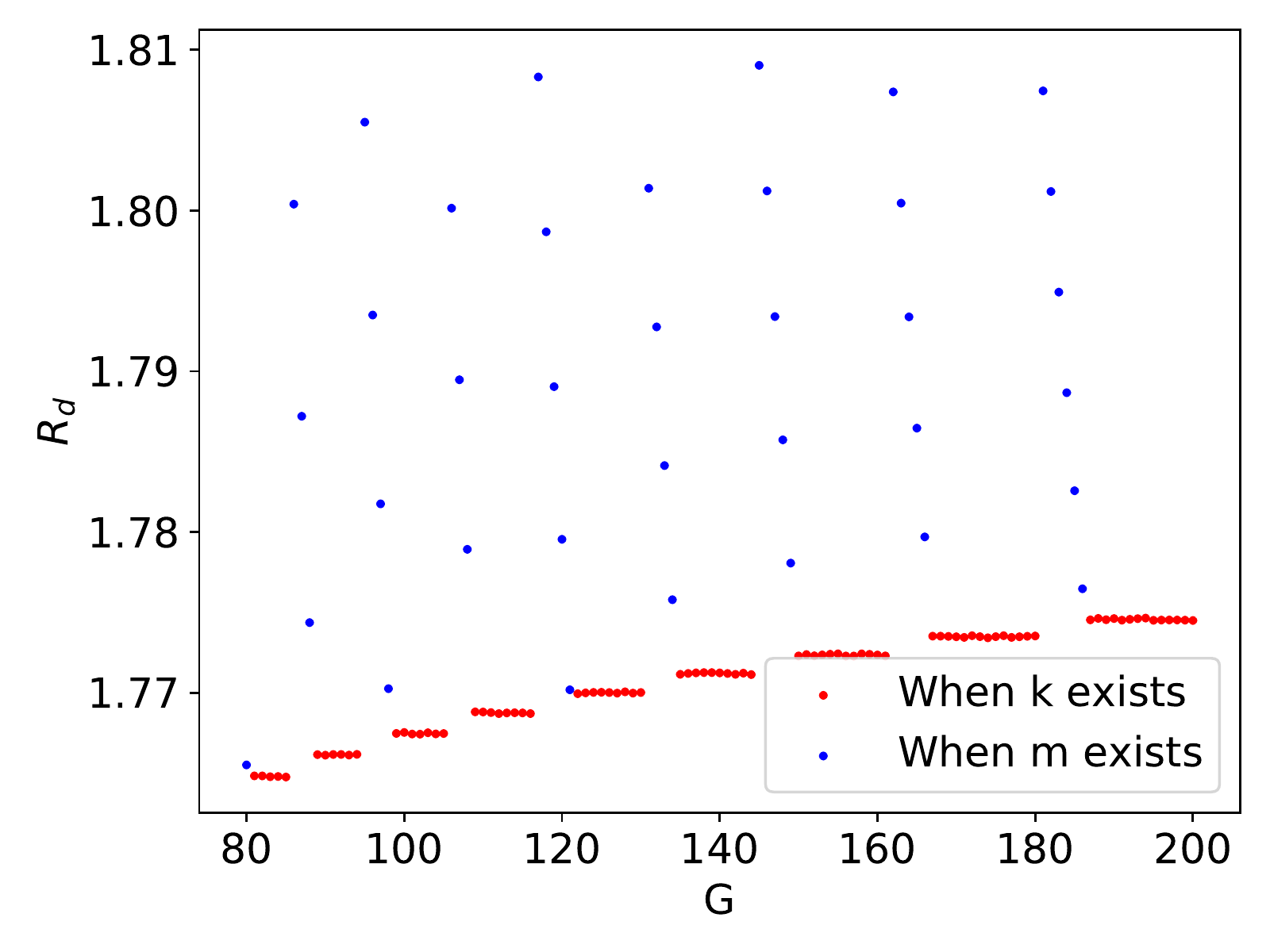}
 \caption{$R_d$ under different $G$}\label{figG2}
 \end{minipage}%
 \centering
\end{figure}

\textbf{The impact of the number of iterations $G$.} We show how the number of iterations $G$ affects the number of reduced free riders $N_f$ and the data contribution ratio $R_d$ in Fig. \ref{figG1} and Fig. \ref{figG2} respectively.

Fig. \ref{figG1} shows that the number of reduced free riders increases with $G$. A larger $G$ means more iterations in one cross-silo FL process, in which case clients can achieve a global model with a higher model accuracy. As $G$ increases, at the NE of the stage game SPFL, contributors with low valuation-computation ratios for the global model accuracy will change to be free riders, because the global model trained by contributors with high valuation-computation ratios under a large $G$ is good enough. So the number of free riders at the NE of the stage game increases with the number of iterations $G$ in cross-silo FL. 
The curve is a stepwise curve, where each step (i.e., the interval of $G$ values) corresponds to a fixed number of free riders at the NE of the stage game SPFL. When the critical client $k$ exists at equilibrium, its equilibrium strategy $x_k^\ast$ (calculated in Theorem 1) decreases with $G$, until $x_k^\ast=0$ in which case, the number of free riders at equilibrium increases by one. When the critical client $k$ does not exist, the number of free riders at equilibrium does not change with $G$, until a contributor changes to be a free rider.
Our proposed optimal SPNE can convert most free riders into converted contributors and thus solves the free-rider problem effectively.

Fig. \ref{figG2} shows how the data contribution ratio $R_d$ changes with $G$. Specifically, when the critical client $k$ exists, the data contribution ratio $R_d$ generally increases with $G$. As we discussed in Fig. \ref{figG1}, the number of free riders at the NE of the stage game increases with $G$, and there will be less local data for model training. Since the equilibrium strategy $x_k^\ast$ of the critical client $k$ decreases with $G$, converted contributors will choose more local data for model training to reduce their costs, and hence the total data chosen by clients increases. When the critical client $k$ does not exist, under the same number of free riders, the data contribution ratio $R_d$ decreases with $G$. The reason is that under a larger $G$, converted contributors will choose less local data for model training to reduce the computation cost while achieving the same global model accuracy. Under the same number of contributors, the total amount of training data chosen by converted contributors when the critical client $k$ exists is less than that when the critical client $k$ does not exist, thus the blue dots are above the red dots under the same number of contributors.
Our proposed optimal SPNE can increase the amount of total data for model training compared with the NE of the stage game.

\begin{figure}[t]
 \centering
 \begin{minipage}[t]{0.48\linewidth}
 \centering
 \includegraphics[width=1.02\textwidth]{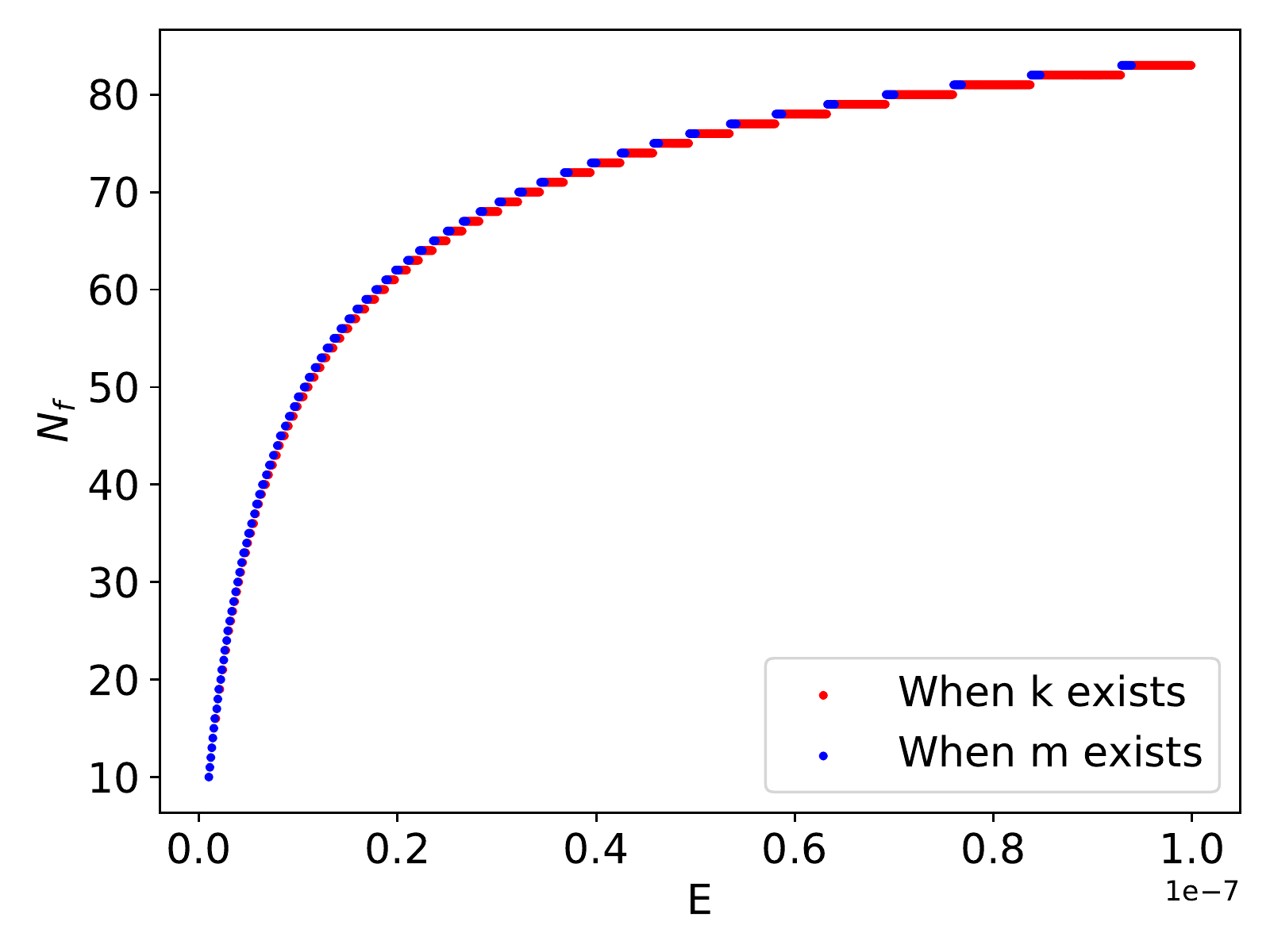}
 \caption{$N_f$ under different $E$}\label{figE1}
 \end{minipage}%
 \begin{minipage}[t]{0.48\linewidth}
 \centering
 \includegraphics[width=1.02\textwidth]{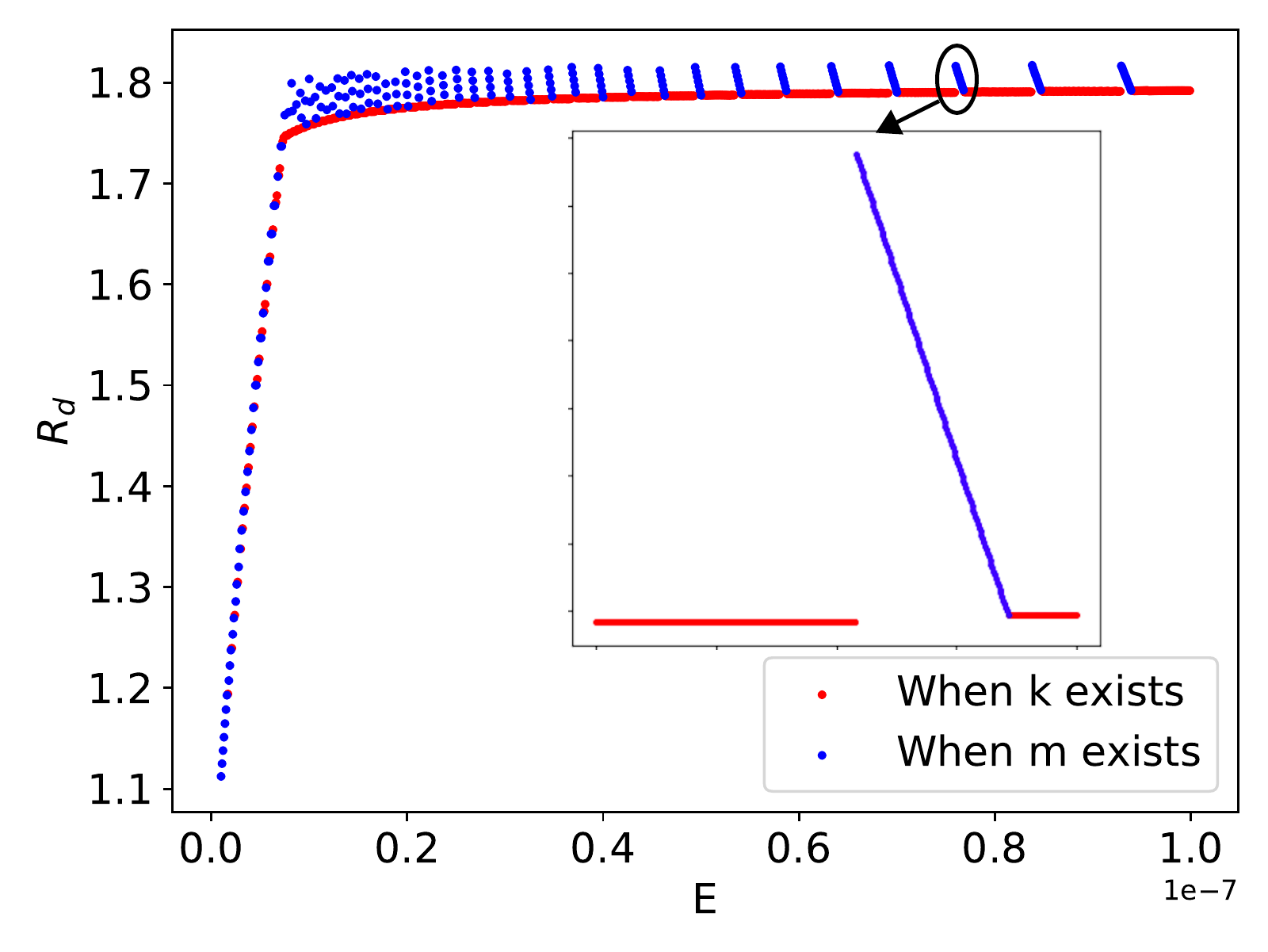}
 \caption{$R_d$ under different $E$}\label{figE2}
 \end{minipage}%
\end{figure}

\emph{In summary, when each client in cross-silo FL has more local data samples, the selfish participation behavior leads to more free riders. Our proposed optimal SPNE can effectively reduce the number of free riders by $98.4\%$, and increase the data contribution by $80.9\%$.}

\textbf{The impact of the computation cost coefficient $E$.} We show how the parameter $E$ affects the number of reduced free riders $N_f$ and the data contribution ratio $R_d$ in Fig. \ref{figE1} and Fig. \ref{figE2} respectively.

Fig. \ref{figE1} shows that the number of reduced free riders increases with $E$ with a diminishing marginal return effect. When $E$ is small, performing local model training incurs a small computation cost, and hence there are few free riders at the NE of the stage game. When $E$ is large, clients experience a huge computation cost when performing model training, and hence many clients with low valuation-computation ratios choose to be free riders at the NE of the stage game to avoid computation cost. Our proposed optimal SPNE can effectively reduce the number of free riders. 

Fig. \ref{figE2} shows how the data contribution ratio $R_d$ changes with $E$, and the trend of the curve is similar to that in Fig. 6. Specifically, when the critical client $k$ exists, the data contribution ratio $R_d$ generally increases with $E$. When $E$ is small, $R_d$ increases quickly with $E$, and when $E$ is large, $R_d$ increases slowly in a zigzag pattern. When $E$ is small, there are few free riders at the NE of the stage game, and hence there is no much room to increase the amount of local data for model training. As $E$ increases, the cooperative strategy allows the converted contributors to choose a positive amount of local data to perform model training, which increases the amount of chosen data by up to $81.6\%$. The subfigure in Fig. \ref{figE2} shows the zigzag shape of the curve under the same number of free riders. The reason is similar to that of the subfigure of Fig. 6.

\begin{figure}[t]
 \centering
 \begin{minipage}[t]{0.48\linewidth}
 \centering
 \includegraphics[width=1.05\textwidth]{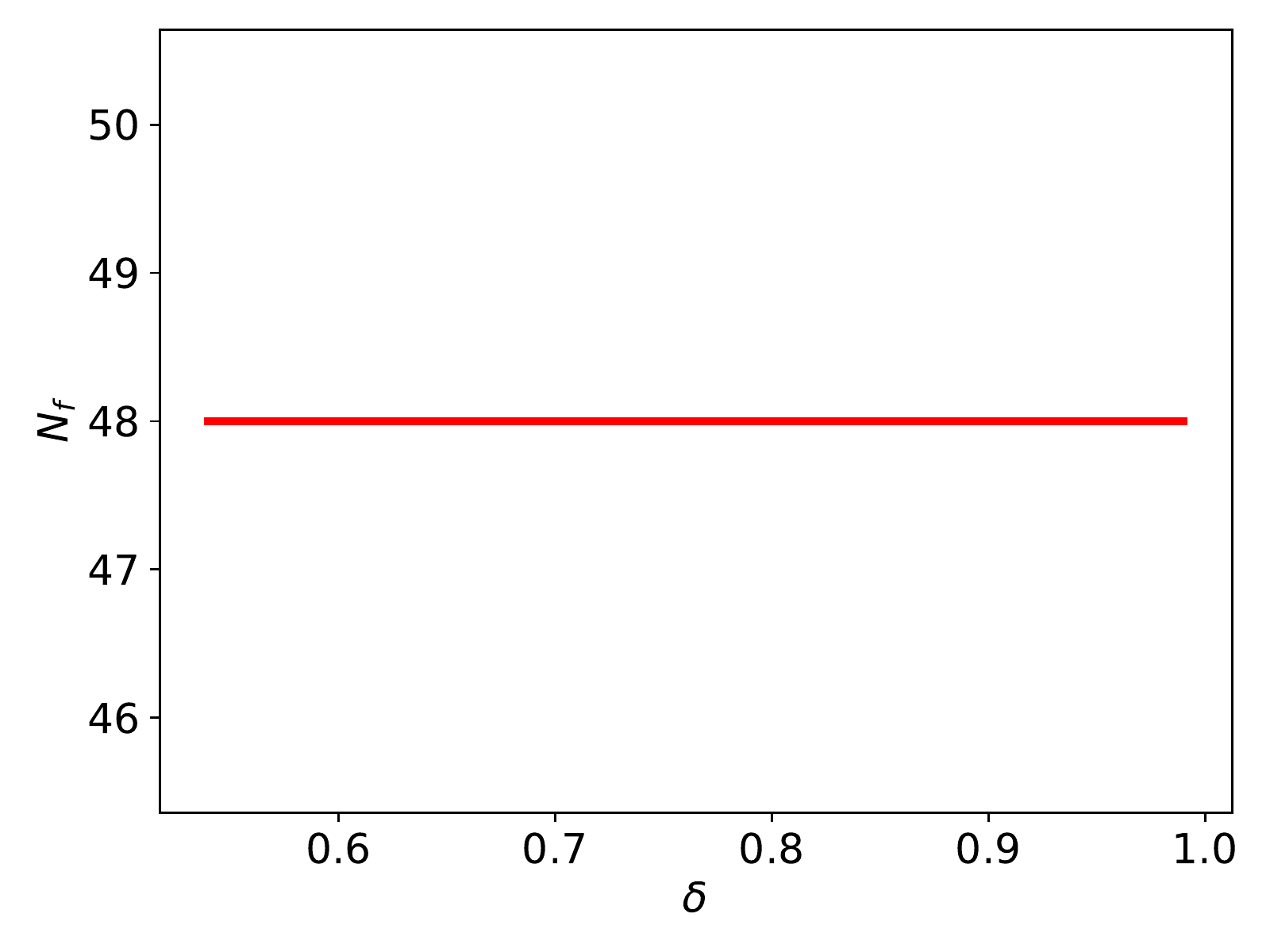}
 \caption{$N_f$ under different $\delta$}\label{figd1}
 \end{minipage}
 \begin{minipage}[t]{0.48\linewidth}
 \centering
 \includegraphics[width=1.05\textwidth]{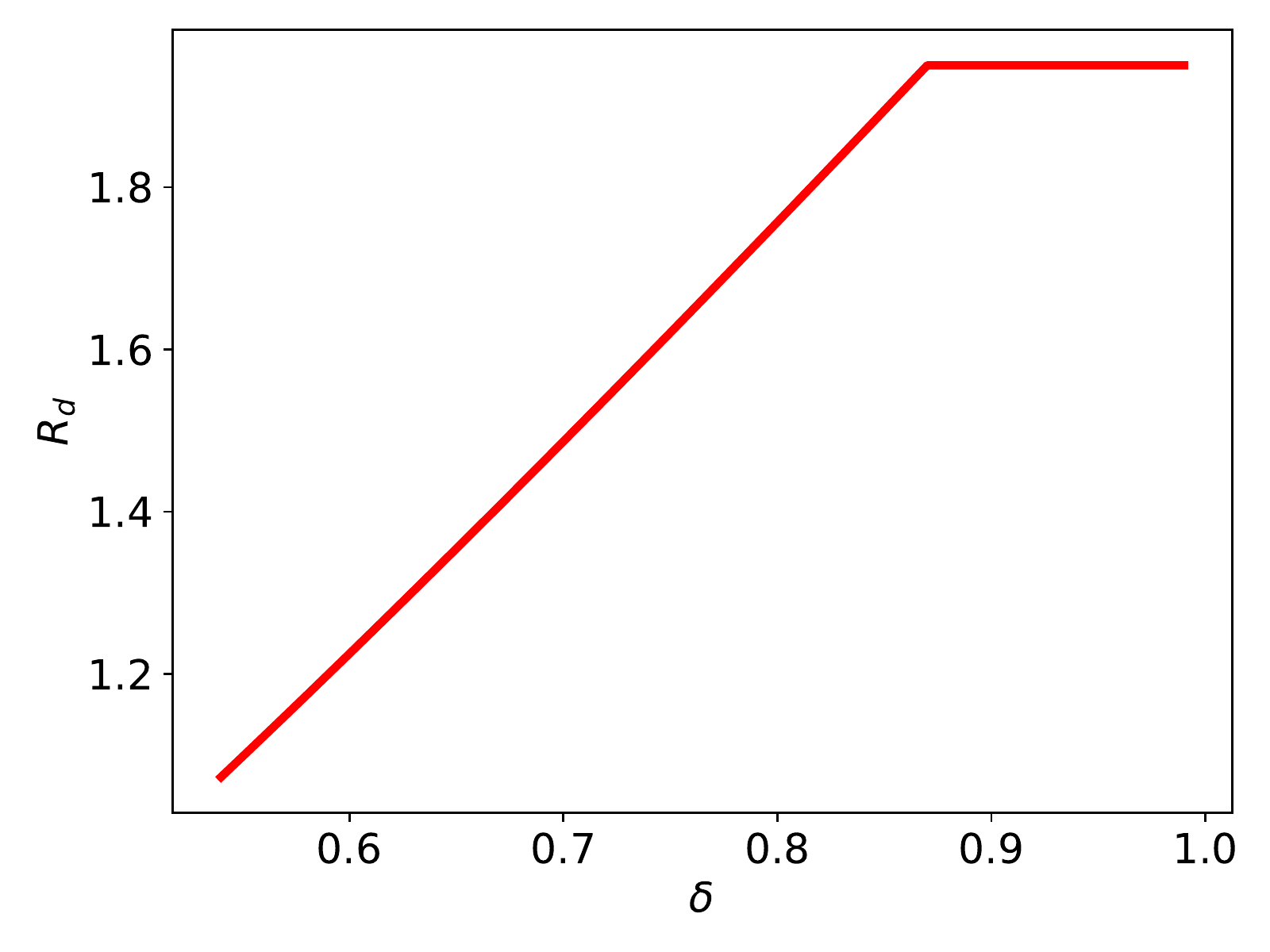}
 \caption{$R_d$ under different $\delta$}\label{figd2}
 \end{minipage}%
\end{figure}

\textbf{The impact of the discount factor $\delta$.} We show how the discount factor $\delta$ affects the number of reduced free riders $N_f$ and the data contribution ratio $R_d$ in Fig. \ref{figd1} and Fig. \ref{figd2} respectively.

Fig. \ref{figd1} shows that the discount factor $\delta$ does not affect the number of reduced free riders. The reason is that the number of free riders depends on the valuation-computation ratios $\{\frac{\rho_n}{E_n}: \forall n\in\mathcal{N}\}$, the bounds $\{B_n: \forall n\in\mathcal{N}\}$, and the local data set sizes $\{D_n: \forall n\in\mathcal{N}\}$, and is independent of $\delta$.

Fig. \ref{figd2} shows that the data contribution ratio first increases with $\delta$, and then remains to be a constant when $\delta$ is larger than $0.87$. A larger discount factor $\delta$ indicates that converted contributors are more patient, and hence they are more willing to choose a larger amount of local data to reduce the long-term discounted total cost. When converted contributors are patient enough, i.e., $\delta \geq 0.87$, they will perform model training with all their local data, in which case the data contribution ratio is $R_d=1.96$.

\textbf{The impact of the distributions of $\rho$.} Fig. \ref{delta3} shows how the number of reduced free riders changes with $G$ under three distributions of $\rho$: $80\%$ low, $20\%$ high; $50\%$ low, $50\%$ high; and $20\%$ low, $80\%$ high. Here $80\%$ low means the valuation parameters of 80 clients are uniformly distributed in $(0,50]$, $20\%$ high means the valuation parameters of 20 clients are uniformly distributed in $(50, 100]$. We can see that there will be more free riders at the NE of the stage game when more clients have low valuations for global model accuracy. Our proposed optimal SPNE can convert most free riders into \emph{converted contributors} and thus solves the free-rider problem effectively.

\begin{figure}[t]
\centering
\includegraphics[width=2in]{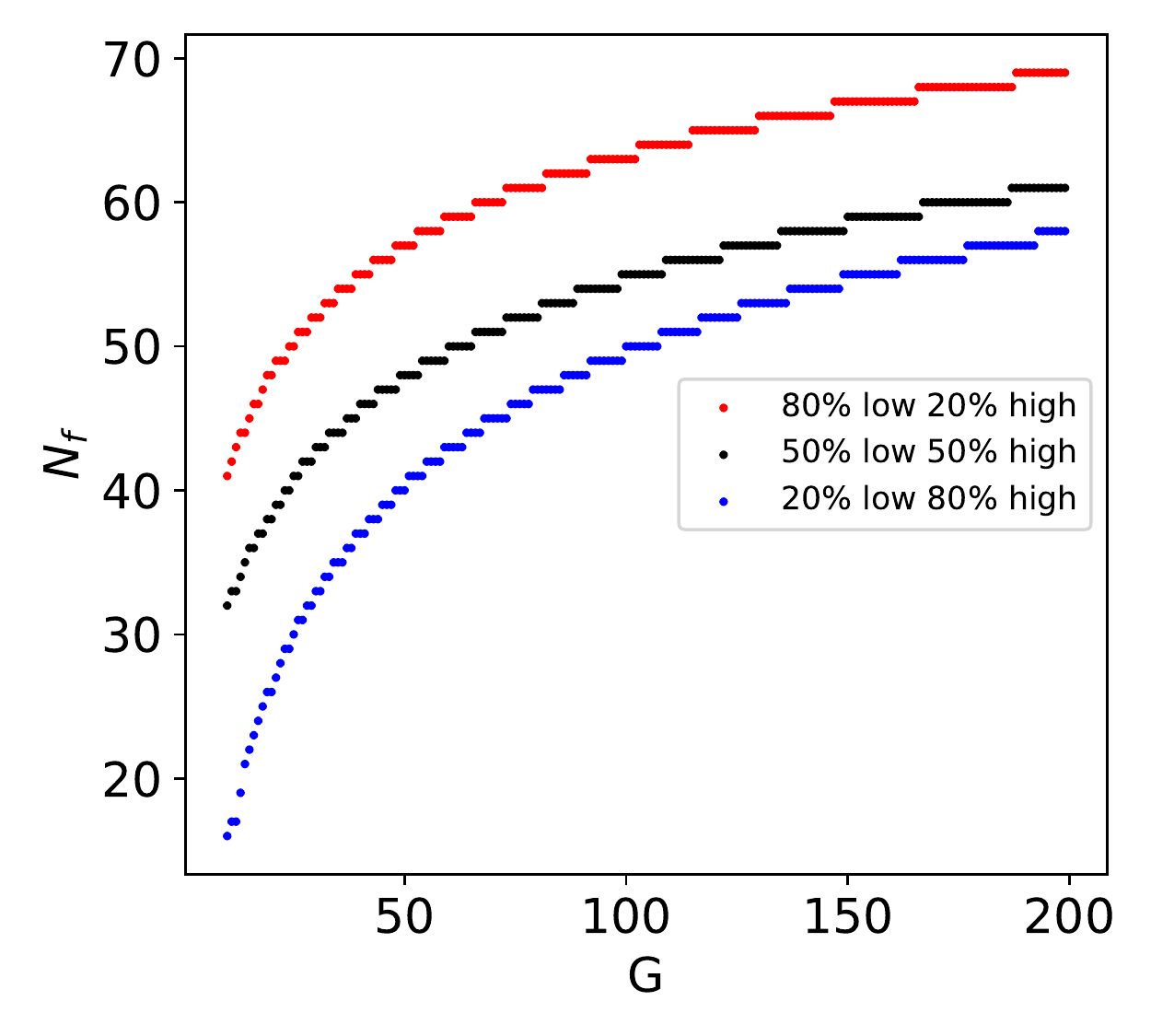}
\centering
\caption{$N_f$ under three different distributions of $\rho$ and different $G$}\label{delta3}
\end{figure}

\begin{figure}[t]
\centering
\begin{minipage}[h]{0.48 \linewidth}
\centering
\includegraphics[width=1.04\textwidth]{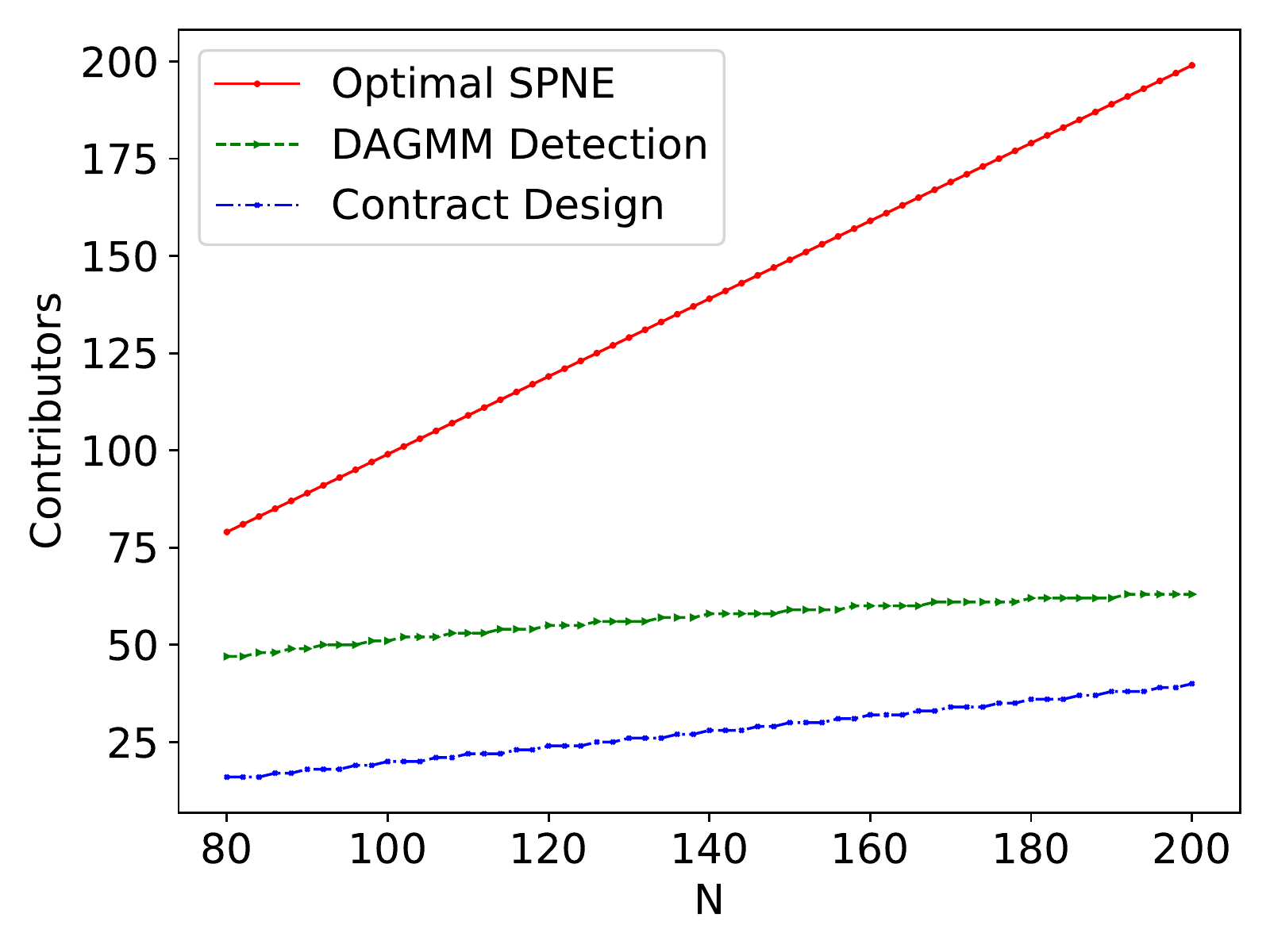}
  \caption{The number of contributors under different $N$}\label{CN}
\end{minipage}
\begin{minipage}[h]{0.48 \linewidth}
\centering
\includegraphics[width=1.04\textwidth]{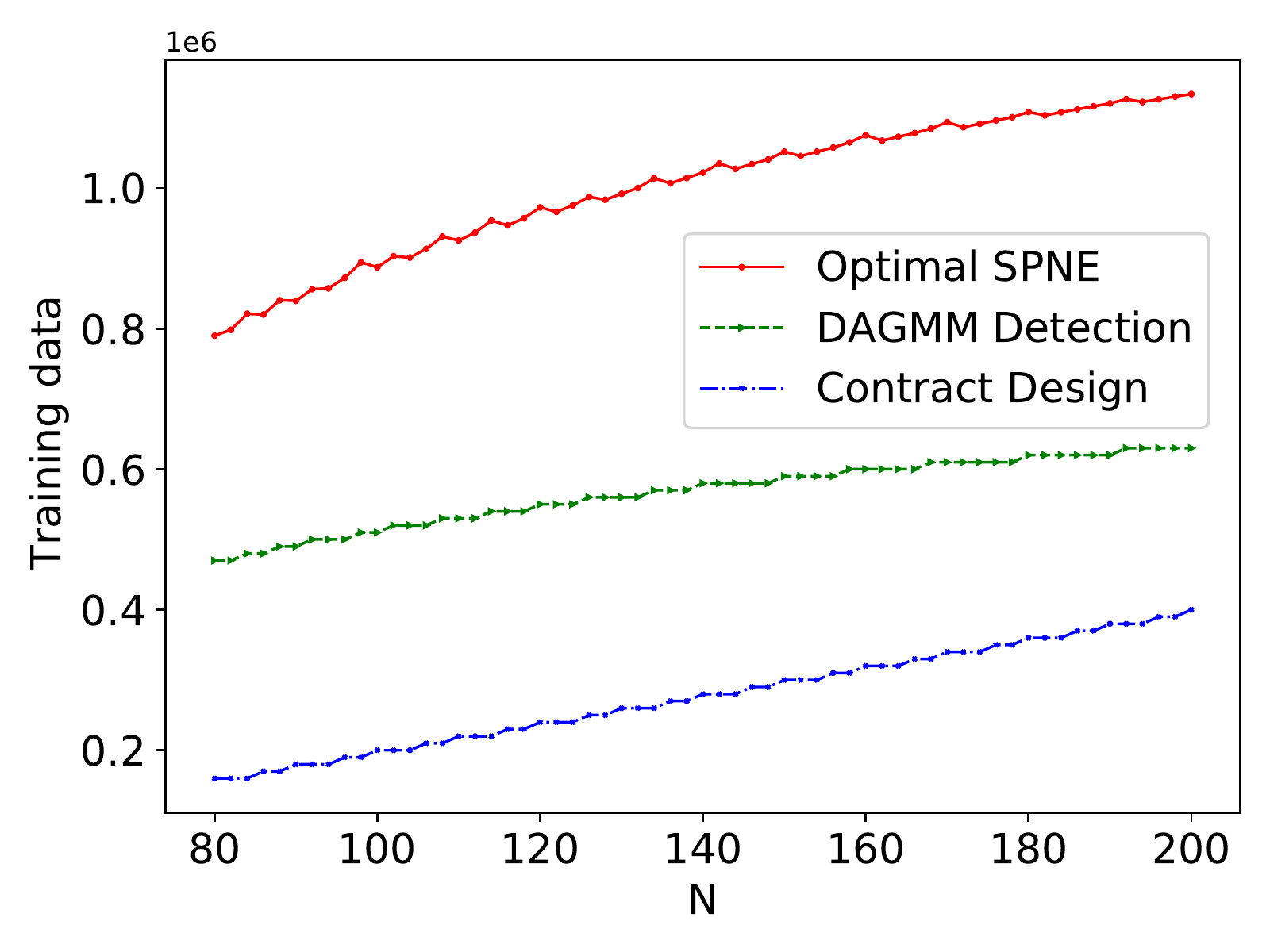}
  \caption{The total amount of training data under different $N$}\label{TN}
\end{minipage}
\end{figure}

In the following, we compare our method with the DAGMM detection method and the contract design method under different system parameters (i.e., the number of clients $N$, the number of local data samples $D$ and the computation cost coefficient $E$). 

Fig. \ref{CN} and Fig. \ref{TN} show the comparisons of the three methods under different values of $N$. Fig. \ref{CN} shows that among the three methods, our method achieves the largest number of contributors under different values of $N$. Under the DAGMM detection method, as $N$ increases, the global model trained by the contributors with high valuation-computation ratios is good enough, and hence the contributors with low valuation-computation ratios may become free riders to reduce their total costs. When the global model is good enough, if a new client joins the cross-silo FL process, it is very likely to choose to be a free rider. Therefore, as $N$ increases, the number of contributors under the DAGMM detection method remains almost unchanged, which implies that the number of free riders increases. Similarly, under the contract design method, the number of contributors increases only when the new client is provided with a positive contract item. Our proposed method can effectively motivate almost all free riders to become converted contributors.

Fig. \ref{TN} shows that our method has a larger amount of training data compared with the other two methods. The reason is that almost all free riders at the NE of the stage game choose to be converted contributors at the optimal SPNE of our method and use a positive amount of local data for model training. These free riders are abandoned in the other two methods.

\begin{figure}[t]
\centering
\begin{minipage}[h]{0.48 \linewidth}
\centering
\includegraphics[width=1.03\textwidth]{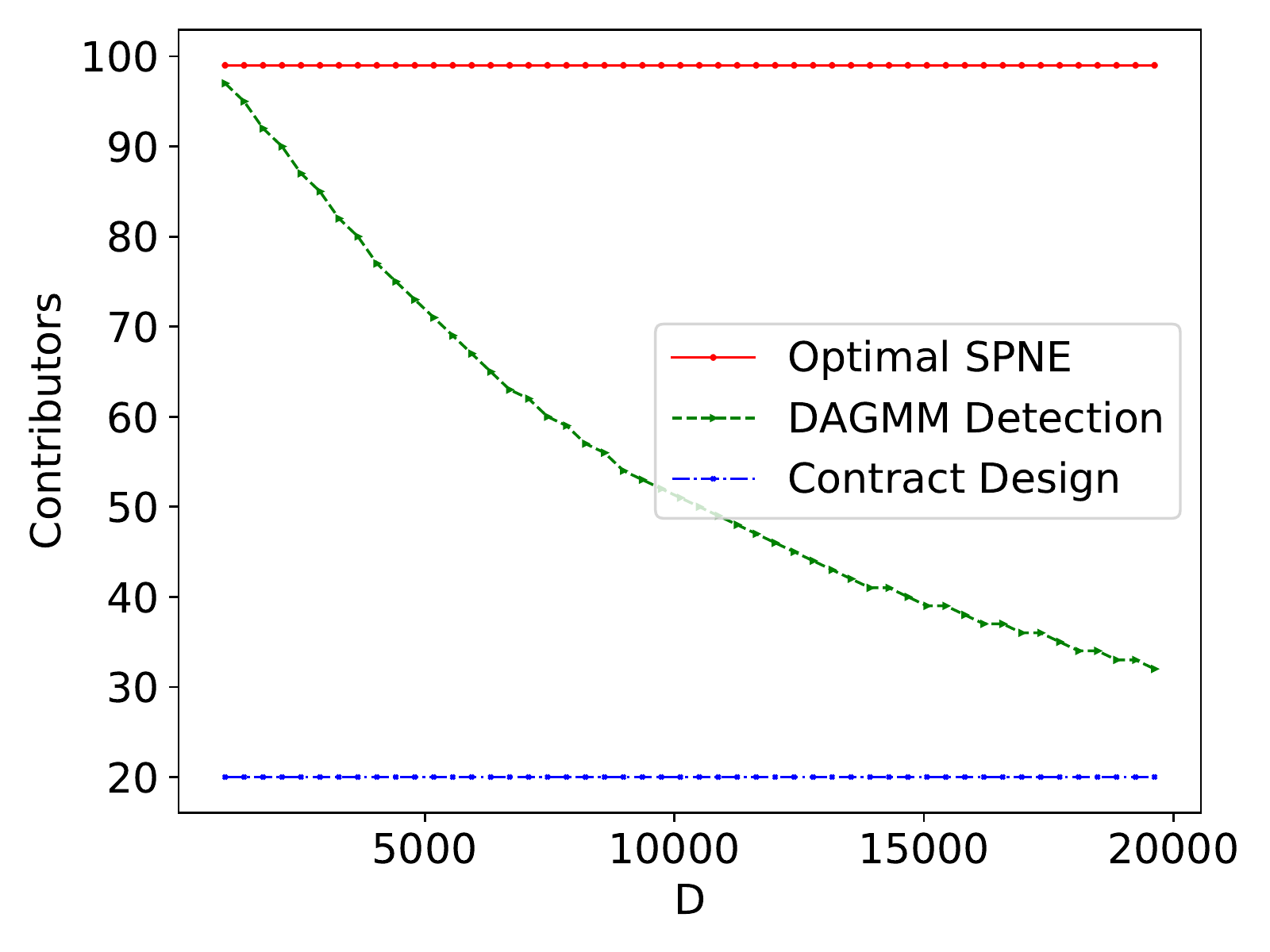}
  \caption{The number of contributors under different $D$}\label{CD}
\end{minipage}
\begin{minipage}[h]{0.48 \linewidth}
\centering
\includegraphics[width=1.03\textwidth]{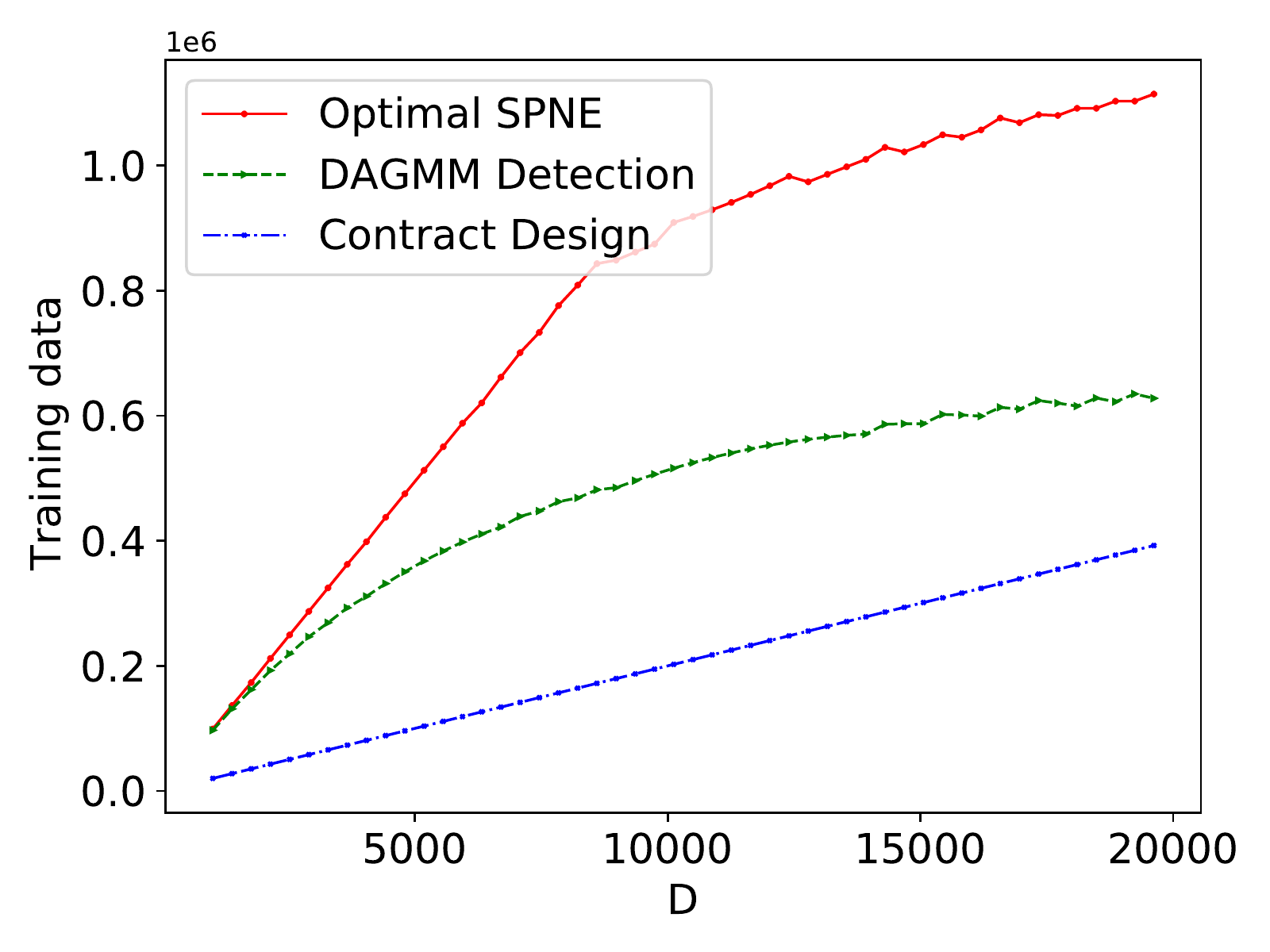}
  \caption{The total amount of training data under different $D$}\label{TD}
\end{minipage}
\end{figure}

Fig. \ref{CD} and Fig. \ref{TD} show the comparisons of the three methods under different values of $D$. Fig. \ref{CD} shows that among the three methods, our method achieves the largest number of contributors under different values of $D$. Specifically, under our method, the number of contributors remains unchanged and almost all clients participate in the training. Under the DAGMM detection method, the number of contributors decreases with $D$. Under the contract design method, only $20\%$ clients participate in the training. The reason is that under the DAGMM detection method, when $D$ is small, to achieve a good global model with a high model accuracy requires many clients to be contributors, and hence there are few free riders. When $D$ is large, a small number of contributors can train a good global model since each client has a large amount of local data to perform model training. Therefore, under the DAGMM detection method, the number of contributors decreases with $D$. Under the contract design method, since the change in $D$ does not affect the number of clients that are provided with a positive contract item, the number of contributors remains unchanged. Our proposed method can effectively motivate almost all free riders to become converted contributors. 

Fig. \ref{TD} shows that our method has a larger amount of training data compared with the other two methods. Specifically, under our method, the total amount of training data first increases linearly with $D$, and then increases slowly with $D$. The reason is that a smaller $D$ leads to a smaller threshold discount factor of clients. So when $D$ is small, converted contributors will participate in training with all local data. As $D$ increases, converted contributors can use only part of local data in training. Under the DAGMM detection method, since the number of contributors decreases with $D$, the total amount of training data increases slowly with $D$. Under the contract design method, since the change in $D$ does not affect the number of contributors, the total amount of training data increases linearly with $D$. 

\begin{figure}[t]
\centering
\begin{minipage}[h]{0.48 \linewidth}
\centering
\includegraphics[width=1.04\textwidth]{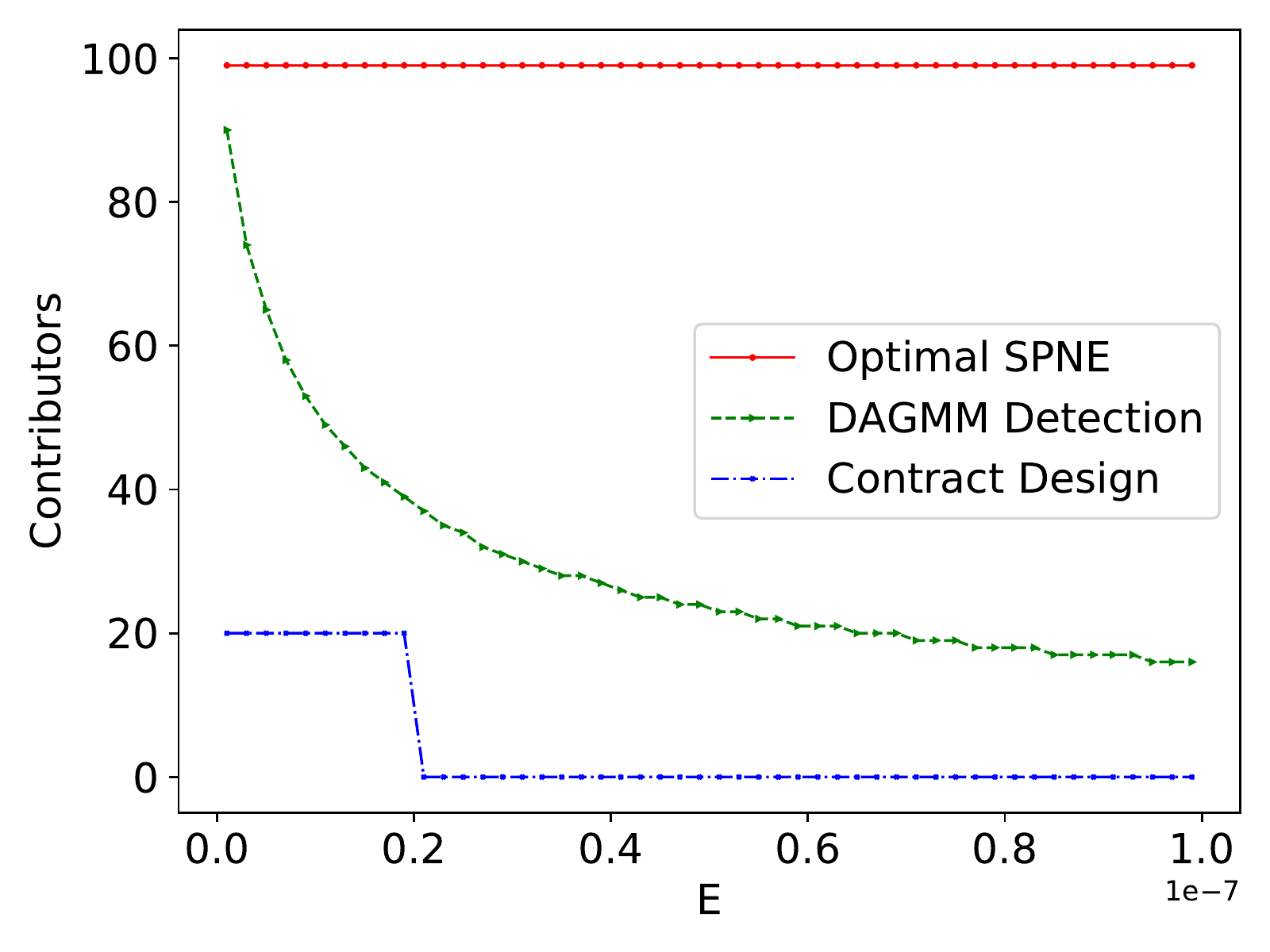}
  \caption{The number of contributors under different $E$}\label{CE}
\end{minipage}
\begin{minipage}[h]{0.48 \linewidth}
\centering
\includegraphics[width=1.04\textwidth]{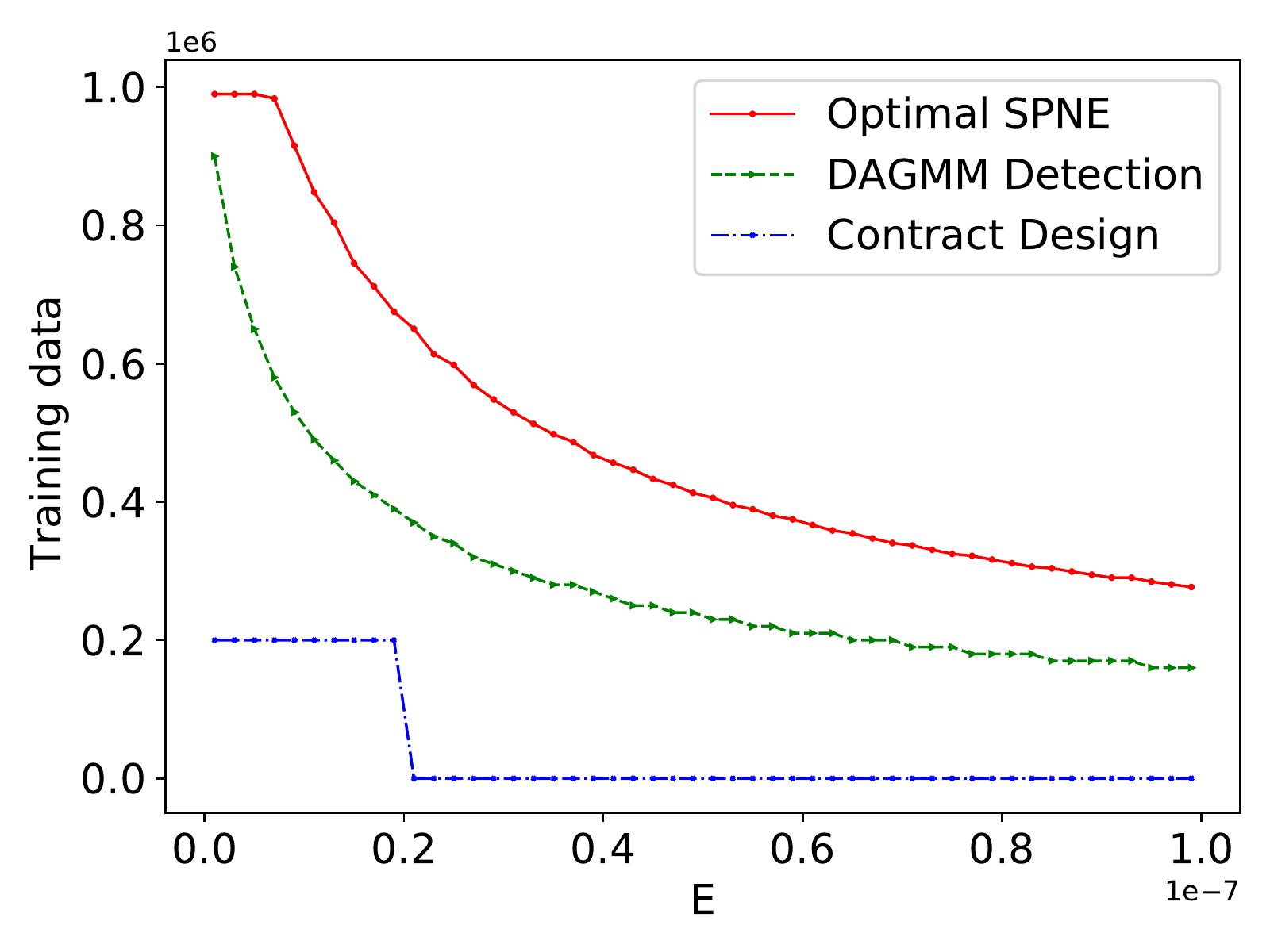}
  \caption{The total amount of training data under different $E$}\label{TE}
\end{minipage}
\end{figure}

Fig. \ref{CE} and Fig. \ref{TE} show the comparisons of the three methods under different values of $E$. Fig. \ref{CE} shows that among the three methods, our method achieves the largest number of contributors under different values of $E$. Specifically, under our method, the number of contributors remains unchanged and almost all clients participate in the training. Under the DAGMM detection method, the number of contributors decreases with $E$. The reason is that when $E$ is small, performing local model training incurs a small computation cost, and hence there are few free riders. When $E$ is large, clients experience a huge computation cost when performing model training, and hence many clients with low valuations for model accuracy choose to be free riders to avoid computation cost. Similarly, under the contract design method, the number of contributors first is a small positive number, and then drops to $0$ when $E$ is large. Our proposed method can effectively motivate almost all free riders to become converted contributors. 

Fig. \ref{TE} shows that our method has a larger amount of training data compared with the other two methods. Specifically, under our method, the total amount of training data first remains unchanged, and then decreases with $E$. The reason is that a smaller $E$ leads to a smaller threshold discount factor of clients. So when $E$ is small, converted contributors will participate in training with all local data. As $E$ increases, converted contributors can use only part of local data in training. Under the DAGMM detection method, since the number of contributors decreases with $E$, the total amount of training data also decreases. Under the contract design method, the change of the amount of training data with $E$ follows the same pattern as the change of the number of contributors in Fig. \ref{CE}.

\end{document}